\documentclass[letterpaper]{article}

\usepackage{ijcai19}
\usepackage{olo}
\usepackage{times}
\usepackage[utf8]{inputenc}
\usepackage[small]{caption}
\usepackage{booktabs}
\usepackage{adjustbox}
\usepackage{subcaption}
\usepackage{multirow}
\usepackage{balance}
\usepackage{grffile}
\newcommand{\err}{\text{err}}
\newcommand{\nnz}{\text{nnz}}
\newcommand{\dcg}{\text{DCG}\xspace}
\newcommand{\ndcg}{\text{nDCG}\xspace}
\newcommand{\tvx}{\tilde\vx}
\newcommand{\tvz}{\tilde\vz}
\newcommand{\tvc}{\tilde\vc}
\newcommand{\tvw}{\tilde\vw}
\newcommand{\tvdelta}{\tilde\vdelta}
\newcommand{\parabel}{\text{Parabel}\xspace}
\newcommand{\dismec}{\text{DiSMEC}\xspace}
\newcommand{\annexml}{\text{AnnexML}\xspace}
\newcommand{\craftml}{\text{CRAFTML}\xspace}
\newcommand{\sleec}{\text{SLEEC}\xspace}
\newcommand{\leml}{\text{LEML}\xspace}
\newcommand{\lpsr}{\text{LPSR}\xspace}
\newcommand{\ppds}{\text{PPDSparse}\xspace}
\newcommand{\plt}{\text{PLT}\xspace}
\newcommand{\pfrexml}{\text{PfastreXML}\xspace}
\newcommand{\fastxml}{\text{FastXML}\xspace}
\newcommand{\itdc}{\text{ITDC}\xspace}
\newcommand{\scbc}{\text{SCBC}\xspace}
\newcommand{\bcls}{\text{BCLS}\xspace}
\newcommand{\lsc}{\text{LSC}\xspace}
\newcommand{\defrag}{\textsf{DEFRAG}\xspace} 
\newcommand{\defragx}{\textsf{DEFRAG-X}\xspace}
\newcommand{\defragxy}{\textsf{DEFRAG-XY}\xspace}
\newcommand{\defragn}{\textsf{DEFRAG-N}\xspace}
\newcommand{\refrag}{\textsf{REFRAG}\xspace} 
\newcommand{\fiat}{\textsf{FIAT}\xspace} 

\graphicspath{{figs/}}

\makeatletter
\newcommand{\newreptheorem}[2]{\newtheorem*{rep@#1}{\rep@title}
\newenvironment{rep#1}[1]{\def\rep@title{#2 \ref*{##1}}\begin{rep@#1}}{\end{rep@#1}}
}
\makeatother
\makeatother

\newreptheorem{lemma}{Lemma}
\newreptheorem{theorem}{Theorem}
\newreptheorem{claim}{Claim}

\title{Accelerating Extreme Classification via Adaptive Feature Agglomeration}

\author{
Ankit Jalan \And Purushottam Kar\\
\affiliations
Department of CSE, IIT Kanpur, INDIA\\
\emails
aankitjalan@gmail.com, purushot@cse.iitk.ac.in
}

\begin{document}

\maketitle

\begin{abstract}
Extreme classification seeks to assign each data point, the most relevant labels from a universe of a million or more labels. This task is faced with the dual challenge of high precision and scalability, with millisecond level prediction times being a benchmark. We propose \defrag, an adaptive feature agglomeration technique to accelerate extreme classification algorithms. Despite past works on feature clustering and selection, \defrag distinguishes itself in being able to scale to millions of features, and is especially beneficial when feature sets are sparse, which is typical of recommendation and multi-label datasets. The method comes with provable performance guarantees and performs efficient task-driven agglomeration to reduce feature dimensionalities by an order of magnitude or more. Experiments show that \defrag can not only reduce training and prediction times of several leading extreme classification algorithms by as much as 40\%, but also be used for feature reconstruction to address the problem of missing features, as well as offer superior coverage on rare labels.
\end{abstract}

\section{Introduction}
\label{sec:intro}
The task of taking assigning data points, one or more labels from a vast universe of millions of labels is often referred to as the \emph{extreme classification} problem. Although reminiscent of the classical multi-label learning problem, the emphasis on addressing extremely large label spaces distinguishes extreme classification. Recent advances in extreme classification have allowed problems such as ranking, recommendation and retrieval to be viewed and formulated as multi-label problems, indeed with millions of labels.

This focus on extremely large label sets has given us state-of-the-art methods for product recommendation \cite{JainPV2016}, search advertising \cite{PrabhuKHAV2018}, and video recommendation \cite{WestonMY2013}, as well as led to advances in our understanding of scalable optimization \cite{PrabhuKHAV2018}, and distributed and parallel processing \cite{YenHDRDCX2017,BabbarS2017}. Recent advances have utilized a variety of techniques -- label embeddings, random forests, binary relevance, which we review in \S\ref{sec:related}.

Nevertheless, extreme classification algorithms continue to face several challenges that we enumerate below.

\paragraph{} 1) \textbf{Precision}: data points often have only 5-6 or fewer labels relevant to them (e.g. very few products, of the possibly millions on sale on an online marketplace, would interest any given customer). It is challenging to accurately identify these 5-6 relevant labels among the millions of irrelevant ones.\\
\indent 2) \textbf{Prediction}: given their use in (live) recommendation systems, extremely rapid predictions are expected, typically within milliseconds. This often restricts the algorithmic techniques that can be used, to computationally frugal ones.\\
\indent 3) \textbf{Processing}: extreme classification datasets contain not only millions of labels, but also millions of data points, each represented as a million-dimensional vector itself. It is challenging to offer scalable training on such large datasets.\\
\indent 4) \textbf{Parity}: huge label sets often exhibit power-law behavior with most labels being \emph{rare} i.e. relevant to very few data points. This makes it tempting for algorithms to focus only on popular labels, neglecting the vast majority of rare ones. However, this is detrimental for recommendation outcomes.

\paragraph{} In this work, we develop the \defrag method and variants to address these specific challenges for a large family of algorithms. Our contributions are summarized below.

\paragraph{Our Contributions.}\hspace*{3ex}\\
\indent 1) We propose the \defrag algorithm that accelerates extreme classification algorithms by performing efficient feature agglomeration on datasets with millions of features and data points. \defrag performs agglomeration by constructing a balanced hierarchy which novel, and offers faster and better agglomerates than traditional clustering methods.\\
\indent 2) We show that \defrag \emph{provably} preserves the performance of a large family of extreme classification algorithms. This is corroborated experimentally where using \defrag significantly reduces training and prediction times of algorithms but with no significant reduction in precision levels.\\
\indent 3) We exploit \defrag's agglomerates in a novel manner to develop the \refrag algorithm to address the parity problem by performing efficient label re-ranking. This vastly improves the coverage of existing algorithms by accurately predicting extremely rare labels.\\
\indent 4) We develop the \fiat algorithm to perform scalable feature imputation which preserves prediction accuracy even when a large fraction of data features are removed.\\
\indent 5) We perform extensive experimentation on large-scale datasets to establish that \defrag not only offers significant reductions in training and prediction times, but that it does so with little or no reduction in precision.

\newcommand{\hd}{\hat d}
\newcommand{\hL}{\hat L}

\section{Problem Formulation and Notation}
\label{sec:formulation}
The training data will be provided as $n$ labeled data points $(\vx^i, \vy^i), i = 1, \ldots, n$ where $\vx^i \in \bR^d$ is the feature vector and $\vy^i \in \bc{0,1}^L$ is the label vector. There may be several (upto $L$) labels associated with each data point. Extreme classification datasets exhibit extreme \emph{sparsity} in feature and label vectors. Let $\hd$ denote the average number of non-zero features per data point and $\hL$ denote the average number of active labels per data point. \S\ref{sec:exps} shows that $\hd \ll d$ and $\hL \ll L$. We will denote the feature matrix using $X = \bs{\vx^1, \ldots, \vx^n} \in \bR^{d \times n}$ and the label matrix using $Y = \bs{\vy^1, \ldots, \vy^n} \in \bc{0,1}^{L \times n}$.

\paragraph{Notation.} Let $\cF = \bc{F_1,\ldots,F_K}$ denote any $K$-partition of the feature set $[d]$ i.e. $F_i \cap F_j = \emptyset$ if $i \neq j$ and $\bigcup_{k=1}^KF_k = [d]$. Let $d_k := \abs{F_k}$ denote the size of the $k\nth$ cluster. For any vector $\vz \in \bR^d$, let $\vz_j$ denote its $j\nth$ coordinate. For any set $F_k \in \cF$, let $\vz_{F_k} := [\vz_j]_{j \in F_k}^\top \in \bR^{d_k}$ denote the (shorter) vector containing only coordinates from the set $F_k$.

\paragraph{Feature Agglomeration.} Feature agglomeration involves creating clusters of features and then summing up features within a cluster. If $\cF$ is a partition of the features $[d]$, then corresponding to every cluster $F_k \in \cF$, we create a a single ``super''-feature. Thus, given a vector $\vz \in \bR^d$, we can create an agglomerated vector $\tvz^{[\cF]} \in \bR^K$ (abbreviated to just $\tvz$ for sake of notational simplicity) with just $K$ features using the clustering $\cF$. The $k\nth$ dimension of $\tvz$ will be $\tvz_k = \sum_{j \in F_k}\vz_i$ for $k = 1, \ldots, K$. The \defrag algorithm will automatically learn relevant feature clusters $\cF$.

\section{Related Works}
\label{sec:related}
We discuss relevant works in extreme classification and scalable clustering and feature agglomeration techniques here.
\paragraph{Binary Relevance.} Also known as \emph{one-vs-all} methods, these techniques, for example \dismec \cite{BabbarS2017}, \ppds \cite{YenHDRDCX2017}, and ProXML \cite{BabbarS2019}, learn $L$ binary classifiers: for each label $l \in [L]$, a binary classifier is learnt to distinguish data points that contain label $l$ from those that do not. Binary relevance methods offer some of the highest precision values among extreme classification algorithms \cite{PrabhuKHAV2018}. However, despite advances in parallel training and active set methods, they still incur training and prediction times that are prohibitive for most applications.
\begin{figure}%
\includegraphics[width=\columnwidth]{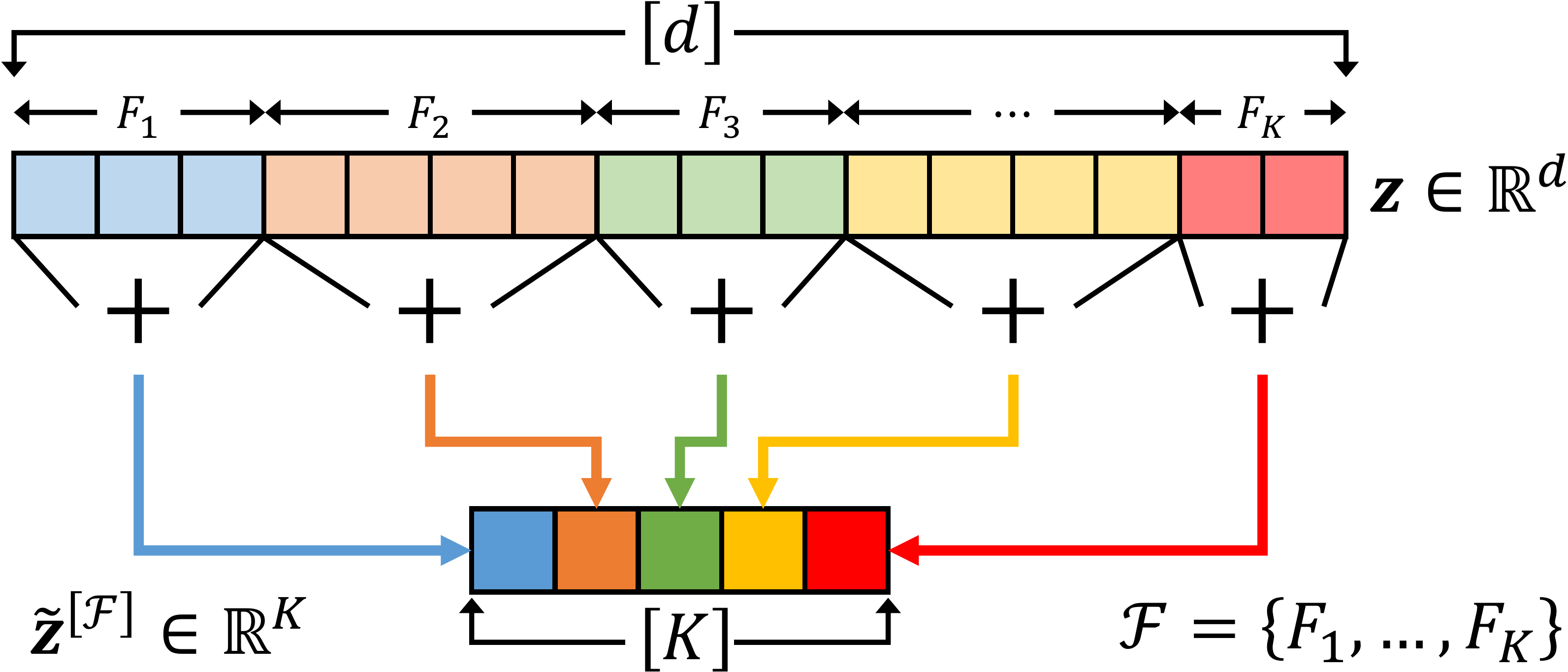}%
\caption{An illustration of the feature agglomeration process.}%
\label{fig:agglo}%
\end{figure}
\paragraph{Label/Feature Embedding.} These techniques project feature and/or label vectors onto a low dimensional space i.e. $\vx^i \mapsto \hat\vx^i, \vy^i \mapsto \hat\vy^i$ where $\hat\vx^i, \hat\vy^i \in \bR^p$, $p \ll \min\bc{d,L}$ using random or learnt projections. Prediction and training is performed in the low dimensional space $\bR^p$ for speed. These methods \sleec \cite{BhatiaJKVJ2015}, \annexml \cite{Tagami2017} and \leml \cite{YuJKD2014} offer strong theoretical guarantees, but are usually forced to choose a moderate value of $p$ to maintain scalability. This often results in low precision values and causes these methods to struggle on rare labels.
\paragraph{Data Partitioning.} These techniques learn a decision tree over the data points which are hierarchically clustered into several leaves, with the hope that the similar data points, i.e. those with similar label vectors, end up in the same leaf. A simple classifier (usually constant) performs label prediction at a leaf. These methods \pfrexml \cite{JainPV2016}, \fastxml \cite{PrabhuV2014} and \craftml \cite{SibliniKM2018} offer fast prediction times due to prediction being logarithmic in number of leaves in a balanced tree.
\paragraph{Label Partitioning.} These methods instead learn to organize labels into (overlapping) clusters, using hierarchical partitioning techniques. Prediction is done by taking a data point to one or more of the leaves of the tree and using a simple method such as 1-vs-all among labels present at that leaf. These methods \plt \cite{JasinskaDB-FPKH2016}, \parabel \cite{PrabhuKHAV2018}, \lpsr \cite{WestonMY2013} offer fast prediction times due to the tree structure, as well as high precision by using a 1-vs-all classifier at the leaves, as \parabel does.
\paragraph{Large Scale (Feature) Clustering.} Clustering, as well as feature clustering and agglomeration, are well-studied topics. Past works include techniques for scalable balanced k-means using alternating minimization techniques \scbc \cite{BanerjeeG2006} and \bcls \cite{LiuHNL2017}, scalable spectral clustering using landmarking \lsc \cite{ChenC2011}, and scalable information-theoretic clustering \itdc \cite{DhillonMK2003}. We do compare \defrag against all these algorithms. These algorithms were chosen since they were able to scale to at least the smallest datasets in our experiments.

\section{Adaptive Extreme Feature Agglomeration}
\label{sec:method}

\begin{algorithm}[t]
	\caption{\defrag: Make-Tree}
	\label{algo:defrag}
	\begin{algorithmic}[1]
		{\small
		\REQUIRE Feature set $S \subseteq [d]$, representative vectors $\vz^i \in \bR^p$ for each feature $i \in S$, maximum leaf size $d_0$
		\ENSURE A tree with each leaf having upto $d_0$ features%
		\IF{$|S| \leq d_0$}
			\STATE $\vn \leftarrow \text{Make-Leaf}(S)$ \COMMENT{No need to split this node}
		\ELSE
			\STATE $\vn \leftarrow \text{Make-Internal-Node}()$
			\STATE $\bc{S_+, S_-} \leftarrow \text{Balanced-Split}(S, \bc{\vz^i, i \in S})$\\\COMMENT{ Balanced spherical k-means or nDCG split}
			\STATE $\vn_+ \leftarrow \text{Make-Tree}(S_+, \bc{\vz^i, i \in S_+}, d_0)$
			\STATE $\vn_- \leftarrow \text{Make-Tree}(S_+, \bc{\vz^i, i \in S_-}, d_0)$
			\STATE $\vn.\text{Left-Child} \leftarrow \vn_+$
			\STATE $\vn.\text{Right-Child} \leftarrow \vn_-$
		\ENDIF
		\STATE \textbf{return} Root node of this tree $\vn$
		}
	\end{algorithmic}
\end{algorithm}%

We now describe the \defrag method, discuss its key advantages and then develop the \refrag method for rare label prediction and the \fiat method for feature imputation. Recall from \S\ref{sec:formulation} that given a $K$-partition $\cF$ of the features $[d]$, feature agglomeration takes each cluster $F_k \in \cF$ and agglomerates all features $j \in F_k$ by summing up their feature values.

\paragraph{\defrag: aDaptive Extreme FeatuRe AGglomeration.} Given a dataset with $d$-dimensional features, \defrag first clusters these features into balanced clusters, with each cluster containing, say no more than $d_0$ features. Suppose this process results in $K$ clusters. \defrag then uses feature agglomeration (see \S\ref{sec:formulation} and Figure~\ref{fig:agglo}) to obtain $K$-dimensional features for all data points in the dataset which are then used for training and testing.

\defrag first creates a representative vector for each feature $j \in [d]$ and then performs \emph{hierarchical clustering} on them (see Algorithm~\ref{algo:defrag}) to obtain feature clusters. At each internal node of the hierarchy, features at that node are split into two children nodes of equal sizes by solving either a balanced spherical $2$-means problem or else by minimizing a \emph{ranking loss} like nDCG \cite{PrabhuV2014} which we call \defragn (see Appendix~\ref{app:method} for details). This process is continued till we are left with less than $d_0$ features at a node, in which case the node is made a leaf. We now discuss two methods to construct these representative vectors.

\paragraph{\defragx} This variant clusters together \emph{co-occurrent} features e.g. $j,j' \in [d]$ where data points with a  non-zero (or high) value for feature $j$ also have a non-zero (or high) value for feature $j'$. \defragx represents each feature $j \in [d]$ as an $n$-dimensional vector $\vp^j = [\vx^1_j,\ldots,\vx^n_j]^\top \in \bR^n$, essentially as the list of values that feature takes in all data points.

\paragraph{\defragxy} This variant clusters together \emph{co-predictive} features e.g. $j,j' \in [d]$ where data points where feature $j$ is non-zero have similar labels as data points where feature $j'$ is non-zero. To do so, \defragxy represents each feature $j \in [d]$ as an $L$-dimensional vector $\vq^j = \sum_{i=1}^n\vx^i_j\vy^i \in \bR^L$, essentially as a weighted aggregate of the label vectors of data points where the feature $j$ is non-zero.



\paragraph{Suited for sparse, high-dim. data.} \defrag is superior to classical dimensionality reduction techniques like PCA/random projection for high-dimensional, sparse data.

1) Applying feature agglomeration to a vector simply involves summing up the coordinates of that vector and is much cheaper than performing PCA or a random projection.\\
\indent 2) PCA/random projection densify vectors and so methods such as \leml and \sleec are compelled to use a small embedding dimension ($\approx$ 500) for sake of scalability which leads to information loss. Feature agglomeration, however, does not densify vectors: if a vector $\vx \in \bR^d$ has only $s$ non-zero coordinates, then for any feature $K$-clustering $\cF$, the vector $\tvx^{[\cF]} \in \bR^K$ cannot have more than $s$ non-zero coordinates. This allows \defrag to operate with relatively large values of $K$ (e.g. $K = d/8$ is default in our experiments as we set $d_0 = 8$) without worrying about memory or time issues. Thus, \defrag can offer \emph{mildly} agglomerated vectors which preserve much of the information of the original vector, yet offer speedups due to the reduced dimensionality.\\
\indent 3) Feature agglomeration has an implicit \emph{weight-tying} effect since once we learn a model over the agglomerated features, all features belonging to a given cluster $F_k \in \cF$ effectively receive the same model weight. This reduces the capacity of the model and can improve generalization error.

\paragraph{Provably bounded distortion, Performance preservation.} We show in \S\ref{sec:analysis} that if we obtain a feature clustering $\cF$ with small clustering error, then feature agglomeration using $\cF$ \emph{provably} preserves the performance of \emph{all} linear models. Specifically, for every model $\vw \in \bR^d$ over the original vectors, there must exist a model $\tvw \in \bR^K$ over the agglomerated vectors such that for \emph{any vector} $\vx \in \bR^d$, we have $\vw^\top\vx \approx \tvw^\top\tvx^{[\cF]}$. This ensures that similar 1-vs-all models can be learnt over $\tvx^{[\cF]}$, as well as similar trees and label partitions can be built. We note that \emph{all} algorithms discussed in \S\ref{sec:related} ultimately use just linear models as components (e.g. binary relevance methods learn $L$ linear classifiers, embedding methods learn linear projections, data and label partitioning methods learn linear models to split internal nodes).



\paragraph{Task adaptivity.} \defragxy takes into account labels in its feature representation which makes it task-adaptive as compared to dimensionality reduction or clustering methods like k-means, PCA which do not take consider labels. Indeed, we will see that on many datasets, \defragxy outperforms \defragx which also does not take labels into account.

\paragraph{Novelty, Speed and Scalability.} Hierarchical \emph{feature} agglomeration is novel in the context of extreme classification although hierarchical \emph{data} partitioning (\parabel) and hierarchical \emph{label} partitioning (\pfrexml) have been successfully attempted before. The representative vectors created by \defrag are themselves sparse and hierarchical feature agglomeration offers speedy feature clustering. \defrag's overhead on the training process is thus, very small.

\paragraph{Time Complexity.} Let $\nnz(X) = n \cdot \hat d$ be the number of non-zero elements in the feature matrix $X$. Computing the feature representations $\vp^j, j \in [d]$ takes $\bigO{\nnz(X)}$ time. The total time taken to perform balanced spherical $2$-means clustering for all nodes at a certain level in the tree is $\bigO{\nnz(X)}$ as well. Since \defrag performs balanced splits, there can be at most $\bigO{\log d}$ levels in the tree, thus giving us a total time complexity of $\bigO{\nnz(X)\log d}$.







\paragraph{\fiat: Feature Imputation via AgglomeraTion.} Co-occurence based feature imputation has been popularly used to overcome the problem of missing features. However, this becomes prohibitive for extreme classification settings since the standard co-occurrence matrix $C = XX^\top$ is too dense to store and operate. We exploit the feature clusters offered by \defrag to create a scalable co-occurrence based feature imputation algorithm \fiat. For any feature cluster $F_k \in \cF$ let $X_{F_k} \in \bR^{d \times n}$ denote the matrix with only those rows that belong to the cluster $F_k$. Given this, we compute a pseudo co-occurrence matrix $C^{\cF} = \sum_{k=1}^KX_{F_k}X_{F_k}^\top \in \bR^d$.

Note that $C^{\cF}$ has a block-diagonal structure and has only upto $\frac{d^2}K$ non-zero entries where $K$ is the number of clusters. Thus, it is much cheaper to store and operate. Given a feature vector $\vx \in \bR^d$ that we suspect has missing features, we perform feature imputation on it by simply calculating $C^{\cF}\vx$.

\paragraph{\refrag: REranking via FeatuRe AGglomeration.} The presence of a vast majority of rare labels that occur in very few data points can cause algorithms to neglect rare labels in favor of popular ones \cite{WeiL2018}. To address this, we propose an efficient reranking solution based on the pseudo co-occurrence matrix $C^{\cF}$ described earlier. First compute the matrix product $C^{\cF}XY^\top \in \bR^{d \times L}$. The $l\nth$ column of this matrix $l \in [L]$ can be interpreted as giving us a \emph{prototype data point} $\vxi^l \in \bR^d$ for the label $l$.

These prototypes can be used to get the affinity score of a test data point $\vx^t$ to a label $l \in [L]$ as $e^{-\frac\gamma2\cdot\norm{\vx^t - \vxi^l}_2^2}$. Once a base classification algorithm such as \parabel or \dismec has given scores for the test point $\vx^t$ with respect to various labels, instead of predicting the labels with the highest scores right-away, we combine the classifier scores with these affinity scores and then make the predictions. We note that a similar approach was proposed by \cite{JainPV2016} who did achieve enhanced performance on rare labels.

However, whereas their method requires an optimization problem to be solved to obtain the prototypes, we have a closed form expression for prototypes in our model given the efficiently computable pseudo co-occurrence matrix $C^{\cF}$.

Due to lack of space, further algorithmic details as well as proofs of theorems in \S\ref{sec:analysis} are presented in the full version of the paper available at the URL given below.


\section{Performance Guarantees}
\label{sec:analysis}
In this section we establish that \defrag provably preserves the performance of extreme classification algorithms. For any vector $\vv \in \bR^p$ we will utilize the orthogonal decomposition $\vv = \vv^\parallel + \vv^\perp$ where $\vv^\parallel$ is the component of $\vv$ along the all-ones vector $\vone_p = (1,\ldots,1) \in \bR^p$ and $\vv^\perp$ is the component orthogonal to it i.e. $\vone_p^\top\vv^\perp = 0$. At the core of our results is the following lemma. Given a real valued matrix $Z \in \bR^{d \times p}$ for some $p > 0$ and a $K$-partition $\cF$ of the feature set $[d]$, we will let $Z_k \in \bR^{d_k \times p}$ denote the matrix formed out of the rows of the matrix that correspond to the partition $F_k$.

\begin{lemma}
\label{lem:approx}
Given any matrix $Z \in \bR^{d \times p}$ and any $K$-partition $\cF = \bc{F_1,\ldots,F_K}$ of $[d]$, suppose there exist vectors $\vmu^1, \ldots, \vmu^K \in \bR^p$ such that $Z_k = \vone_{d_k}(\vmu^k)^\top + \Delta_k$ where $\vone_{d_k} := (1,\ldots,1)^\top \in \bR^{d_k}$, then for \emph{every} $\vw \in \bR^d$ and every $k \in [K]$, there must exist a real value $c_{\vw,k} \in \bR$ such that
\begin{align*}
\resizebox{\linewidth}{!}{$
	\displaystyle
	(\vw_{F_k} - c_{\vw,k}\cdot\vone_{d_k})^\top Z_kZ_k^\top(\vw_{F_k} - c_{\vw,k}\cdot\vone_{d_k}) \leq \norm{\Delta_k^\top\vw_{F_k}^\perp}_2^2.
$}
\end{align*}%
\end{lemma}
Lemma~\ref{lem:approx} will be used to show below that, if a group of features is ``well-clustered'', then it is possible to tie together weights corresponding to those features in \emph{every} linear model.
\begin{theorem}
\label{thm:defragx}
Upon executing \defragx with a feature matrix $X = [\vx^1,\ldots,\vx^n]$ and label matrix $Y = [\vy^1,\ldots,\vy^n]$, suppose we obtain a feature $K$-partition $\cF = [F_1,\ldots,F_K]$ with $\err_k$ denoting the Euclidean clustering error within the $k\nth$ cluster, then for \emph{any} loss function $\ell(\cdot)$ that is $L$-Lipschitz and for \emph{every} linear model $\vw \in \bR^d$, there must exist a model $\tvw \in \bR^K$ such that for all subsets of data points $S \subseteq [n]$,
\[
\sqrt{\sum_{i \in S}\br{\ell(\vw^\top\vx^i; \vy^i) - \ell(\tvw^\top\tvx^i; \vy^i)}^2} \leq L\cdot\sum_{k=1}^K\norm{\vw_{F_k}^\perp}_2\cdot\err_k.
\]
\end{theorem}
To simplify this result, let $w_0 = \max_{k \in [K]}\norm{\vw_{F_k}^\perp}_2^2 \leq \max_{k \in [K]}\norm{\vw_{F_k}}_2^2$ (since cluster sizes $d_k$ are typically small, $w_0$ is small too) and use the fact that $\err_k \geq 0$ to get
\[
\sqrt{\sum_{i \in S}\br{\ell(\vw^\top\vx^i; \vy^i) - \ell(\tvw^\top\tvx^i; \vy^i)}^2} \leq L\cdot w_0\cdot\sum_{k=1}^K\err_k.
\]
A few points are notable about the above results.
\paragraph{Uniform Model Preservation.} Theorem~\ref{thm:defragx} guarantees that if the clustering error is small (and \defrag does minimize clustering error), then for \emph{every} possible linear model $\vw \in \bR^d$  over the original features $\vx^i$, we can learn a model $\tvw \in \bR^d$ over the agglomerated features $\tvx^i$ such that both models behave similarly with respect to any Lipschitz loss function. It is notable that Theorem~\ref{thm:defragx} holds \emph{simultaneously} for all linear models $\vw$, thus making it algorithm agnostic.
\paragraph{Classifier Preservation.} Most leading algorithms (\parabel, \dismec, \pfrexml, \ppds, \sleec) construct classifiers by learning several linear models using hinge loss or exponential loss which are Lipschitz. By preserving the performance of all such individual linear models, \defrag preserves the overall performance of these algorithms too. Note that Theorem~\ref{thm:defragx} holds uniformly over \emph{all} subsets $S \subseteq [n]$ of data points, which is useful since these algorithms often learn several linear models on various subsets of the data.
\paragraph{Graceful Adaptivity.} Suppose for a model $\vw$, the weights within a cluster $F_k$ are similar i.e. $\vw_{F_k} \approx w\cdot\vone_{d_k}, w \in \bR$. Then this implies $\vw_{F_k}^\perp \approx \vzero$ and the contribution of this cluster to the total error will be very small. This indicates that if some of the original weights are anyway tied together, \defrag automatically offers extremely accurate reconstructions.

In Appendix~\ref{app:proofs-defragxy}, we show that \defragxy preserves the performance of label clustering methods such as \parabel.


\section{Experimental Results}
\label{sec:exps}

\begin{table}
	\centering
	\begin{adjustbox}{max width=\linewidth}
	\begin{tabular}{cc}
    \begin{tabular}{l r r r r r r r}
    \toprule
	Method & LMI & Bal. & Ent. & Time & P1(\%) & P3(\%) & P5(\%)\\
	& & & & (min)\\
    \midrule
    \multicolumn{8}{c}{\textbf{ EURLex-4K }}\\
    \midrule
    ITDC & 0.47 & Inf & 0.87 & 0.7 & 73.85 & 59.63 & 49.19 \\
    SCBC & 0.39 & 321 & 0.75 & 0.55 & 71.96 & 59.50 & 49.41 \\
    LSC & 0.60 & 130 & 0.93 & 5.7 & 71.44 & 58.68 & 48.71 \\
    BCLS & 0.50 & Inf & 0.91 & 2.5 & 74.14 & 60.96 & 50.53 \\
    \defragx & \textbf{0.37} & \textbf{1.11} & \textbf{0.99} & \textbf{0.03} & \textbf{78.97} & \textbf{65.68} & \textbf{54.46} \\
    \toprule
    \multicolumn{8}{c}{\textbf{ Wiki10-31K }}\\
    \midrule
    ITDC & 0.52 & Inf & 0.88 & 23 & 82.03 & 69.72 & 60.07 \\
    \defrag-G & 0.46 & \textbf{1.08} & \textbf{0.99} & 0.48$^\ast$ & 82.72 & 69.62 & 60.23 \\
    \defragx & \textbf{0.36} & \textbf{1.08} & \textbf{0.99} & \textbf{0.45} & \textbf{84.99} & \textbf{73.47} & \textbf{63.91} \\
    \bottomrule
    \end{tabular}
	\end{tabular}
	\end{adjustbox}
	\caption{A comparison of \defrag with other clustering algorithms on clustering quality (see Appendix~\ref{app:exps} for definitions of clustering metrics), as measured by loss of mutual information (LMI), balance factor, normalized entropy, clustering time, and classification performance when the \parabel algorithm was executed upon agglomerated features given by the clustering algorithms. BCLS, LSC and SCBC could not scale to Wiki10. A balance factor of Inf indicates the presence of an empty cluster. \defrag not only outperforms other clustering algorithms in terms of clustering quality and classification accuracy, but offers clustering times that can be an order of magnitude smaller. \defrag-G denotes the \defragx algorithm executed on word features learnt by the GloVe algorithm \protect\cite{PenningtonSM2014}. \defrag could not be outperformed by carefully crafted word vector representations like GloVe either.\\$^\ast$The clustering time for \defrag-G does not include time taken to extract (dense) GloVe word features from raw text.}
	\label{tab:clustering}
\end{table}

We studied the effects of using \defrag variants with several extreme classification algorithms, as well as compared \defrag with other clustering algorithms. Our implementation of \defrag is available at the URL given below.

\noindent\textbf{Code Link}: \url{https://github.com/purushottamkar/defrag/}\\
\paragraph{Datasets and Implementations.} All datasets, train-test splits, and implementations of extreme classification algorithms were sourced from the Extreme Classification Repository \cite{XCRep2019} (see Table~\ref{tab:statistics} in the appendix). Implementations of clustering algorithms were sourced from the authors whenever possible. For \scbc, \lsc and \itdc, public implementations were not available and scalable implementations were created in the Python language.
\paragraph{Hyperparameters.} If available, hyperparameter settings recommended by authors were used for all methods. If unavailable, a fine grid search was performed over a reasonable range to offer adequately tuned hyperparameters to the methods. \defrag had its only hyperparameter, the max size of a feature cluster $d_0$ (see Algorithm~\ref{algo:defrag}), fixed to $8$.

\paragraph{Comparison with other clustering methods.} Table~\ref{tab:clustering} compares \defrag with other clustering algorithms on clustering quality, execution time and classification performance (see Appendix~\ref{app:exps} for definitions of clustering metrics). Features were agglomerated according to feature clusters given by all algorithms and \parabel was executed on them. \defrag handily outperforms all other methods.

\paragraph{Dataset-wise and Method-wise performance} Table~\ref{tab:metrics} presents the outcome of using \defrag with several leading algorithms on 8 datasets. On Wiki10 and Delicious, \defragxy{}+\parabel offers the best overall performance across all methods. More generally, the table shows 21 instances, across the 8 datasets, of how \defrag performs with various algorithms. In 3 of these instances, \defrag outperforms the base method (EURLex-\ppds, Wiki10-\parabel and Delicious-\parabel), in 7 others, \defrag lags by less than 2.5\%, in 8 others, it lags by less than 5\%. Only in 3 cases is the lag $>$ 5\%. \defrag variants do seem to work best with the \parabel method.

\paragraph{Trade-offs offered by \defrag.} It is easy to see that if we create a small number of clusters $K$, by setting $d_0$ to be a large value, then the agglomerated vectors will be lower dimensional and as such, offer faster training/prediction and smaller model sizes. However this may cause a dip in prediction accuracy. Figure~\ref{fig:fig1} shows that \defrag variants offer attractive trade-offs in this respect.

\paragraph{Rare-label prediction with \refrag.} Table~\ref{tab:prop-metrics} shows that \refrag offers much better propensity-weighted metrics \cite{JainPV2016} (which down-weigh popular and emphasize rare labels) than \pfrexml which also attempts label reranking. Figure~\ref{fig:fig2} shows that \refrag achieves much better coverage (3.85) than \parabel (1.23) on Delicious and in general, predicts rare labels far more accurately. Figure~\ref{fig:fig2} also shows that \fiat offers resilience to feature erasures.


\begin{figure}[t]%
\centering
\begin{subfigure}[b]{0.22\textwidth}
	\includegraphics[width=\textwidth]{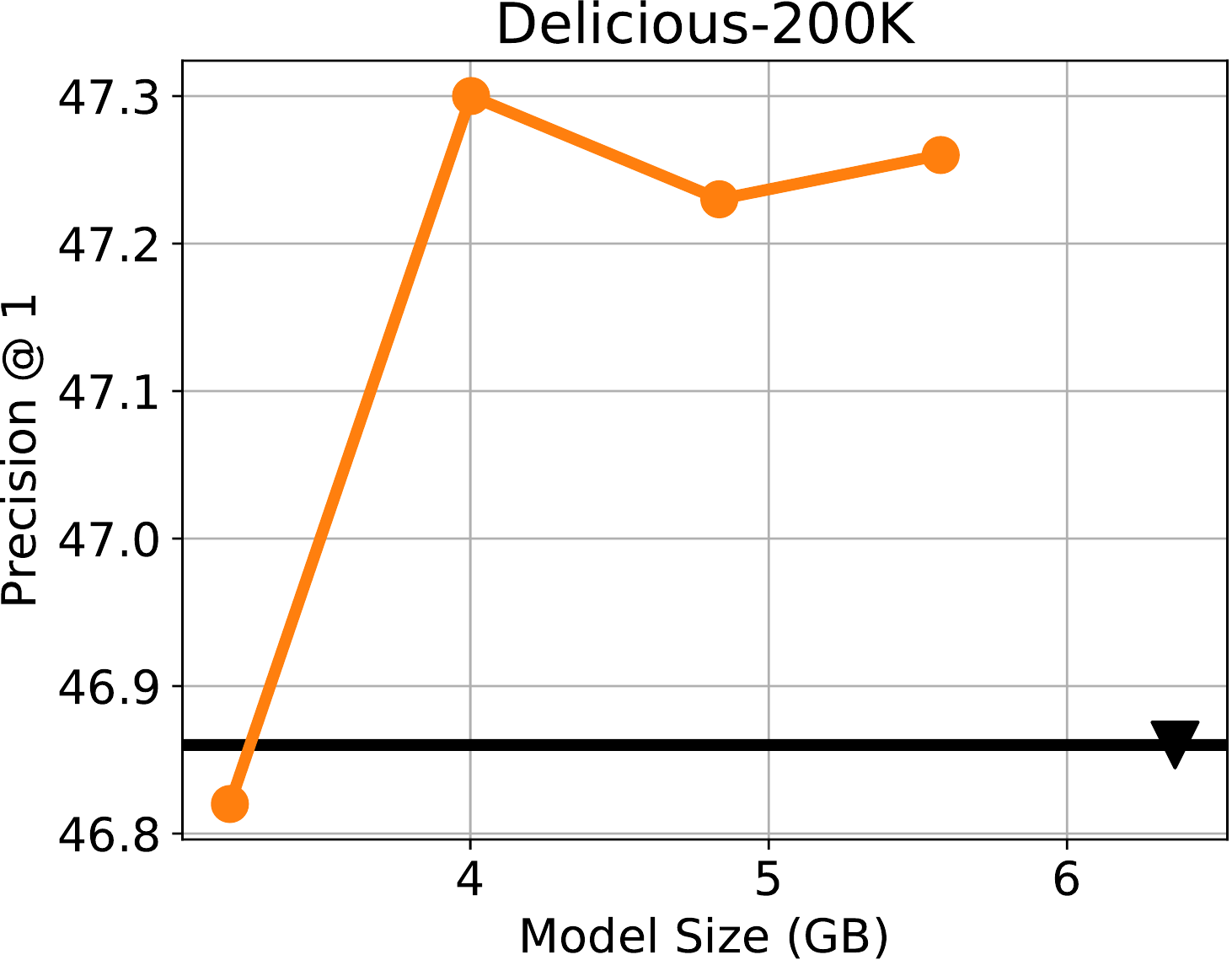}%
\end{subfigure}
\begin{subfigure}[b]{0.22\textwidth}
	\includegraphics[width=\textwidth]{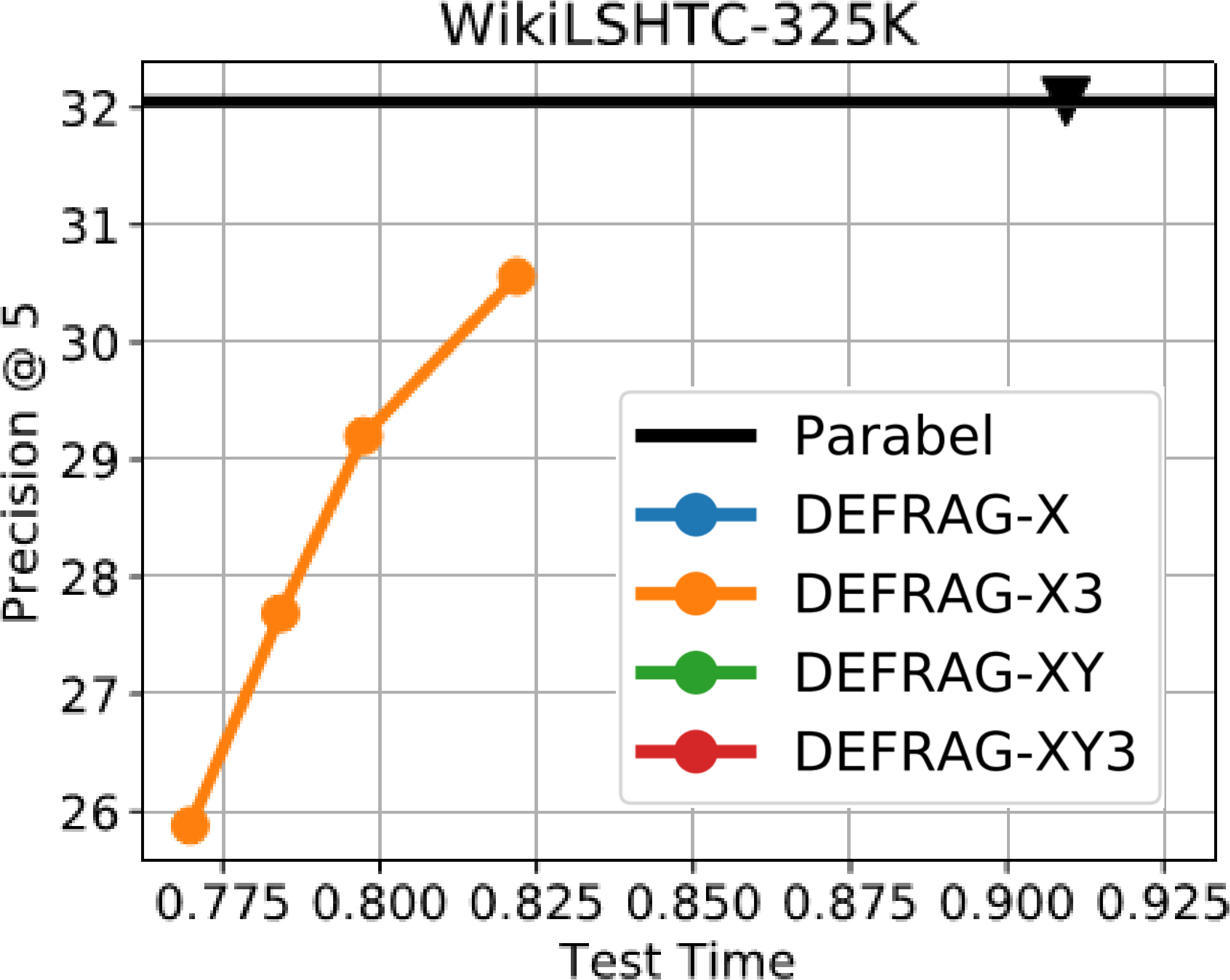}%
\end{subfigure}
\begin{subfigure}[b]{0.22\textwidth}
	\includegraphics[width=\textwidth]{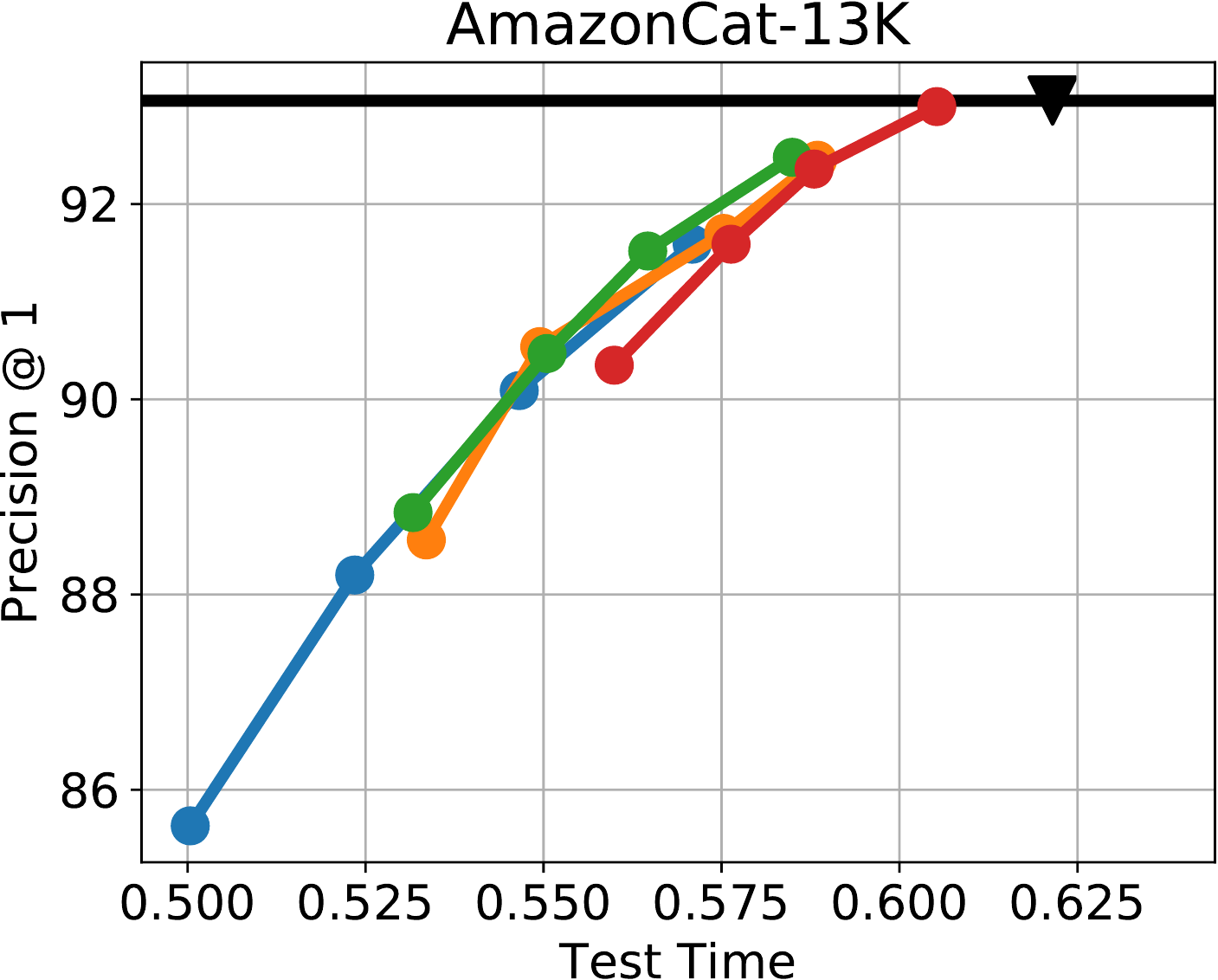}%
\end{subfigure}
\begin{subfigure}[b]{0.22\textwidth}
	\includegraphics[trim=0 0 165 0,clip,width=\textwidth]{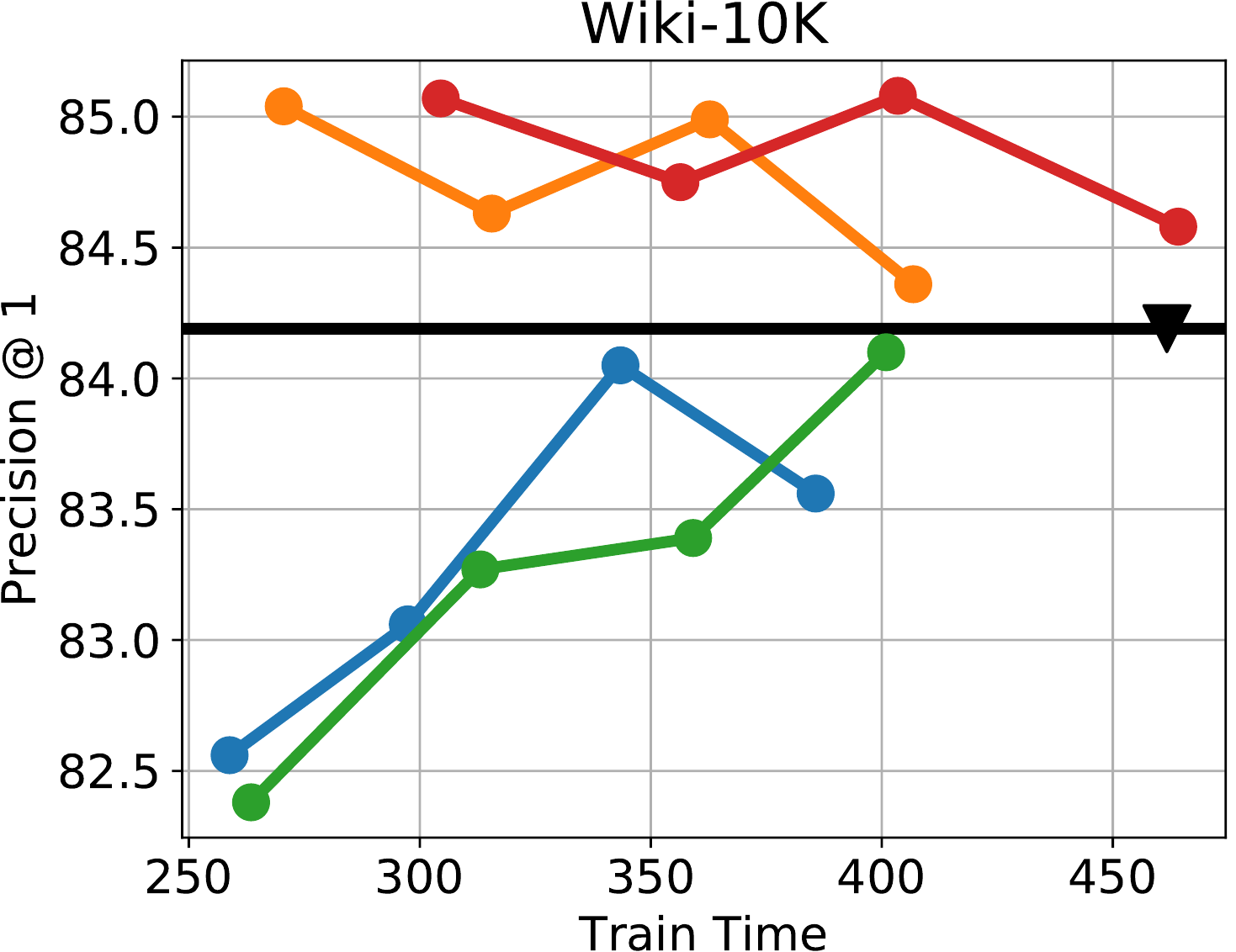}%
\end{subfigure}
\caption{Trade-offs offered by \defrag variants. The maximum size of clusters in \defrag variants $d_0$ was changed among the values $32, 16, 8 \text{ (default), and } 4$. The plots show how this affects prediction accuracy, training/test time and model sizes. The black line shows \parabel's default performance with the black triangle marking its model size, test time etc. Aggressive clustering (e.g. $d_0 = 32$) gives faster training/prediction times and smaller model sizes but also some drop in accuracy. \defragx{}3 and \defragxy{}3 refer to an ensemble of 3 independent realizations of \defrag (see Appendix~\ref{app:method-ensemble}). Appendix~\ref{app:exps-supp} presents several more trade-off plots.}%
\label{fig:fig1}%
\end{figure}

\begin{table}[t]
	\centering
	\begin{adjustbox}{max width=0.95\linewidth}
			\begin{tabular}{c r r r r r r}
				\toprule
				Method & P1 & P3 & P5 & N1 & N3 & N5 \\
				\midrule
				\multicolumn{7}{c}{\textbf{ Wiki-10K }}\\
				\midrule
				PfastReXML$^\|$      & 19.02 & 18.34 & 18.43 & 19.02 & 18.49	& 18.52	\\
				\refrag			& \textbf{20.56}& \textbf{19.51}& \textbf{19.26}&\textbf{20.56} & \textbf{19.74} & \textbf{19.54}	\\
				\midrule
				\multicolumn{7}{c}{\textbf{ Delicious-200K }}\\
				\midrule
				PfastReXML$^\|$      & 3.15 	& 3.87 	& 4.43 	& 3.15 	& 3.68 	& 4.06	\\
				\refrag			& \textbf{7.34}	& \textbf{8.05}	& \textbf{8.66} & \textbf{7.34} & \textbf{7.86} & \textbf{8.27}	\\
				\bottomrule
			\end{tabular}
	\end{adjustbox}
	\caption{\refrag with propensity scored metrics. N1,3,5 refer to propensity weighted nDGG@k. $^\|$Values from \protect\cite{XCRep2019}.}
\label{tab:prop-metrics}
\end{table}

\begin{figure*}[t]%
\centering
\begin{subfigure}[b]{0.245\textwidth}
	\includegraphics[width=\textwidth]{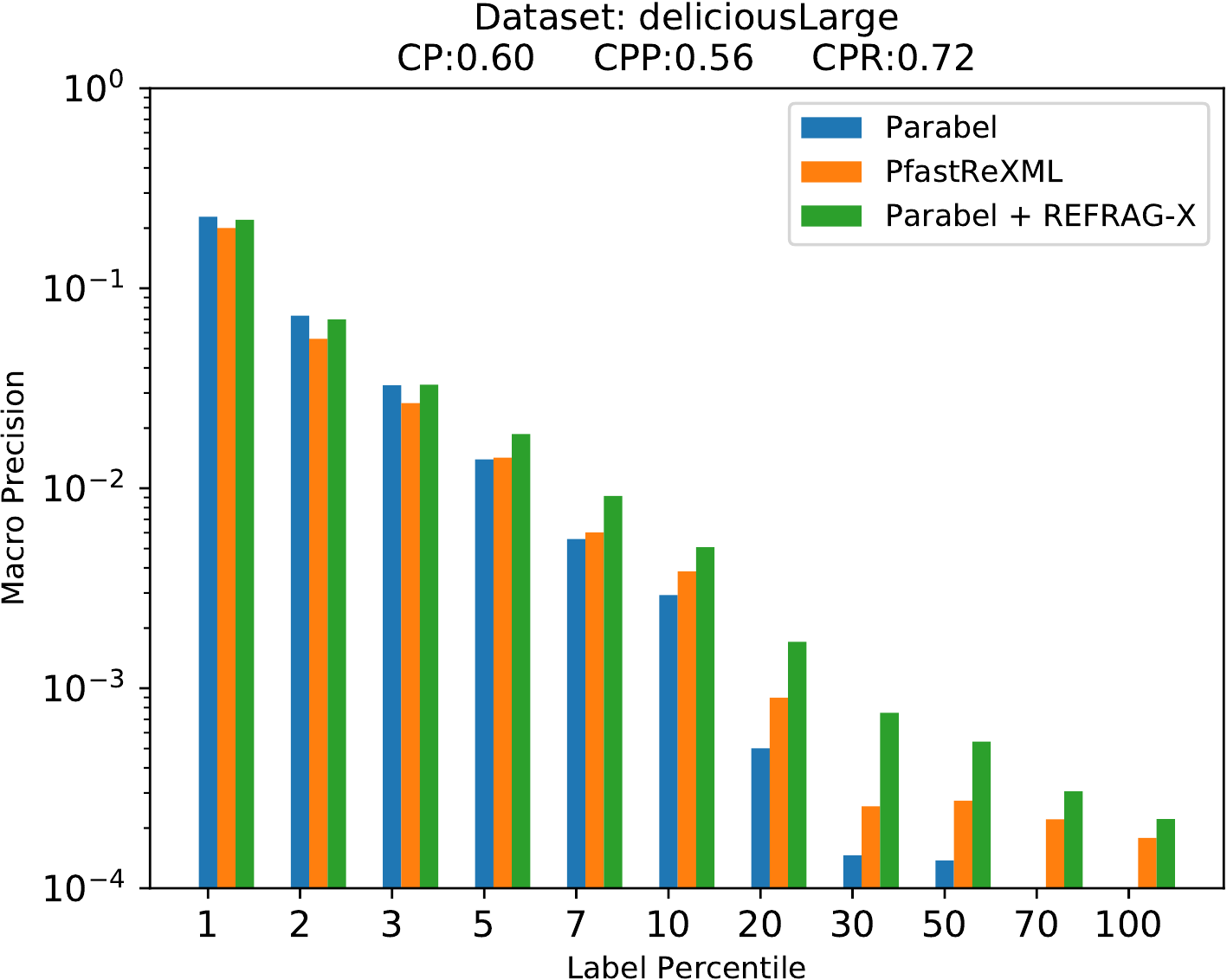}%
\end{subfigure}
\begin{subfigure}[b]{0.245\textwidth}
	\includegraphics[width=\textwidth]{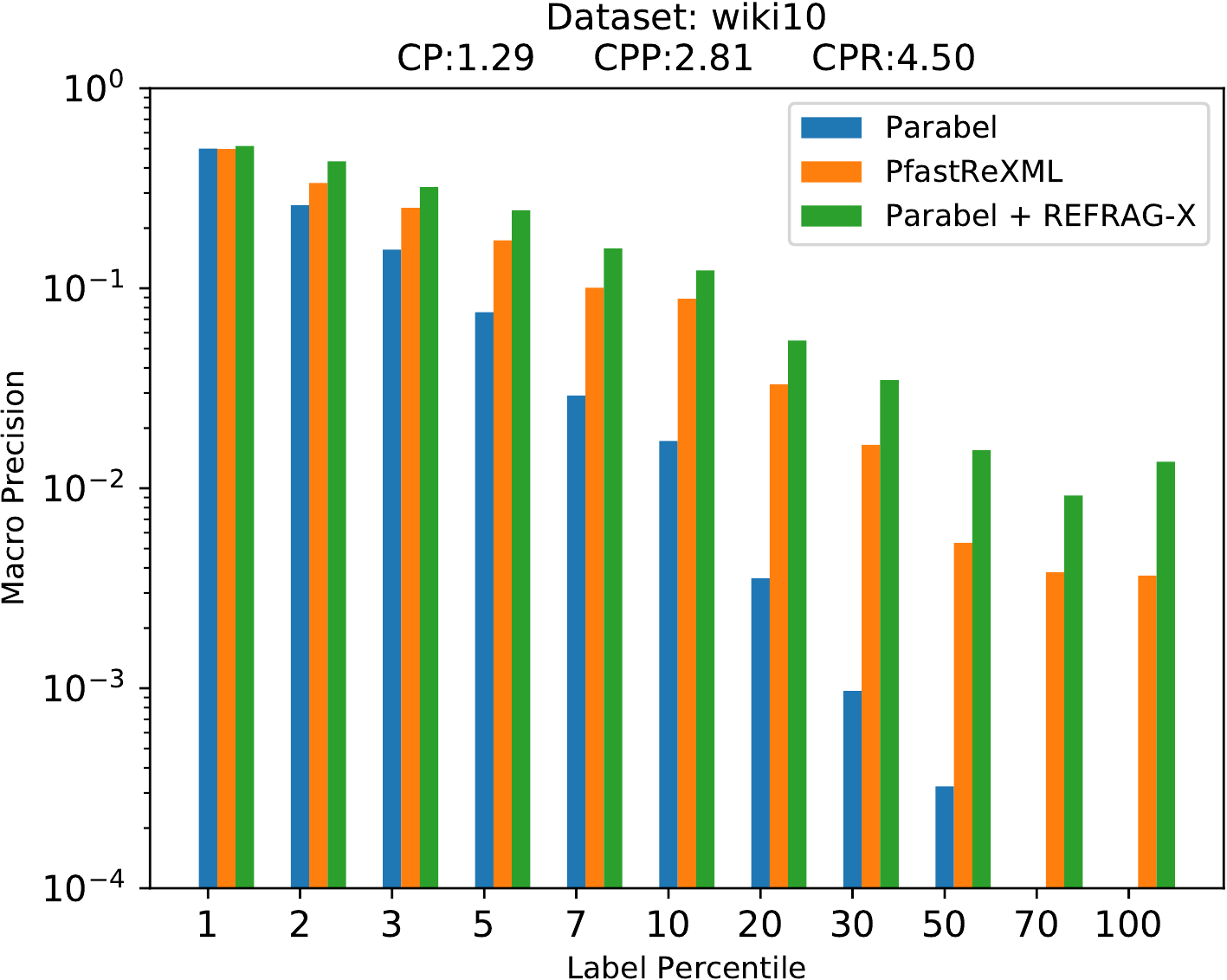}%
\end{subfigure}
\begin{subfigure}[b]{0.245\textwidth}
	\includegraphics[width=\textwidth]{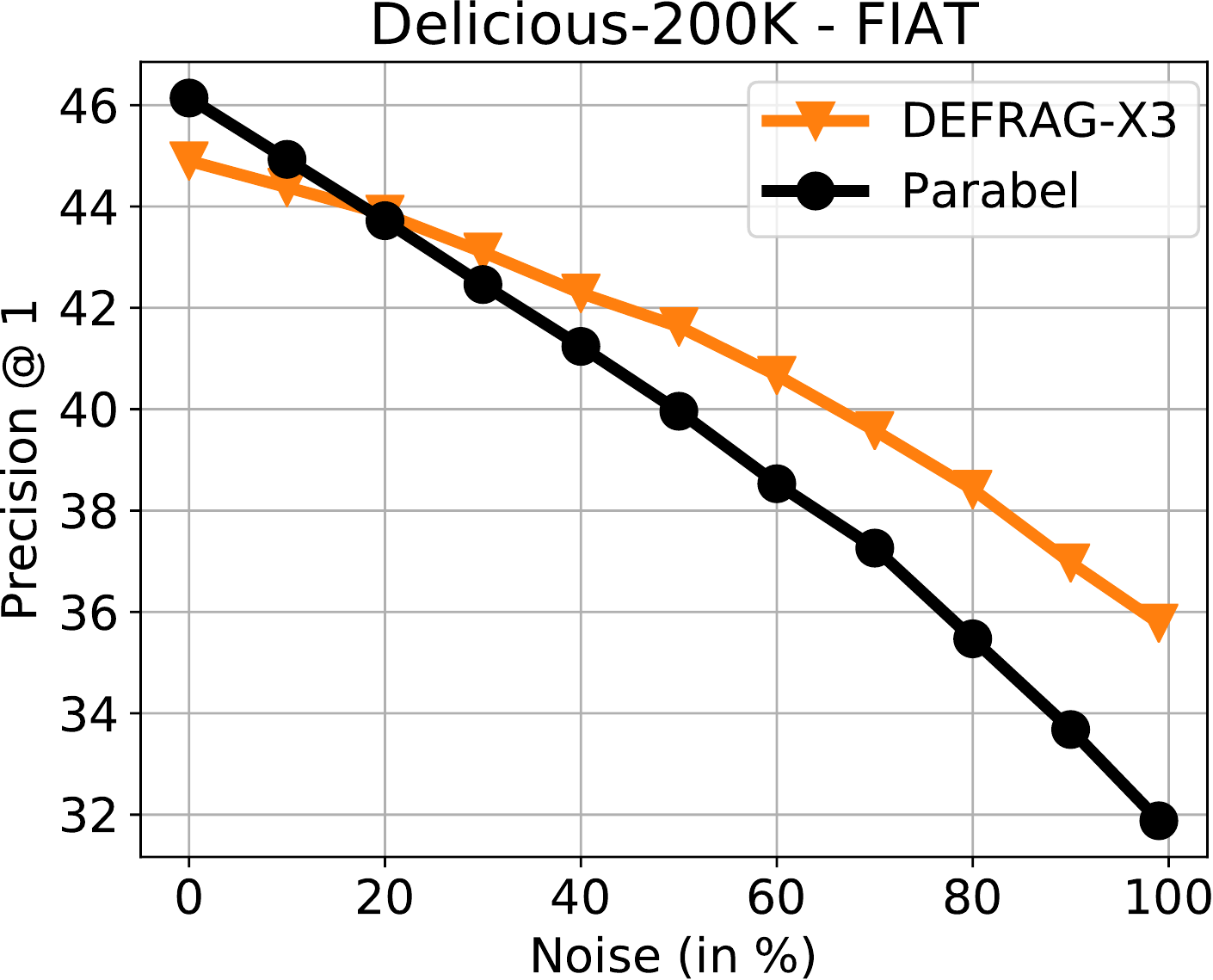}%
\end{subfigure}
\begin{subfigure}[b]{0.245\textwidth}
	\includegraphics[width=\textwidth]{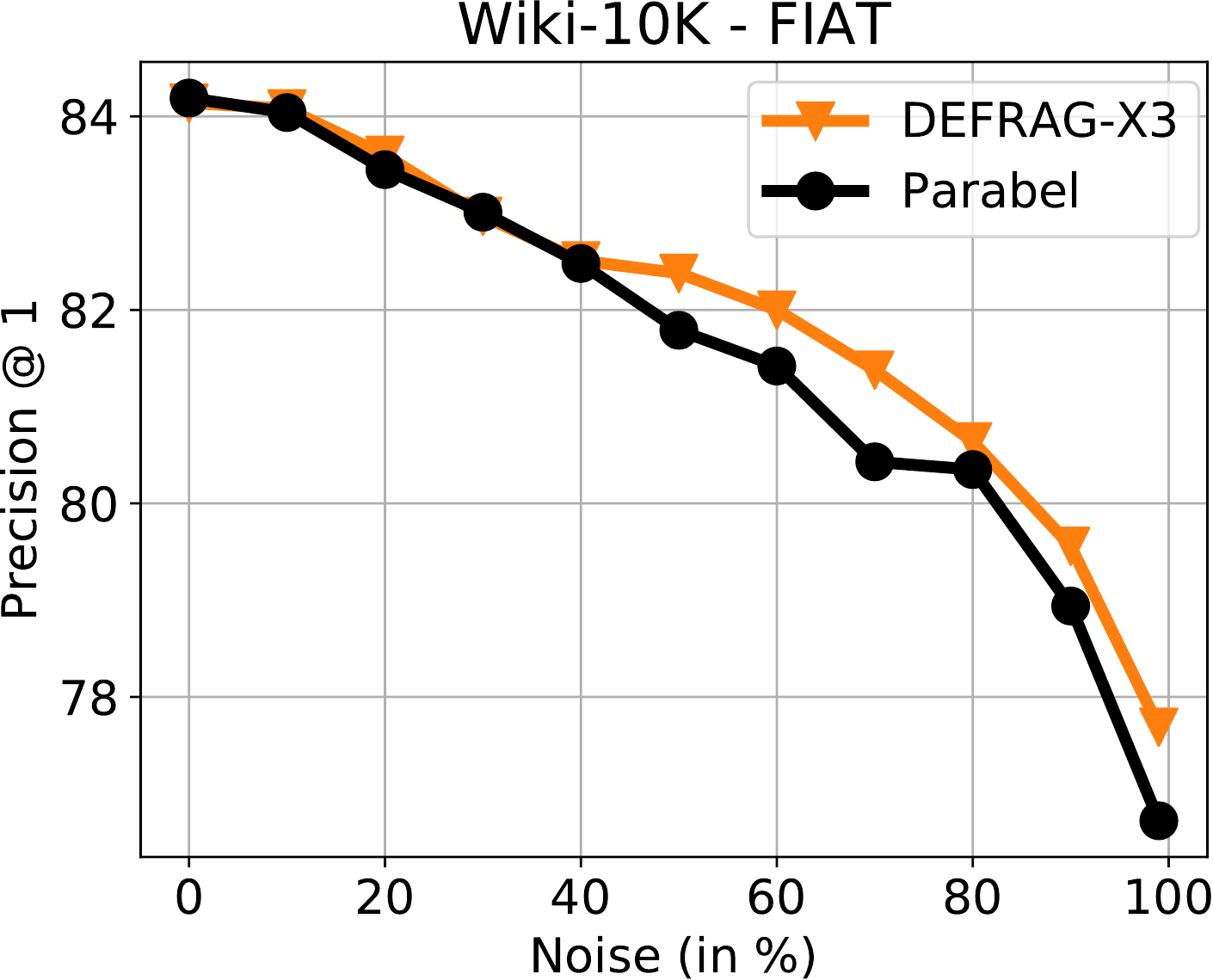}%
\end{subfigure}
\caption{\refrag offers far superior coverage of rare labels than \parabel on Wiki10 and Delicious datasets. Please see Figure~\ref{fig:app-rerank} and caption thereof in Appendix~\ref{app:exps-supp} for definitions of coverage and other details. In the last two plots, a fraction of features was randomly erased from the feature vectors of all test data points. The \fiat algorithm is more robust to such erasures than the default \parabel algorithm. The gap between \fiat and \parabel widens as erasures become more common. Please see Figure~\ref{fig:app-noise} in Appendix~\ref{app:exps-supp} for more details and plots.}%
\label{fig:fig2}%
\end{figure*}

\begin{table*}[t]
	\centering
	\begin{adjustbox}{max width=\textwidth}
  \begin{tabular}{cc}
    \begin{tabular}{c r r r r r r r}
    \toprule
        &  &  &  & Total & Train & Test & Model \\
    Method & P1(\%) & P3(\%) & P5(\%) & Time & Time & Time & Size \\
        &  &  &  & (hr) & (hr) & (ms) & (GB) \\
    
    \midrule
    \multicolumn{7}{c}{\textbf{ EURLex-4K }}\\
    \midrule
    PfastReXML      & 70.41 & 59.22 & 50.56 & 0.07 & 0.07 & 1.26 & 0.26\\
    \defragx        & 68.50 & 56.57 & 47.78 & 0.05 & 0.05 & 1.55 & 0.19\\
    \midrule
    SLEEC		    & 72.96 & 56.03 & 45.49 & 0.06 & 0.06 & 1.62 & 0.70\\
    \defragx        & 67.89 & 51.55 & 42.04 & 0.03 & 0.03 & 1.05 & 0.31\\
    \midrule
    Dismec          & 82.85 & 70.37 & 58.69 & 0.04 & 0.04 & 1.1 & 0.08\\
    \defragx        & 79.12 & 66.39 & 54.97 & 0.02 & 0.02 & 0.5 & 0.02\\
    \midrule
    PPDSparse       & 72.90 & 57.1 & 45.8 & 0.010 & 0.010 & 0.03 & 0.01\\
    \defragx        & 71.40 & \textbf{57.9} & \textbf{47.4} & 0.006 & 0.006 & 0.05 & 0.01\\
    \midrule
    Parabel         & 82.28 & 68.81 & 57.58 & 0.010 & 0.010 & 0.73 & 0.03\\
    \defragx        & 78.97 & 65.68 & 54.46 & 0.009 & 0.008 & 0.59 & 0.01\\
    \defragxy       & 79.23 & 65.77 & 54.65 & 0.009 & 0.008 & 0.59 & 0.01\\
    \midrule
    ProXML$^\ddagger$& 83.40	& 70.90	& 59.10 & -		& -		& -	   & - \\

    \midrule
    \multicolumn{7}{c}{\textbf{ Wiki10-31K }}\\
    \midrule
    PfastReXML      & 75.67 & 64.55 & 57.35 & 0.27 & 0.27 & 11.94 & 1.12\\
    \defragx        & 69.79 & 58.54 & 52.52 & 0.14 & 0.13 & 11.74 & 0.80\\
    \midrule
    SLEEC		    & 84.28 & 72.05 & 61.80 & 0.38 & 0.38 & 6.00 & 3.9\\
    \defragx        & 83.87 & 70.35 & 59.76 & 0.17 & 0.17 & 3.50 & 2.0\\
    \midrule
    Dismec          & 84.12 & 74.71 & 65.94 & 1.48 & 1.48 & 42 & 7.1\\
    \defragx        & 82.30 & 72.14 & 63.78 & 0.66 & 0.65 & 15 & 1.5\\
    \midrule
    PPDSparse       & 74.68 & 60.03 & 49.12 & 0.59 & 0.59 & 2.2 & 0.04\\
    \defragx        & 63.55 & 50.42 & 41.20 & 0.50 & 0.50 & 3.7 & 0.03\\
    \midrule
    Parabel         & 84.19 & 72.46 & 63.37 & 0.13 & 0.13 & 2.04 & 0.18\\
    \defragx        & 84.99 & 73.47 & 63.91 & 0.10 & 0.09 & 1.46 & 0.14\\
    \defragxy       & \textbf{85.08} & \textbf{73.76} & \textbf{64.06} & 0.11 & 0.09 & 1.47 & 0.14\\

	\midrule
	\multicolumn{7}{c}{\textbf{ Amazon-670K }}\\
	\midrule
	PfastReXML      & 36.90 & 34.22 & 32.10 & 5.70 & 5.70 & 6.10 & 10.98\\
	\defragx        & 32.67 & 30.27 & 28.40 & 2.70 & 2.69 & 7.21 & 9.40\\
	\midrule
	SLEEC		    & 32.48 & 28.87 & 26.31 & 2.22 & 2.22 & 1.43 & 8.0\\
	\defragx        & 31.40 & 28.04 & 25.69 & 1.63 & 1.63 & 1.62 & 4.2\\
	\midrule
	Parabel         & 44.92 & 39.77 & 35.98 & 0.24 & 0.24 & 0.81 & 1.94\\
	\defragx        & 42.71 & 37.71 & 33.93 & 0.23 & 0.21 & 0.76 & 1.68\\
	\defragxy       & 42.62 & 37.72 & 33.94 & 0.23 & 0.21 & 0.77 & 1.66\\
	\midrule
	\dismec$^\S$	& 44.70 & 39.70 & 36.10 & -	   & -	  &	-	 & -   \\
	ProXML$^\ddagger$& 43.50	& 38.70	& 35.30	& -	   & -	  &	-	 & -   \\

	\bottomrule
    \end{tabular}
    \begin{tabular}{c r r r r r r r}
    \toprule
        &  &  &  & Total & Train & Test & Model \\
    Method & P1(\%) & P3(\%) & P5(\%) & Time & Time & Time & Size \\
        &  &  &  & (hr) & (hr) & (ms) & (GB) \\

	\midrule
	\multicolumn{7}{c}{\textbf{ AmazonCat-13K }}\\
	\midrule
	PfastReXML      & 85.56 & 75.19 & 62.84 & 11.66 & 11.66 & 0.54 & 19.02\\
	\defragx        & 84.71 & 73.48 & 61.19 & 7.01 & 7.00 & 0.53 & 16.17\\
	\midrule
	Dismec$^\dagger$& 93.80 & 79.07 & 64.05 & 6.68 & 6.68 & 1.45 & 6.0\\
	\defragx        & 89.39 & 74.90 & 60.67 & 4.19 & 4.18 & 0.89 & 1.1\\
	\midrule
	Parabel         & 93.06 & 79.15 & 64.51 & 0.43 & 0.43 & 0.62 & 0.61\\
	\defragx        & 91.70 & 77.25 & 62.79 & 0.44 & 0.40 & 0.57 & 0.39\\
	\defragxy       & 92.36 & 78.20 & 63.55 & 0.43 & 0.40 & 0.58 & 0.38\\

    \midrule
    \multicolumn{7}{c}{\textbf{ Delicious-200K }}\\
    \midrule
    Parabel         & 46.86 & 40.08 & 36.70 & 5.33 & 5.33 & 2.22 & 6.36\\
    \defragx        & 47.23 & 40.53 & 37.19 & 3.23 & 3.16 & 1.05 & 4.83\\
    \defragxy       & \textbf{47.61} & \textbf{40.90} & \textbf{37.66} & 3.34 & 3.12 & 1.06 & 4.76\\
    \midrule
    PfastReXML$^\|$	& 41.72	& 37.83	& 35.58 & -		& -		& -	   & - \\
    \dismec$^\S$	& 45.50	& 38.70	& 35.50 & -		& -		& -	   & - \\

    \midrule
    \multicolumn{7}{c}{\textbf{ WikiLSHTC-325K }}\\
    \midrule
    PfastReXML      & 58.47 & 37.70 & 27.57 & 11.23 & 11.23 & 2.66 & 14.20\\
    \defragx        & 50.86 & 32.08 & 23.40 & 6.03 & 6.03 & 2.14 & 12.82\\
    \midrule
    Parabel         & 65.04 & 43.24 & 32.05 & 0.58 & 0.58 & 0.91 & 3.09\\
    \defragx        & 59.49 & 39.25 & 29.20 & 0.56 & 0.50 & 0.79 & 2.50\\
    \defragxy       & 61.38 & 40.42 & 29.99 & 0.54 & 0.50 & 0.78 & 2.44\\
    \midrule
    PPDSparse$^\|$	& 64.08 & 41.26 & 30.12 & -	& -	& -	& - \\
    \dismec$^\S$	& 64.40 & 42.50 & 31.50 & -	   & -	  &	-	 & -   \\
    ProXML$^\ddagger$& 63.60	& 41.50	& 30.80	& -	   & -	  &	-	 & -   \\

	\midrule
    \multicolumn{7}{c}{\textbf{Wikipedia-500K$^\ast$}}\\
    \midrule
    Parabel         & 68.70 & 49.57 & 38.64 & 5.13 & 5.13 & 3.11 & 5.68\\
    \defragx        & 65.15 & 44.96 & 34.85 & 3.27 & 3.14 & 1.62 & 5.25\\
    \defragxy       & 64.73 & 44.79 & 34.76 & 3.31 & 3.20 & 1.62 & 5.22\\
 	\midrule
	\dismec$^\S$	& 70.20 & 50.60 & 39.70 & -	   & -	  &	-	 & -   \\
	ProXML$^\ddagger$& 69.00	& 49.10	& 38.80	& -	   & -	  &	-	 & -   \\

	\midrule
	\multicolumn{7}{c}{\textbf{Amazon-3M$^\ast$}}\\
	\midrule
	Parabel         & 47.42 & 44.66 & 42.55 & 3.14 & 3.14 & 0.73 & 31.43\\
	\defragx        & 45.68 & 42.85 & 40.76 & 2.93 & 2.83 & 0.66 & 25.34\\
	\defragxy       & 45.11 & 42.36 & 40.30 & 2.87 & 2.80 & 0.63 & 25.22\\
	\midrule
	PfastReXML$^\|$	& 43.83	& 41.81	& 40.09 & -		& -		& -	   & - \\
    \bottomrule
    
    \end{tabular}
  \end{tabular}
	\end{adjustbox}
	\caption{\defrag's performance when used with various extreme classification algorithms. ``Total time'' for \defrag includes clustering=train time. On EURLex, Wiki10 and Delicious, \defrag actually achieves better classification accuracy than the base classifier itself (bold items). \defrag allows expensive 1-vs-all methods like \dismec and \ppds to be executed in a scalable manner with training time reductions of upto 40\% on AmazonCat and model size reductions of upto 20\% on WikiLSHTC. Precision values on certain datasets are being reported for sake of easy comparison. These were not obtained in our experiments and are being sourced from original publications: $^\S$\protect\cite{BabbarS2017}, $^\ddagger$\protect\cite{BabbarS2019}, $^\|$\protect\cite{XCRep2019}.\\
	$^\ast$ For all but Wiki-500K and Amazon-3M, \defrag was executed in an ensemble of 3 independent realizations (see Appendix~\ref{app:method-ensemble}).\\
	$^\dagger$ Except for \dismec on AmazonCat (which required 12 cores to execute scalably), all times are reported on a single core.}
	\label{tab:metrics}
\end{table*}

\section*{Acknowledgments}
The authors thank the reviewers for comments on improving the presentation of the paper. The authors are also thankful to the lab-team members at the CSE department, IIT Kanpur esp. M. Bagga, S. Malhotra, N. Yadav, B. K. Mishra for their support in running the experiments. P. K. thanks Microsoft Research India and Tower Research for research grants.

\clearpage

\balance
\bibliographystyle{named}
\bibliography{refs}

\appendix

\onecolumn

\allowdisplaybreaks

\begin{figure*}[t]
\begin{minipage}[c]{0.49\textwidth}
\begin{algorithm}[H]
	\caption{Balanced Spherical k-means}
	\label{algo:kmeans}
	{\small
	\begin{algorithmic}[1]
		\REQUIRE Feature set $S \subseteq [d]$, representative vectors for each feature $\vz^i \in \bR^p, i \in S$
		\ENSURE A balanced split of the feature set $S$
		\STATE Initialize centroids $\vc^+,\vc^-$ to two randomly selected representative vectors%
		\WHILE{not converged}
			\STATE Calculate scores $s_i = (\vc^+ - \vc^-)^\top\vz^i$ for each $i \in S$
			\STATE $S_1 \leftarrow$ the $\ceil{\abs{S}/2}$ features with highest scores
			\STATE $S_2 \leftarrow$ the $\floor{\abs{S}/2}$ features with lowest scores
			\STATE Recompute centroids $\vc^\pm = \frac1{\abs{S_\pm}}\sum_{i \in S_\pm}\vz^i$
		\ENDWHILE
		\STATE \textbf{return} $\bc{S_+,S_-}$
	\end{algorithmic}
	}
\end{algorithm}
\end{minipage}%
\hfill
\begin{minipage}[c]{0.49\textwidth}
\begin{algorithm}[H]
	\caption{Balanced nDCG Split}
	\label{algo:ndgc}
	{\small
	\begin{algorithmic}[1]
		\REQUIRE Feature set $S \subseteq [d]$, representative vectors for each feature $\vz^i \in \bR_+^p, i \in S$
		\ENSURE A balanced split of the feature set $S$
		\STATE Select two vectors $\vz^+, \vz^-$ randomly from representative vectors%
		\STATE Initialize centroids as $\vr^\pm \leftarrow \rank(\vz^\pm)$
		\WHILE{not converged}
			\STATE Calculate $s_i = \ndcg(\vr^+,\vz^i) - \ndcg(\vr^-,\vz^i)$ for $i \in S$
			\STATE $S_1 \leftarrow$ the $\ceil{\abs{S}/2}$ features with highest scores
			\STATE $S_2 \leftarrow$ the $\floor{\abs{S}/2}$ features with lowest scores
			\STATE Recompute centroids $\vr^\pm = \rank\br{\sum_{i \in S_\pm}I(\vz^i)\cdot\vz^i}$
		\ENDWHILE
		\STATE \textbf{return} $\bc{S_+,S_-}$
	\end{algorithmic}
	}
\end{algorithm}
\end{minipage}%
\end{figure*}

\section{Algorithmic Details from \S\ref{sec:method}}
\label{app:method}

In this section we present details of implementations of algorithms that were omitted from the main text due to lack of space. We start off with implementation details of the \defragx and \defragxy methods.

\subsection{\defrag Implementation Details}
As mentioned in the main text, \defrag chooses to represent each feature $j \in [d]$ as a vector, either as an $n$-dimensional vector $\vp^j = [\vx^1_j,\ldots,\vx^n_j]^\top \in \bR^n$ of the values that feature takes on in the $n$ data points (used by \defragx), or as an $L$-dimensional vector $\vq^j = \sum_{i=1}^n\vx^i_j\vy^i \in \bR^L$, essentially as a weighted aggregate of the label vectors of data points where the feature $j$ is non-zero.

Irrespective of the representation used, \defrag next performs hierarchical clustering on the representative vectors as described in Algorithm~\ref{algo:defrag}. Starting with the root node which contains all features, nodes are split evenly till the number of features falls below a set threshold at which point the node is made a leaf. For our experiments in Figure~\ref{fig:fig1}, we varied this threshold among the values $4,8,16,32$, thus obtaining respectively $d/4,d/8,d/16,d/32$ clusters, but for all other experiments in Figure~\ref{fig:fig2} and Tables~\ref{tab:clustering} and \ref{tab:metrics}, we fixed this threshold to $8$, thus obtaining $d/8$ clusters.

The clustering was done using one of two methods - balanced spherical k-means, or balanced nDCG splitting. The details of these implementations are given below. Given the extreme sparsity of both the data and label vectors, the feature representations we obtain via \defragx and \defragxy, i.e. $\vp^j, \vq^j$ are themselves very sparse and so case Euclidean notions of proximity such as dot products and norms may not be appropriate for clustering. Thus, we also develop a version which we call \defragn which minimizes an nDCG ranking loss at each node \cite{PrabhuV2014}. We found that balancing did not affect the performance of \defragn too much. This may be due to the fact that the minimizing the nDCG loss naturally produces rather balanced clusters, something that has been independently observed by \cite{JainPV2016,PrabhuV2014}.

\paragraph{Balanced Spherical k-means.} Algorithm~\ref{algo:kmeans} presents the node splitting routine using the balanced k-means algorithm. The algorithm essentially follows the traditional Lloyd's algortihm except at the cluster assignment stage when, instead of just assigning each point to the nearest cluster, the algorithm performs a fair split.

This algorithm can be derived as a special case (for $k$ = 2) of the refinement step in the constrained clustering routine proposed in \cite{BanerjeeG2006}. It is notable that \cite{PrabhuKHAV2018} derive essentially the same algorithm for splitting labels into balanced cluster, but they derive their approach starting from a different graph flow-based approach to constrained clustering.

\paragraph{nDCG Splitting.} Given that the vector representations of the features we use are sparse vectors, we employed a ranking-based splitting technique as well. The technique was adapted from the the work of \cite{PrabhuV2014} and is described here (their work applies the technique to cluster binary vectors whereas we apply it to cluster feature representative vectors which need not be binary). Note that the technique requires that the vector representations contain only non-negative values. This is indeed true in all our experiments since the features in our datasets are created out of bag-of-words representations which are indeed non-negative.

For any vector $\vv \in \bR_+^p$ with positive coordinates i.e. $\vv_i \geq 0, i \in [p]$, we let $\rank(\vv) \in S_p$ denote the permutation ranking the $p$ coordinates of $\vv$ in decreasing order i.e. if $\vr := \rank(\vv)$ then $\vv_{\vr_i} \geq \vv_{\vr_j}$ if $i > j$. For any positive vector $\vv \in \bR_+^p$ and any permutation $\vr \in S_p$ is any permutation (not necessarily the one that ranks the coordinates of $\vv$), then we define the \emph{Discounted Cumulative Gain} (DCG) score of the permuatation $\vr$ with respect to the vector $\vv$ as
\[
\dcg(\vr,\vv) := \sum_{j=1}^p\frac{\vv_{\vr_j}}{\log(1+j)}
\]
We also define the maximum such score any ranking can achieve as the following
\[
I(\vv) := \br{\max_{\vr \in S_p}\dcg(\vr,\vv)}^{-1} = \br{\dcg(\rank(\vv),\vv)}^{-1}
\]
Given the above the \emph{normalized} DCG score of any permutation $\vr \in S_p$ with respect to the vector $\vv$ as
\[
\ndcg(\vr,\vv) := I(\vv)\cdot\dcg(\vr,\vv)
\]
Now, given a set of vectors $\vv^1,\ldots,\vv^m \in \bR^p$, a single ``centroid'' ranking that fits all of them the best can be found as $\max_{\vr \in S_p} \sum_{i=1}^m\ndcg(\vr,\vv)$ where
\[
\sum_{i=1}^m\ndcg(\vr,\vv^i) = \sum_{i=1}^mI(\vv^i)\sum_{j=1}^p\frac{\vv^i_{\vr_j}}{\log(1+j)} = \sum_{j=1}^p\frac1{\log(1+j)}\sum_{i=1}^mI(\vv^i)\cdot\vv^i_{\vr_j}
\]
This implies that the best ranking is given by $\arg\max_{\vr \in S_p} \sum_{i=1}^m\ndcg(\vr,\vv) = \rank\br{\sum_{i=1}^mI(\vv^i)\cdot\vv^i_{\vr_j}}$. Algorithm~\ref{algo:ndgc} presents the \defragn clustering technique that  uses the above rule to recompute cluster centroids.

\subsection{Accelerated Clustering}
Given the large size of the datasets used in our experiments, it was important to ensure that the feature clustering time of \defrag did not exceed the savings in training time it offered for various methods, so as to ensure that the total training time of \defrag (clustering + agglomeration + training) still remained smaller than that of the various extreme classification algorithms.

To do so we notice that, given the heavy tailed phenomenon exhibited by most large-scale datasets, clustering time can be reduced significantly by subsampling data points and labels. More specifically, we notice that in the execution of \defragx, performance is not affected even if we represent each feature using only its values in $\tilde n$ most \emph{voluminous data points} where the ``volume'' of a data point $i \in [n]$ is calculated as $\norm{\vx^i}_1$. Similarly, we noticed that the execution of \defragxy is not affected, i.e. clusters are not affected, even if we take into account only the $\tilde L$ most popular labels. Such an effect (of performance not being affected by taking only ``head'' objects) has been observed before \cite{WeiL2018} as well.

Although the above steps do not greatly affect the clustering performance, they do drastically reduce the clustering time of the \defrag variants. Thus, all our \defragx experiments are executed taking only the $\tilde n = 0.25 n$ most voluminous documents. Thus, each feature is represented only as an $\tilde n$-dimensional vector. Similarly, all our \defragxy experiments are executed taking only the $\tilde n = 0.25 n$ most voluminous documents and the $\tilde L = 0.05 L$ most popular labels. Thus, each feature is represented only as an $\tilde L$-dimensional vector.

\subsection{Ensemble Training}
\label{app:method-ensemble}
Several extreme classification methods, such as \parabel, \sleec, \pfrexml construct ensemble classifiers by executing the algorithm independently a few times to obtain different classifiers and then using consensus voting techniques to aggregate the predictions of these different classifiers. Unless otherwise mentioned, we always executed \defrag variants afresh for each member of the ensemble as well.

For example, the \parabel algorithm trains and ensemle of 3 tree-based classifiers, each of which is independently capable of making classifications. We ran \defrag independently 3 times as well, once for each execution of \parabel. Since \defrag uses random initializations in its clustering routines, we found that the cluster partitions were not identical across the three executions. We did find this step to boost accuracy, presumably since it allowed different sets of features to be clustered together in different executions.

\subsection{Cluster Averaging}
We note that although \defrag simply sums up the feature values within a feature cluster, an alternative technique could be to average the feature values within the cluster. Although not a significant step in general, averaging does affect the performance of classifiers like \dismec or \ppds which practice \emph{model trimming} i.e. setting model coordinates which have a value below a certain threshold, to zero in order to save model space. For such classifiers, simple agglomeration may produce features with inflated feature values which result in small model values which in turn get trimmed to zero. Cluster averaging may help in these settings.

For instance, on the EURLex dataset with \dismec, \defrag with cluster averaging can yield more than 1.3\% boost in P1 accuracies. However, we note that this effect is not uniform and that on some datasets, averaging can actually hurt performance. For instance, on the Wiki10 dataset with \dismec, \defrag with cluster averaging actually causes minor dips of upto 0.4\% in P1.

\subsection{\refrag Implementation Details}
As mentioned in the main text, for a given test point $\vx^i \in \bR^d$, the pseudo co-occurrence model $C^{\cF}$ was used to obtain a score $a_l = e^{-\frac\gamma2\cdot\norm{\vx^t - \vxi^l}_2^2}$ indicating the affinity of the test point $\vx^t$ to the label $l$. The base algorithm, say \parabel or \ppds, was used to obtain a separate score, say $b_l$ indicating what did that algorithm think regarding the suitability of label $l$ for the data point $\vx^t$. These two scores were combined in the following manner.
\[
c_l = \alpha\log b_l + (1-\alpha)\log a_l,
\]
and the labels with the highest score, as assigned by $c_l$ were assigned to the data point $\vx^t$. As Figure~\ref{fig:fig2} indicates, this does not appreciably reduce the precision on popular labels as they keep getting predict as before. However, this does greatly increase the algorithm's ability to predict rare labels. We set $\alpha = 0.8$ for all experiments.

\section{Proofs from \S\ref{sec:analysis}}
\label{app:proofs}

We now provide complete proofs for the results mentioned in the main paper starting with the base lemma and the analysis for \defragx followed by the analysis for \defragxy.

\begin{replemma}{lem:approx}
\label{replem:approx}
Given any matrix $Z \in \bR^{d \times p}$ and any $K$-partition $\cF = \bc{F_1,\ldots,F_K}$ of $[d]$, suppose there exist vectors $\vmu^1, \ldots, \vmu^K \in \bR^p$ such that $Z_k = \vone_{d_k}(\vmu^k)^\top + \Delta_k$ where $\vone_{d_k} := (1,\ldots,1)^\top \in \bR^{d_k}$, then for \emph{every} $\vw \in \bR^d$ and every $k \in [K]$, there must exist a real value $c_{\vw,k} \in \bR$ such that
\begin{align*}
(\vw_{F_k} - c_{\vw,k}\cdot\vone_{d_k})^\top Z_kZ_k^\top(\vw_{F_k} - c_{\vw,k}\cdot\vone_{d_k}) \leq \norm{\Delta_k^\top\vw_{F_k}^\perp}_2^2.
\end{align*}%
\end{replemma}
\begin{proof}
Consider a fixed $k$ and for sake of notational simplicity, let us abbreviate $\vu := \vw_{F_k}, V := Z_k, c := c_{\vw,k}$ and $\vone := \vone_{d_k}$. We will show that if $V = \vone(\vmu^k)^\top + \Delta_k$ as promised, then there must exist a $c \in \bR$ such that $(\vu - c\cdot\vone)^\top VV^\top(\vu - c\cdot\vone) \leq \norm{\Delta_k^\top\vu^\perp}_2^2$. To establish this result, first notice that the objective function $f(c) = (\vu - c\cdot\vone)^\top VV^\top(\vu - c\cdot\vone)$ is minimized at a value of $c_{\min} = \frac{\vu^\top VV^\top\vone}{\vone^\top VV^\top\vone}$. At this value, we have
\[
f(c_{\min}) = (\vu^\top VV^\top\vu) - \frac{(\vu^\top VV^\top\vone)^2}{(\vone^\top VV^\top\vone)} = \frac{(\vu^\top VV^\top\vu)\cdot(\vone^\top VV^\top\vone) - (\vu^\top VV^\top\vone)^2}{(\vone^\top VV^\top\vone)}
\]
Concentrating just on the numerator, upon using the orthogonal decomposition $\vu = \vu^\parallel + \vu^\perp$ where $\vone^\top\vu^\perp = 0$ and letting $\vu^\parallel := p \cdot \vone$ for some $p \in \bR$, we get
\begin{align*}
\vu^\top VV^\top\vu &= p^2\cdot\vone^\top VV^\top\vone + 2p\cdot\vone^\top VV^\top\vu^\perp + (\vu^\perp)^\top VV^\top\vu^\perp\\
\vu^\top VV^\top\vone &= p\cdot\vone^\top VV^\top\vone + \vone^\top VV^\top\vu^\perp,
\end{align*}
which, upon inserting in the numerator expression, give us
\begin{align*}
&p^2\cdot(\vone^\top VV^\top\vone)^2 + 2p\cdot(\vone^\top VV^\top\vu^\perp)(\vone^\top VV^\top\vone) + ((\vu^\perp)^\top VV^\top\vu^\perp)(\vone^\top VV^\top\vone)\\
&-p^2\cdot(\vone^\top VV^\top\vone)^2 - 2p\cdot(\vone^\top VV^\top\vu^\perp)(\vone^\top VV^\top\vone) - (\vone^\top VV^\top\vu^\perp)^2\\
={}&((\vu^\perp)^\top VV^\top\vu^\perp)(\vone^\top VV^\top\vone) - (\vone^\top VV^\top\vu^\perp)^2.
\end{align*}
Thus, we have
\[
f(c_{\min}) = \frac{((\vu^\perp)^\top VV^\top\vu^\perp)(\vone^\top VV^\top\vone) - (\vone^\top VV^\top\vu^\perp)^2}{(\vone^\top VV^\top\vone)} = (\vu^\perp)^\top VV^\top\vu^\perp - \frac{(\vone^\top VV^\top\vu^\perp)^2}{(\vone^\top VV^\top\vone)}
\]
Note that the second term in the above expression is always non-negative, albeit not one that is easy to lower-bound. Thus, we simply bound the function value as
\[
f(c_{\min}) \leq (\vu^\perp)^\top VV^\top\vu^\perp = \norm{V^\top\vu^\perp}_2^2
\]
Now, the preconditions of the lemma guarantee us that $V = \vone(\vmu^k)^\top + \Delta_k$ and thus, we have
\[
V^\top\vu^\perp = \vmu^k\vone^\top\vu^\perp + \Delta_k^\top\vu^\perp = \Delta_k^\top\vu^\perp,
\]
where in the last step we have used the fact that $\vone^\top\vu^\perp = 0$ by construction. This finishes the proof.
\end{proof}

\begin{reptheorem}{thm:defragx}
\label{repthm:defragx}
Upon executing \defragx with a feature matrix $X = [\vx^1,\ldots,\vx^n]$ and label matrix $Y = [\vy^1,\ldots,\vy^n]$, suppose we obtain a feature $K$-partition $\cF = [F_1,\ldots,F_K]$ with $\err_k$ denoting the Euclidean clustering error within the $k\nth$ cluster, then for \emph{any} loss function $\ell(\cdot)$ that is $L$-Lipschitz, for \emph{every} linear model $\vw \in \bR^d$, there must exist a model $\tvw \in \bR^K$ such that for all subsets of data points $S \subseteq [n]$,
\[
\sqrt{\sum_{i \in S}\br{\ell(\vw^\top\vx^i; \vy^i) - \ell(\tvw^\top\tvx^i; \vy^i)}^2} \leq L\cdot\sum_{k=1}^K\norm{\vw_{F_k}^\perp}_2\cdot\err_k.
\]
\end{reptheorem}
\begin{proof}
We will describe how to obtain the identity of $\tvw$ in a short while. For now, notice that since all terms in the summation on the left hand side are non-negative, we have for any $S \subseteq [n]$
\[
\sum_{i \in S}\br{\ell(\vw^\top\vx^i; \vy^i) - \ell(\tvw^\top\tvx^i; \vy^i)}^2 \leq \sum_{i=1}^n\br{\ell(\vw^\top\vx^i; \vy^i) - \ell(\tvw^\top\tvx^i; \vy^i)}^2
\]
Using the Lipschitz property of the loss function and using the partition structure gives us
\[
\sum_{i=1}^n\br{\ell(\vw^\top\vx^i; \vy^i) - \ell(\tvw^\top\tvx^i; \vy^i)}^2 \leq L^2\cdot\sum_{i=1}^n\br{\vw^\top\vx^i - \tvw^\top\tvx^i}^2 = L^2\cdot\sum_{i=1}^n\br{\sum_{k=1}^K(\vw_{F_k} - \tvw_k\vone_{d_k})^\top\vx^i_{F_k}}^2
\]
Expanding the right hand side (and ignoring the $L^2$ term for now) gives us
\begin{align*}
&\sum_{i=1}^n\br{\sum_{k=1}^K\br{(\vw_{F_k} - \tvw_k\vone_{d_k})^\top\vx^i_{F_k}}^2 + \sum_{k \neq l}\br{(\vw_{F_k} - \tvw_k\vone_{d_k})^\top\vx^i_{F_k}}\br{(\vw_{F_l} - \tvw_l\vone_{d_l})^\top\vx^i_{F_l}}}\\
&= {\sum_{k=1}^K\sum_{i=1}^n\br{(\vw_{F_k} - \tvw_k\vone_{d_k})^\top\vx^i_{F_k}}^2 + \sum_{k \neq l}\sum_{i=1}^n\br{(\vw_{F_k} - \tvw_k\vone_{d_k})^\top\vx^i_{F_k}}\br{(\vw_{F_l} - \tvw_l\vone_{d_l})^\top\vx^i_{F_l}}}\\
&= {\sum_{k=1}^K(\vw_{F_k} - \tvw_k\vone_{d_k})^\top\bs{\sum_{i=1}^n\vx^i_{F_k}(\vx^i_{F_k})^\top}(\vw_{F_k} - \tvw_k\vone_{d_k}) + \sum_{k \neq l}(\vw_{F_k} - \tvw_k\vone_{d_k})^\top\bs{\sum_{i=1}^n\vx^i_{F_k}(\vx^i_{F_l})^\top}(\vw_{F_l} - \tvw_l\vone_{d_l})}
\end{align*}
Now, \defragx represents each feature $j \in [d]$ as an $n$ dimensional vector, $\vp^j = [\vx^1_j,\ldots,\vx^n_j]^\top \in \bR^n$ and then performs clustering on these vectors to obtain a $K$-partition of the feature set $[d]$, say $\cF = \bc{F_1,\ldots,F_k}$. Say the centroids of these $K$ clusters are $\vmu^1,\ldots,\vmu^K \in \bR^n$. Consider one of these clusters, say the $k\nth$ cluster $F_k$ with, say $d_k$ features in that cluster. If we denote $P_k = [\vp^j]_{j \in F_k}^\top \in \bR^{d_k \times n}$, then the following observations are immediate
\begin{enumerate}
	\item For any $\vv \in \bR^{d_k}$, we have $\vv^\top P_k = \vv^\top\sum_{i=1}^n\vx^i_{F_k}$
	\item The Euclidean clustering error within the $k\nth$ cluster is given by
	\[
		\err_k^2 := \sum_{j \in F_k} \norm{\vp^j - \vmu^k}_2^2 = \norm{P - \vone_{d_k}(\vmu^k)^\top}_F^2
	\]
	\item If we denote $P_k = \vone_{d_k}(\vmu^k)^\top + \Delta_k$, then we must have $\norm{\Delta_k}_F = \norm{P_k - \vone_{d_k}(\vmu^k)^\top}_F = \err_k$
	\item Lemma~\ref{lem:approx}, when applied with $p = n$ show us that with the above notation, we have, for some real value $c_{\vw,k}$
\[
\norm{P_k^\top(\vw_{F_k} - c_{\vw,k}\cdot\vone_{d_k})} \leq \norm{\Delta_k^\top\vw_{F_k}^\perp}_2 \leq \norm{\Delta_k^\top}_2\norm{\vw_{F_k}^\perp}_2 = \norm{\Delta_k}_2\norm{\vw_{F_k}^\perp}_2 \leq \norm{\Delta_k}_F\norm{\vw_{F_k}^\perp}_2 = \err_k\cdot\norm{\vw_{F_k}^\perp}_2,
\]
where $\norm{\Delta_k}_2$ denotes the spectral norm (i.e. the largest singular value) of the matrix $\Delta_k$ and $\norm{\Delta_k}_F$ denotes the Frobenius norm of the same matrix.

Note that although the inequality $\norm{\Delta_k}_2 \leq \norm{\Delta_k}_F$ may seem loose at first sight, notice that since $\rank(\Delta_k) \leq d_k$, we also have $\norm{\Delta_k}_F \leq \sqrt{d_k}\cdot\norm{\Delta_k}_2$ and since $d_k$ is typically a small number, this shows that this inequality is not too loose.
\end{enumerate}
The above observations allow us to construct $\tvw$. All we need to do is, for every cluster $k \in [K]$, consult Theorem~\ref{thm:defragx} to obtain a constant $c_{\vw,k}$ that offers the guarantees of the theorem. We then simply concatenate these constants to construct $\tvw = [c_{\vw,k}]_{k \in [K]} \in \bR^K$. Using the first point in the list of observations, we can now also rewrite the last expression in our ongoing calculations as
\[
{\sum_{k=1}^K(\vw_{F_k} - \tvw_k\vone_{d_k})^\top P_kP_k^\top(\vw_{F_k} - \tvw_k\vone_{d_k}) + \sum_{k \neq l}(\vw_{F_k} - \tvw_k\vone_{d_k})^\top P_kP_l^\top (\vw_{F_l} - \tvw_l\vone_{d_l})}
\]
Applying the Cauchy-Schwartz inequality and the rest of the observations allows us to upper-bound the above expression as
\begin{align*}
&\sum_{k=1}^K\norm{P_k^\top(\vw_{F_k} - \tvw_k\vone_{d_k})}_2^2 + \sum_{k \neq l}\norm{P_k^\top(\vw_{F_k} - \tvw_k\vone_{d_k})}_2\norm{P_l^\top (\vw_{F_l} - \tvw_l\vone_{d_l})}_2\\
&\leq \sum_{k=1}^K\err^2_k\cdot\norm{\vw_{F_k}^\perp}^2_2 + \sum_{k \neq l}\err_k\cdot\norm{\vw_{F_k}^\perp}_2\cdot\err_l\cdot\norm{\vw_{F_l}^\perp}_2 = \br{\sum_{k=1}^K\err_k\cdot\norm{\vw_{F_k}^\perp}_2}^2
\end{align*}
This finishes the proof upon putting back the $L^2$ factor we had omitted earlier and taking square roots on both sides.
\end{proof}

\subsection{Performance Guarantees for \defragxy}
\label{app:proofs-defragxy}
We now prove a similar result for \defragxy, specifically that \defragxy accurately preserves the performance of label clustering methods such as \parabel. Since \parabel performs ``spherical k-means'' which relies on scores of the form $(\vc^+ - \vc^-)^\top\vz^l$ to decide cluster assignments, the result below assures that these scores remain preserved on the agglomerated data as well. We stress that the above result can be readily adapted to usual k-means which relies on scores of the form $\norm{\vc^+ - \vz^l}_2^2 - \norm{\vc^- - \vz^l}_2^2$.

\begin{theorem}
\label{thm:defragxy}
Upon executing \defragxy with feature matrix $X$ and label matrix $Y$, suppose we obtain a feature $K$-partition $\cF$ with Euclidean clustering errors $\bc{\err_k}_{k \in [K]}$. Suppose $\vz^l = \sum_i\vy^i_l\vx^i$ and $\tvz^l = \sum_i\vy^i_l\tvx^i$ are the original and agglomerated label features for $l \in [L]$. Then for every $2$-means clustering of the original label features, with cluster centroids $\vc^+$ and $\vc^-$, there must exist centroids $\tvc^+$ and $\tvc^-$ that offer similar clustering error over the agglomerated features. Specifically, we have, for all subsets of labels $T \subseteq [L]$,
\[
\sqrt{\sum_{l \in [T]}\br{(\vc^+ - \vc^-)^\top\vz^l - (\tvc^+ - \tvc^-)^\top\tvz^l}^2} \leq \sum_{k=1}^K\err_k\cdot\norm{\vc^+_{F_k} - \vc^-_{F_k}}_2.
\]
\end{theorem}
\begin{proof}
We will establish the identity of the modified centroids $\tvc^+,\tvc^-$ in a short while. For now, notice that as before, by positivity of all terms in the summation, we have for any $T \subseteq [L]$
\[
\sum_{l \in [T]}\br{(\vc^+ - \vc^-)^\top\vz^l - (\tvc^+ - \tvc^-)^\top\tvz^l}^2 \leq \sum_{l=1}^L\br{(\vc^+ - \vc^-)^\top\vz^l - (\tvc^+ - \tvc^-)^\top\tvz^l}^2
\]
For sake of notational simplicity, let us denote $\vdelta=  \vc^+ - \vc^-$ and $\tvdelta=  \tvc^+ - \tvc^-$. This gives us
\[
\sum_{l=1}^L\br{\vdelta^\top\vz^l - \tvdelta^\top\tvz^l}^2 = \sum_{l=1}^L\br{\sum_{k=1}^K(\vdelta_{F_k} - \tvdelta_k\vone_{d_k})^\top\vz^l_{F_k}}^2
\]
Expanding the right hand side and expanding similarly as before gives us
\[
\sum_{k=1}^K(\vdelta_{F_k} - \tvdelta_k\vone_{d_k})^\top\bs{\sum_{l=1}^L\vz^l_{F_k}(\vz^l_{F_k})^\top}(\vdelta_{F_k} - \tvdelta_k\vone_{d_k}) + \sum_{k \neq j}(\vdelta_{F_k} - \tvdelta_k\vone_{d_k})^\top\bs{\sum_{l=1}^L\vz^l_{F_k}(\vz^l_{F_j})^\top}(\vdelta_{F_j} - \tvdelta_j\vone_{d_j})
\]
Now, \defragxy represents each feature $j \in [d]$ as an $L$ dimensional vector, $\vq^j = \sum_{i=1}^n\vx^i_j\vy^i \in \bR^L$ and then performs clustering on these vectors to obtain a $K$-partition of the feature set $[d]$, say $\cF = \bc{F_1,\ldots,F_k}$. Say the centroids of these $K$ clusters are $\vnu^1,\ldots,\vnu^K \in \bR^n$. Consider one of these clusters, say the $k\nth$ cluster $F_k$ with, say $d_k$ features in that cluster. If we denote $Q_k = [\vq^j]_{j \in F_k}^\top \in \bR^{d_k \times L}$, then the following observations are immediate
\begin{enumerate}
	\item For any $\vv \in \bR^{d_k}$, we have $\vv^\top Q_k = \vv^\top\sum_{l=1}^L\vz^l_{F_k}$
	\item The Euclidean clustering error within the $k\nth$ cluster is given by
	\[
		\err_k^2 := \sum_{j \in F_k} \norm{\vq^j - \vnu^k}_2^2 = \norm{Q - \vone_{d_k}(\vnu^k)^\top}_F^2
	\]
	\item If we denote $Q_k = \vone_{d_k}(\vnu^k)^\top + \Delta_k$, then we must have $\norm{\Delta_k}_F = \norm{Q_k - \vone_{d_k}(\vnu^k)^\top}_F = \err_k$
	\item Lemma~\ref{lem:approx}, when applied with $p = L$ show us that with the above notation, we have for some real value $d_{\vdelta,k}$
\[
\norm{Q_k^\top(\vdelta_{F_k} - d_{\vdelta,k}\cdot\vone_{d_k})} \leq \norm{\Delta_k^\top\vdelta_{F_k}^\perp}_2 \leq \norm{\Delta_k^\top}_2\norm{\vdelta_{F_k}^\perp}_2 = \norm{\Delta_k}_2\norm{\vdelta_{F_k}^\perp}_2 \leq \norm{\Delta_k}_F\norm{\vdelta_{F_k}^\perp}_2 = \err_k\cdot\norm{\vdelta_{F_k}^\perp}_2,
\]
where $\norm{\Delta_k}_2$ denotes the spectral norm (i.e. the largest singular value) of the matrix $\Delta_k$ and $\norm{\Delta_k}_F$ denotes the Frobenius norm of the same matrix.

Again, note that the inequality $\norm{\Delta_k}_2 \leq \norm{\Delta_k}_F$ is not extremely loose since, $\rank(\Delta_k) \leq d_k$ gives us $\norm{\Delta_k}_F \leq \sqrt{d_k}\cdot\norm{\Delta_k}_2$ and since $d_k$ is typically a small number, this shows that this inequality is not too loose.

\end{enumerate}
The above observations allow us to construct $\tvdelta$. All we need to do is, for every cluster $k \in [K]$, consult Theorem~\ref{thm:defragx} to obtain a constant $d_{\vdelta,k}$ that offers the guarantees of the theorem. We then simply concatenate these constants to construct $\tvdelta = [d_{\vdelta,k}]_{k \in [K]} \in \bR^K$. Once we have $\tvdelta$, we may construct $\tvc^+$ and $\tvc^-$ as any two vectors such that $\tvc^+ - \tvc^- = \tvdelta$. Moreover, using the first point in the list of observations, we can now also rewrite the last expression in our ongoing calculations as
\[
{\sum_{k=1}^K(\vdelta_{F_k} - \tvdelta_k\vone_{d_k})^\top Q_kQ_k^\top(\vdelta_{F_k} - \tvdelta_k\vone_{d_k}) + \sum_{k \neq j}(\vdelta_{F_k} - \tvdelta_k\vone_{d_k})^\top Q_kQ_j^\top (\vdelta_{F_j} - \tvdelta_j\vone_{d_j})}
\]
Applying the Cauchy-Schwartz inequality and the rest of the observations allows us to upper-bound the above expression as
\begin{align*}
&\sum_{k=1}^K\norm{Q_k^\top(\vdelta_{F_k} - \tvdelta_k\vone_{d_k})}_2^2 + \sum_{k \neq j}\norm{Q_k^\top(\vdelta_{F_k} - \tvdelta_k\vone_{d_k})}_2\norm{Q_j^\top (\vdelta_{F_j} - \tvdelta_j\vone_{d_j})}_2\\
&\leq \sum_{k=1}^K\err^2_k\cdot\norm{\vdelta_{F_k}^\perp}^2_2 + \sum_{k \neq j}\err_k\cdot\norm{\vdelta_{F_k}^\perp}_2\cdot\err_j\cdot\norm{\vdelta_{F_j}^\perp}_2 = \br{\sum_{k=1}^K\err_k\cdot\norm{\vdelta_{F_k}^\perp}_2}^2
\end{align*}
This finishes the proof upon taking square roots on both sides.
\end{proof}

\section{Experimental Details from \S\ref{sec:exps}}
\label{app:exps}
We provide additional details about experimental settings, as well as additional experimental results in this appendix.

\begin{table}[t]
\centering
\newcommand{\statfont}[1]{\normalsize{#1}}
\renewcommand\multirowsetup{\centering}
\begin{adjustbox}{max width=\linewidth}
\begin{tabular}{lrrrrrr}
\toprule
\multicolumn{1}{c}{\multirow{2}{*}{Data set}} & \multicolumn{1}{c}{Train} & \multicolumn{1}{c}{Features} & \multicolumn{1}{c}{Labels} & \multicolumn{1}{c}{Test} & \multicolumn{1}{c}{} & \multicolumn{1}{c}{}  \\
        &  \multicolumn{1}{c}{$n$}    &  \multicolumn{1}{c}{$d$}      & \multicolumn{1}{c}{$L$}     & \multicolumn{1}{c}{}    &     \multicolumn{1}{c}{$\hat d$}     & \multicolumn{1}{c}{$\hat L$}             \\
\midrule
EURLex-4K & 15539 & 5000 & 3993 & 3809 & 236.8 & 5.31\\
AmazonCat-13K & 1186239 & 203882 & 13330 & 306782 & 71.2 & 5.04\\
Wiki10-31K & 14146 & 101938 & 30938 & 6616 & 673.4 & 18.64\\
Delicious-200K & 196606 & 782585 & 205443 & 100095 & 301.2 & 75.54\\
WikiLSHTC-325K & 1778351 & 1617899 & 325056 & 587084 & 42.1 & 3.19\\
Wikipedia-500K & 1813391 & 2381304 & 501070 & 783743 & 385.3 & 4.77\\
Amazon-670K & 490449 & 135909 & 670091 & 153025 & 75.7 & 5.45\\
Amazon-3M & 1717899 & 337067 & 2812281 & 742507 & 49.7 & 36.17\\
\bottomrule
\end{tabular}
\end{adjustbox}
\caption{Dataset Statistics}
\label{tab:statistics}
\end{table}

\subsection{Clustering Metrics}
\label{app:clustering}
We used several clustering metrics to evaluate the various clustering algorithms compared in Table~\ref{tab:clustering}. Those metrics are defined formally below. Note that in that experiment, to be fair, all algorithms were asked to output the same number of clusters $d/8$ where $d$ is the original dimensionality of the data features.

The notion of mutual information loss (LMI) is defined in \cite{DhillonMK2003} for multi-class classification problems and adapted here for our setting of multilabel classification problems. LMI measures the loss in predictive capability due to clustering of features. The LMI score is a positive real value between $0$ and $1$. Thus, a smaller value of the LMI metric is better. It is notable that in Table~\ref{tab:clustering}, \defrag achieved the lowest LMI score of 0.37 followed by \scbc which achieved an LMI score of 0.39. All other algorithms acheived an LMI score of around 0.50 or more.

\begin{definition}[Mutual Information Loss (LMI)]
Given a $K$-clustering $\cF$ of $[d]$ features, let $X = [\vx^1,\ldots,\vx^n] \in \bR^{d\times n}$ denote the original feature matrix, $\tilde X = [\tvx^1,\ldots,\tvx^n] \in \bR^{K\times n}$ denote the matrix of agglomerated features and $Y = [\vy^1,\ldots,\vy^n] \in \bc{0,1}^{L \times n}$ denote the label matrix. Then the normalized loss in mutual information is defined as
\[
\text{LMI}(\cF) = \frac{I(Y;X) - I(Y;\tilde X)}{I(Y;X)},
\]
where $I(Y;X)$ is the mutual information between the labels and the original features and $I(Y;\tilde X)$ is the mutual information between the agglomerated features. These terms are defined below for the multi-label learning setting.
\end{definition}

\newcommand{\vPi}{\vec{\Pi}}

\begin{definition}[Mutual Information (MI)]
Given a feature matrix $Z = [\vz^1,\ldots,\vz^n] \in \bR^{p \times n}$ with each data point having $p$ features and $Y = [\vy^1,\ldots,\vy^n] \in \bc{0,1}^{L \times n}$ denote the label matrix. Then the mutual information between the labels and the features is defined as
\[
I(Y;Z) = \trace(\vec{\Pi}^\top(\ln(\vec{\Pi})- \ln(\vpi_p)\vone_L^\top - \vone_p\ln(\vpi_L)^\top)) 
\]
where $\vPi = \frac{ZY^\top}{\vone_p^\top ZY^\top\vone_L} \in \Delta_{(p \times L) - 1}$ is the joint feature-label probability matrix, $\ln(\vPi) \in \bR^{p \times L}$ denotes the matrix obtained by taking entry wise logarithm of the $\vPi$ matrix, $\vpi_p = \frac{Y\vone_n}{\norm{Y\vone_n}_1} \in \Delta_{p-1}$ is the feature probability vector, $\vpi_L = \frac{Z\vone_n}{\norm{Z\vone_n}_1} \in \Delta_{L-1}$ is the label probability vector, $\ln(\vpi_p)$ and $\ln(\vpi_L)$ are obtained by taking element-wise logarithm on the vectors $\vpi_p$ and $\vpi_L$ respectively, and $\vone_k = (1,\ldots,1)^\top \in \bR^k$ for any $k > 0$. The above definition assumes that the feature matrix $Z$ contains no negative values. However, this is true for bag-of-words representations used in extreme classification problems.
\end{definition}

The notion of balance factor is a rather unforgiving metric measuring how balanced are the clusters output by a clustering technique. It is frequently seen that imbalanced feature clusters can lead to reduction in classification performance. The balance factor is positive real number greater than or equal to one. A smaller balance factor indicates a more balanced clustering with a balance factor of $1$ denoting perfectly balanced clusters. It is notable that in Table~\ref{tab:clustering}, \defrag achieved the lowest balance factor of 1.11 and 1.08.
\begin{definition}[Balance Factor]
Given a $K$-clustering $\cF = [F_1,\ldots,F_K]$ of $[d]$ features where we define $d_k := \abs{F_k}$, its balance factor is defined as
\[
\text{Balance}(\cF) = \frac{\max_k d_k}{\min_k d_k}
\]
If we have $\min_k d_k = 0$, then the balance factor is defined to be infinity.
\end{definition}

The normalized entropy is a more gentle metric of balance. It takes a value between $0$ and $1$ and a larger value indicates a more balanced clustering and a smaller value indicates the presence of a few very concentrated clusters. It is notable that in Table~\ref{tab:clustering}, \defrag achieved the highest entropy values of 0.99.
\begin{definition}[Normalized Entropy]
Given a $K$-clustering $\cF = [F_1,\ldots,F_K]$ of $[d]$ features where we define $d_k := \abs{F_k}$, its normalized entropy is defined as
\[
\text{NormEnt}(\cF) = -\frac1{\ln d}\sum_{k=1}^K\frac{d_k}d\ln\frac{d_k}d,
\]
where we define $0 \ln 0 = 0$ for sake of avoiding singular points.
\end{definition}

\subsection{Additional Experimental Results}
\label{app:exps-supp}

In this section we report additional results of the reductions \defrag variants offer on training and prediction time and model size along with the effect on prediction accuracy. We also report here results on using the nDCG-based splitting method \defragn, as well as additional results on the performance of the \fiat algorithm on settings with missing features.

\begin{figure}[t]%
\centering
\begin{subfigure}[b]{0.245\textwidth}
	\includegraphics[width=\textwidth]{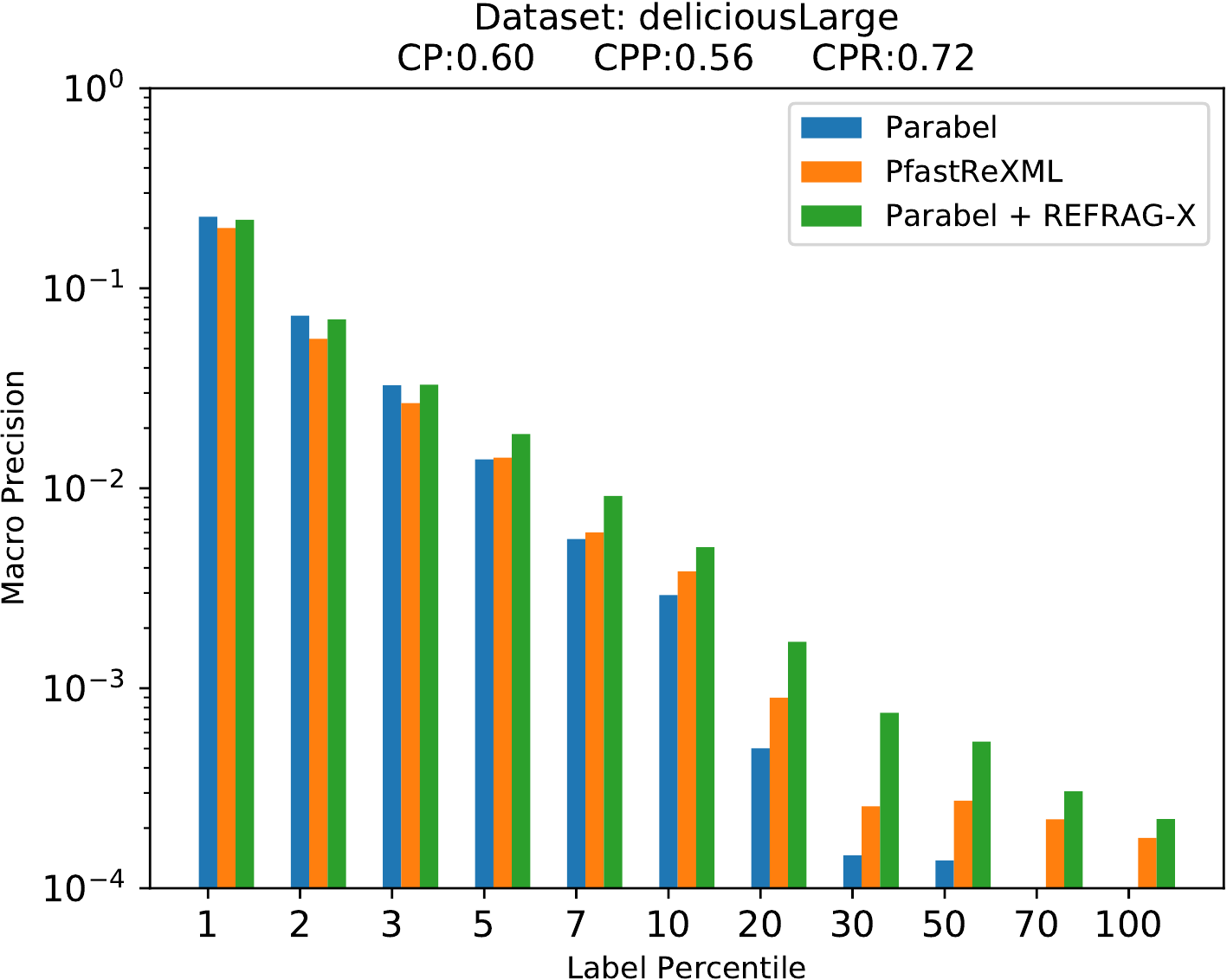}%
\end{subfigure}
\begin{subfigure}[b]{0.245\textwidth}
	\includegraphics[width=\textwidth]{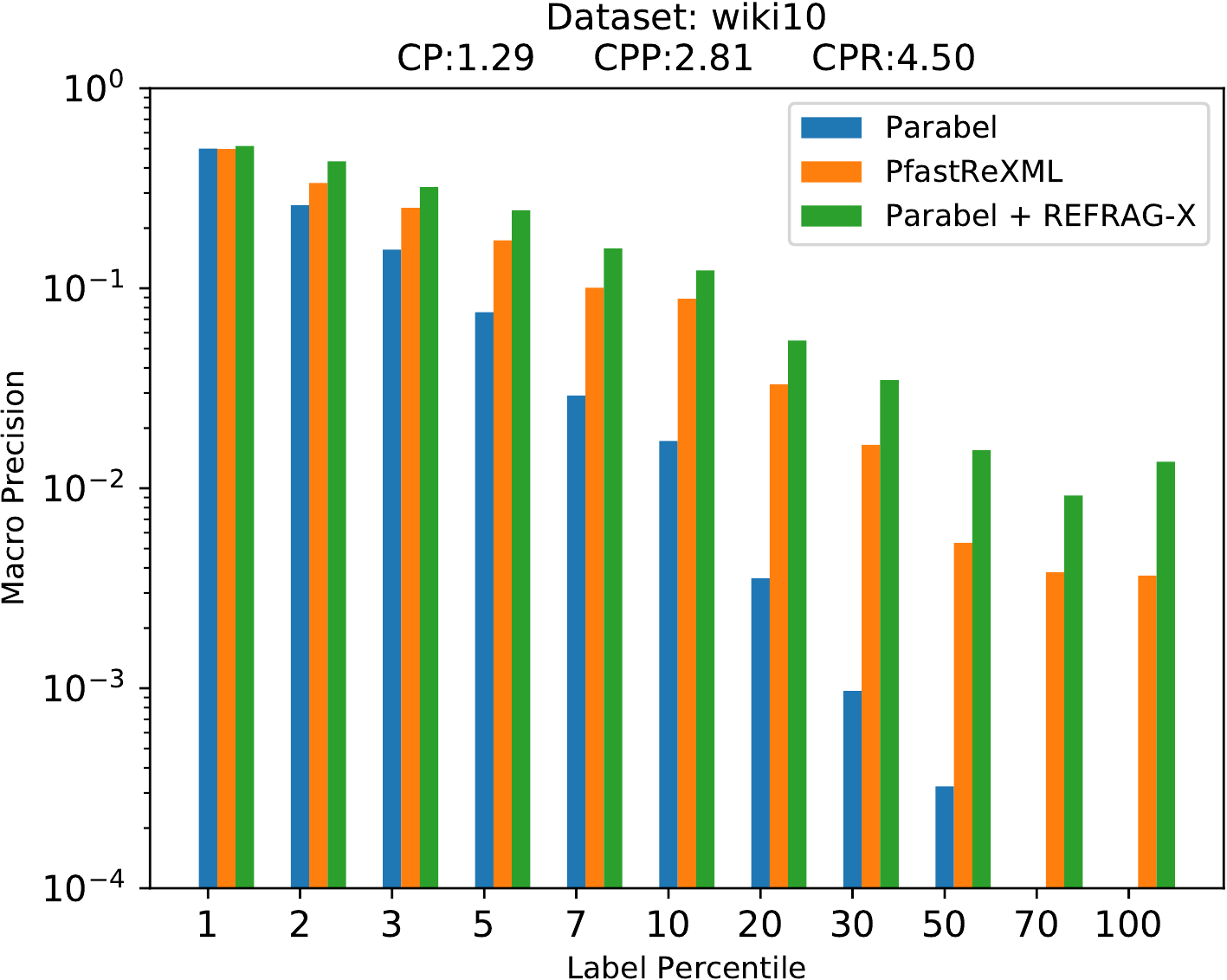}%
\end{subfigure}
\caption{\refrag can offer far superior performance on rare labels as indicated on the Delicious (left) and Wiki10 (right) datasets. The x-axis records the percentile range of labels. For example, the first bar considers labels in the top 0-1 percentile, the second bar considers labels in the 1-2 percentile. Subsequent bars consider labels in the 2-3, 3-5, 5-7, 7-10, 10-20, 20-30, 30-50, 50-70 and 70-100 percentiles respectively. The y-axis records the macro precision of labels within a percentile range. This is calculated by taking an average of label-wise precision scores. Note that this method gives equal weightage to each label, irrespective of its popularity.\\
At the top of each figure, P indicates the overall precision@1 of the methods, R indicates the overall recall@1 of the methods and C indicates the overall coverage@1 of the methods. Coverage@1 is calculated as the fraction of labels that are correctly predicted as the top-ranked label in at least one test document. Note that \defrag more than triples the coverage of the \parabel method on the Wiki10 dataset.\\
Note that on both datasets, \refrag continues to offer meaningful predictions even as we consider very rare labels. The default \parabel algorithm stops offering predictions after the 7-10 percentile on Delicious and the 20-30 percentile on Wiki10. However, \refrag continues to offer predictions uptil the 20-30 percentile on Delicious and 70-100 percentile on Wiki10.}%
\label{fig:app-rerank}%
\end{figure}

\begin{figure}[t]%
\centering
\begin{subfigure}[b]{0.245\textwidth}
	\includegraphics[width=\textwidth]{noise/NOISE_BK_XT_DP25_NC3_PRBL_DP100_wiki10_8_overwritten_P1.png-crop.pdf}%
\end{subfigure}
\begin{subfigure}[b]{0.245\textwidth}
	\includegraphics[width=\textwidth]{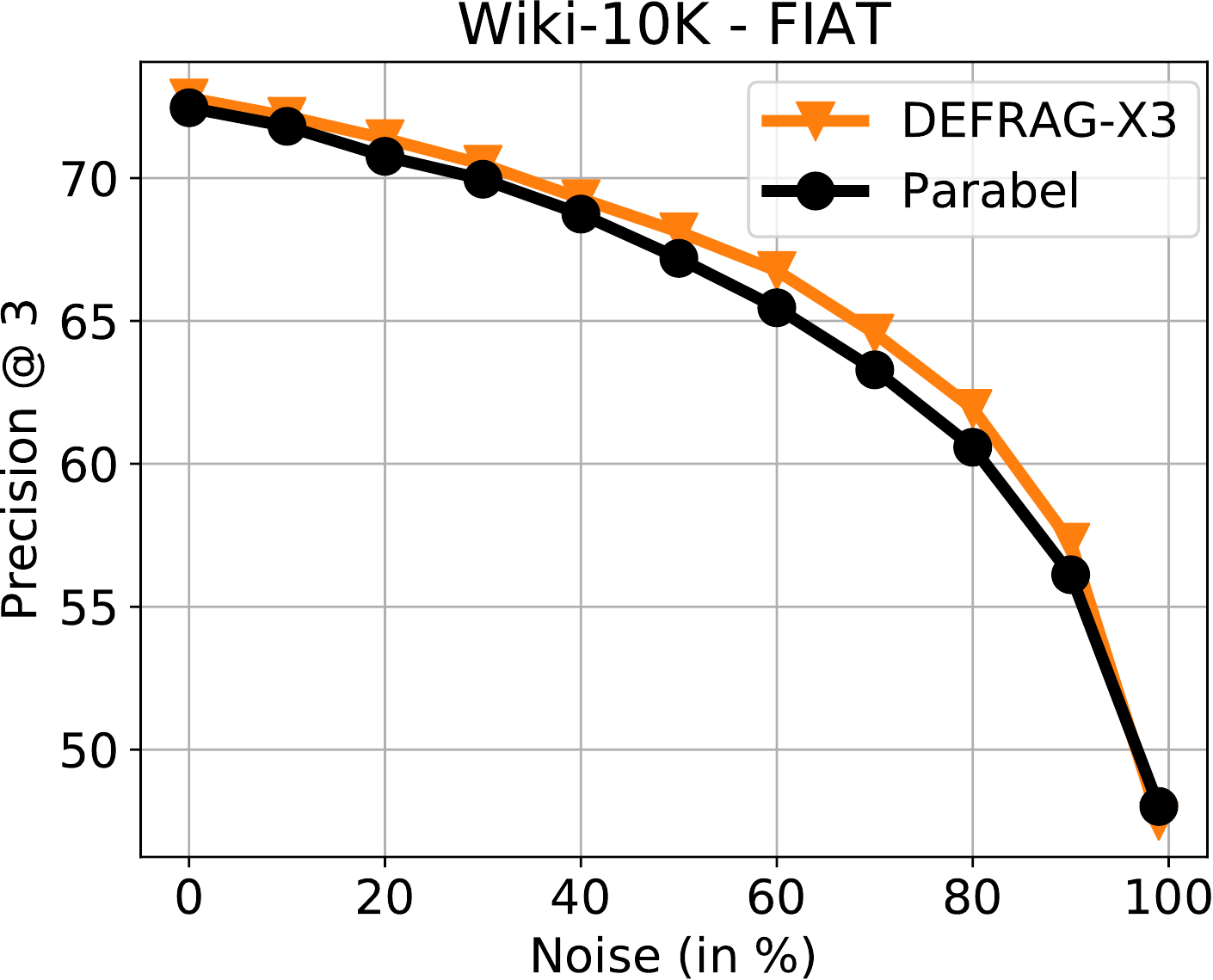}%
\end{subfigure}
\begin{subfigure}[b]{0.245\textwidth}
	\includegraphics[width=\textwidth]{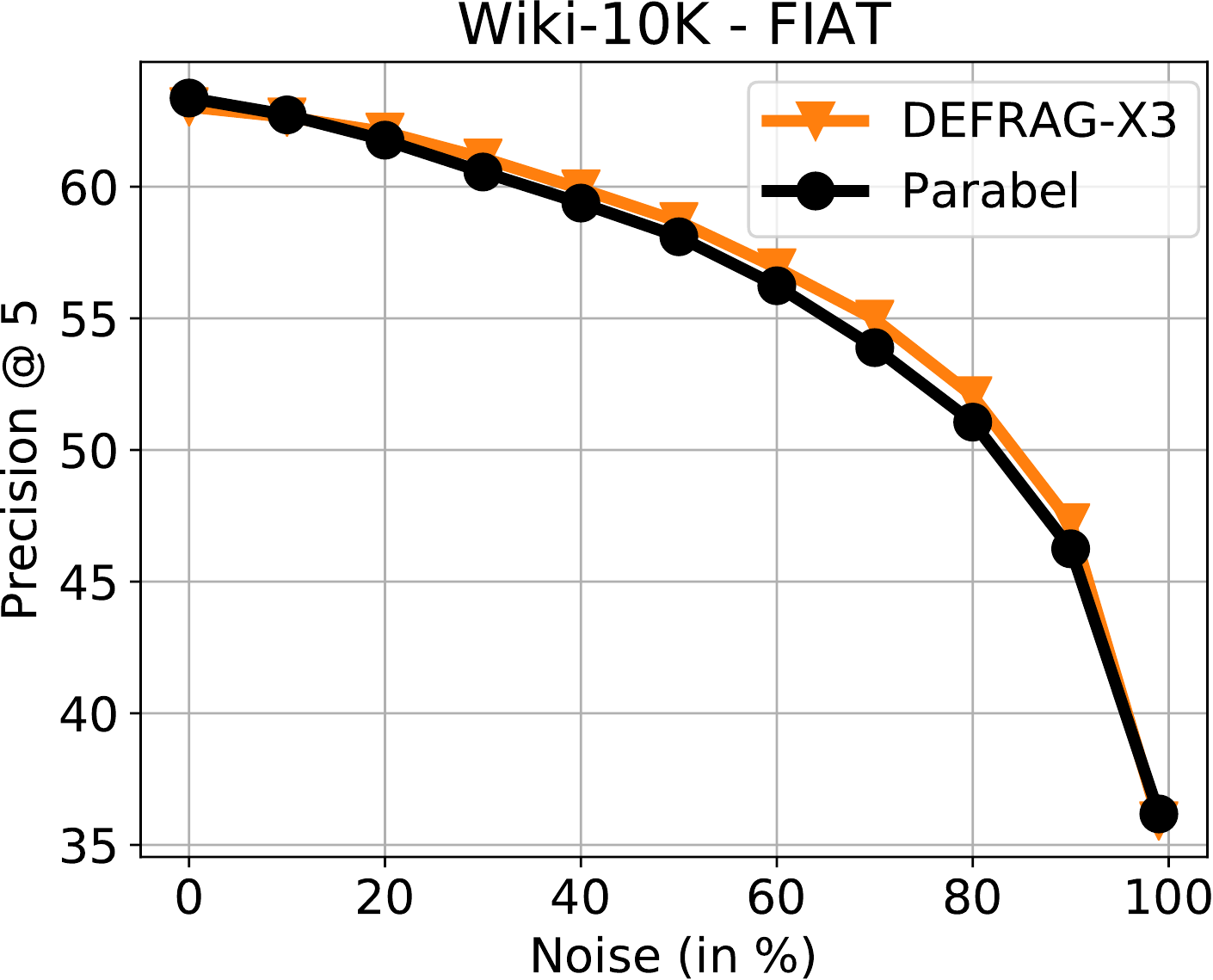}%
\end{subfigure}\\
\begin{subfigure}[b]{0.245\textwidth}
	\includegraphics[width=\textwidth]{noise/NOISE_BK_XT_DP25_NC3_PRBL_DP100_deliciousLarge_8_overwritten_P1.png-crop.pdf}%
\end{subfigure}
\begin{subfigure}[b]{0.245\textwidth}
	\includegraphics[width=\textwidth]{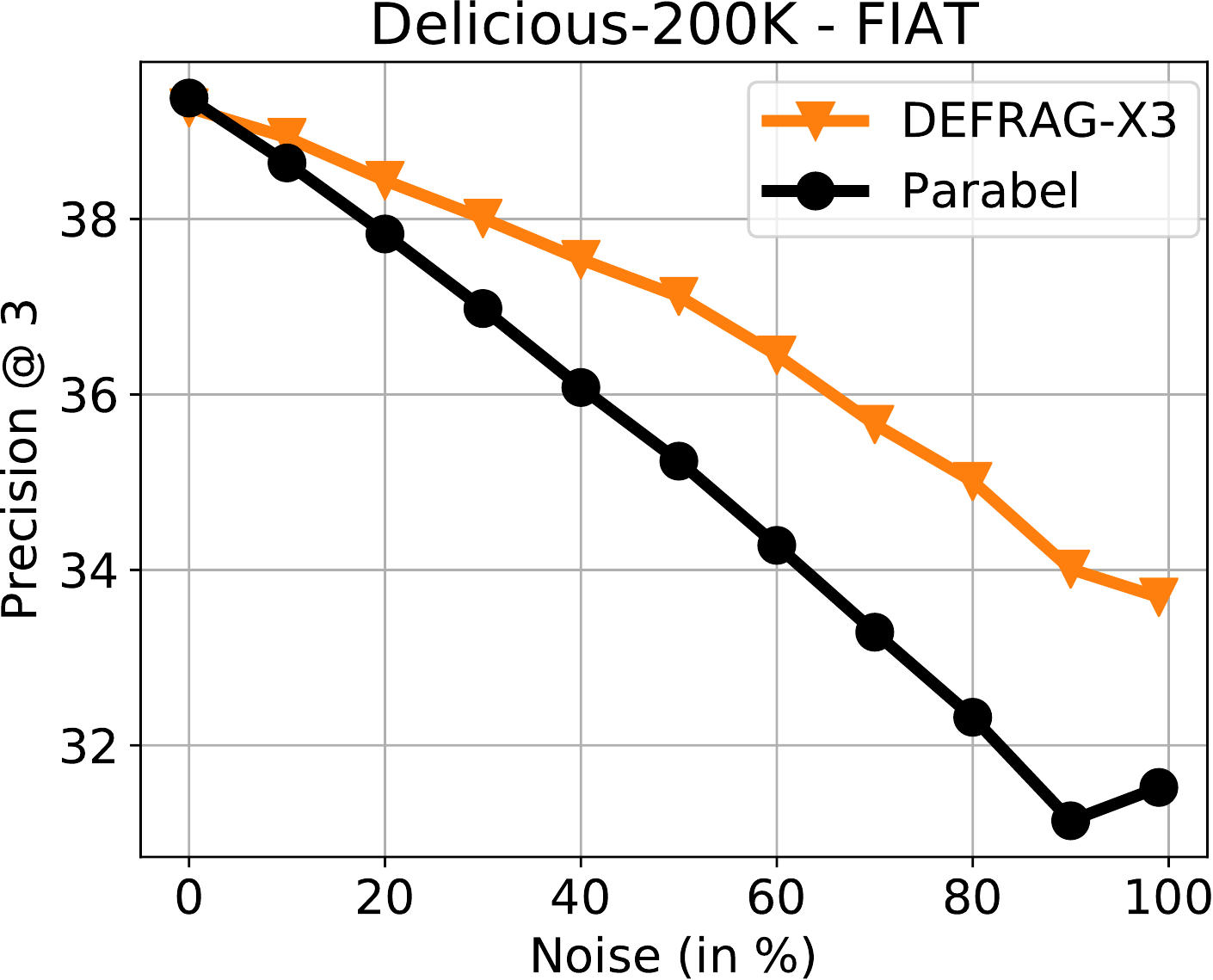}%
\end{subfigure}
\begin{subfigure}[b]{0.245\textwidth}
	\includegraphics[width=\textwidth]{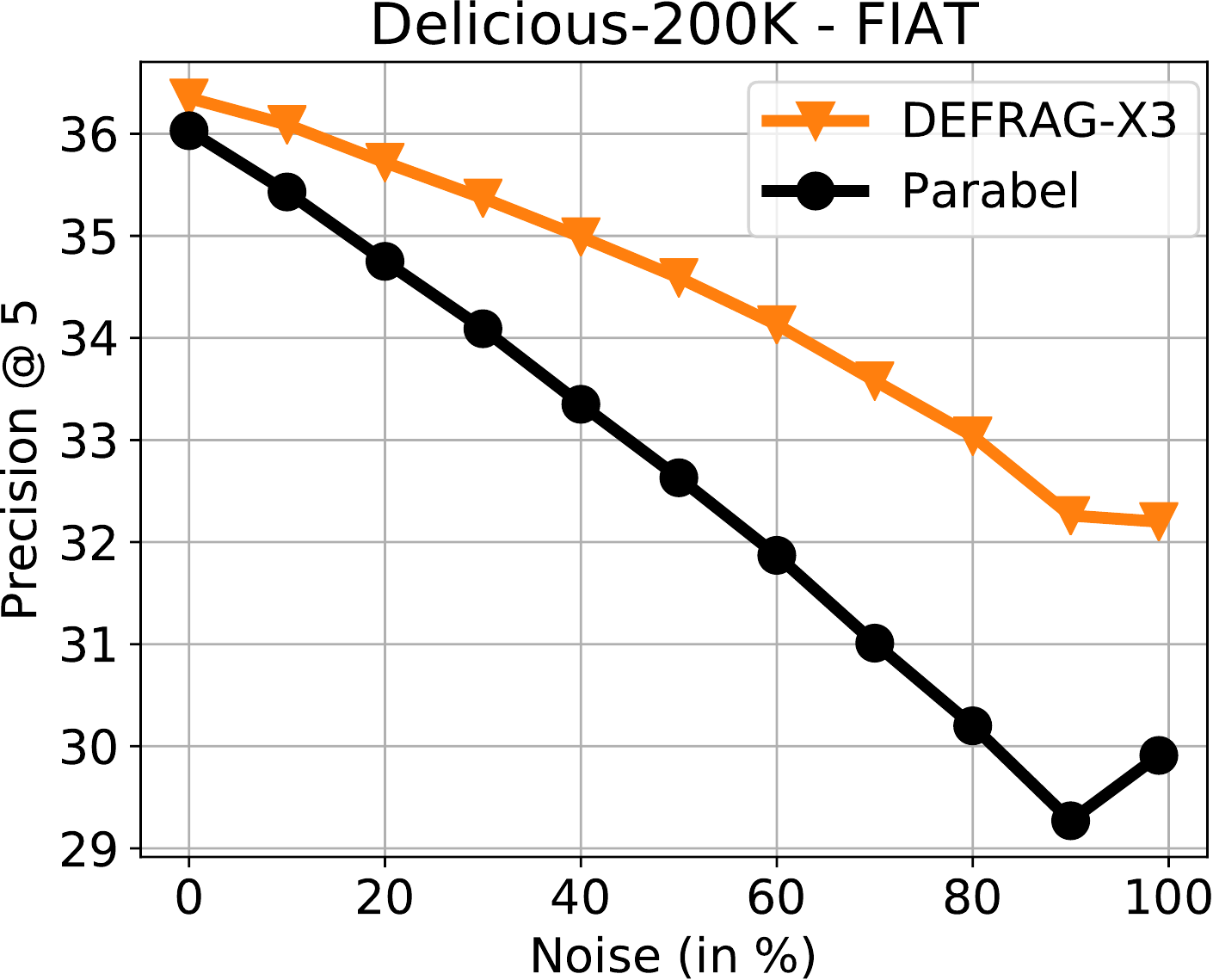}%
\end{subfigure}
\caption{The figure shows the performance of \fiat on two different datasets Wiki10 (first row) and Delicious-200K (second row) with respect to precision at 1, 3 and 5. A certain fraction (as indicated in the x-axis) of features that were present in each of the test data points was randomly removed. Thus, if a test data point had $\hat d$ of the $d$ features with non-zero values in its feature representation, then at 50\% noise level, $\hat d/2$ of these features were removed and their value set to zero. \fiat is able to offer better resilience to removal of features from data points, especially when noise levels go higher. At low noise levels, \fiat is competitive with \parabel but as noise levels grow, the performance gap between the two methods expands.}%
\label{fig:app-noise}%
\end{figure}

\begin{figure}[t]%
\centering
\begin{subfigure}[b]{0.22\textwidth}
	\includegraphics[width=\textwidth]{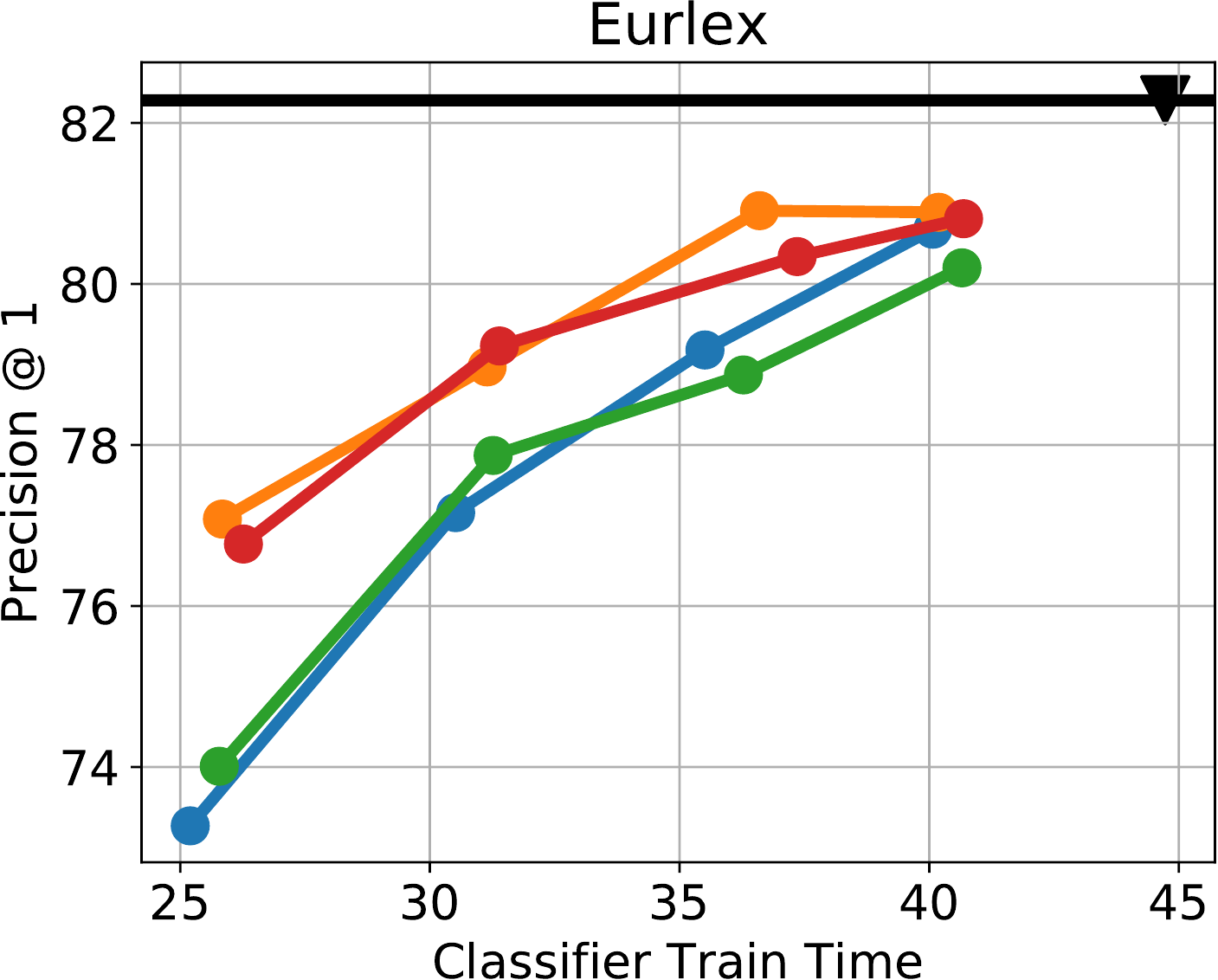}%
\end{subfigure}
\begin{subfigure}[b]{0.22\textwidth}
	\includegraphics[width=\textwidth]{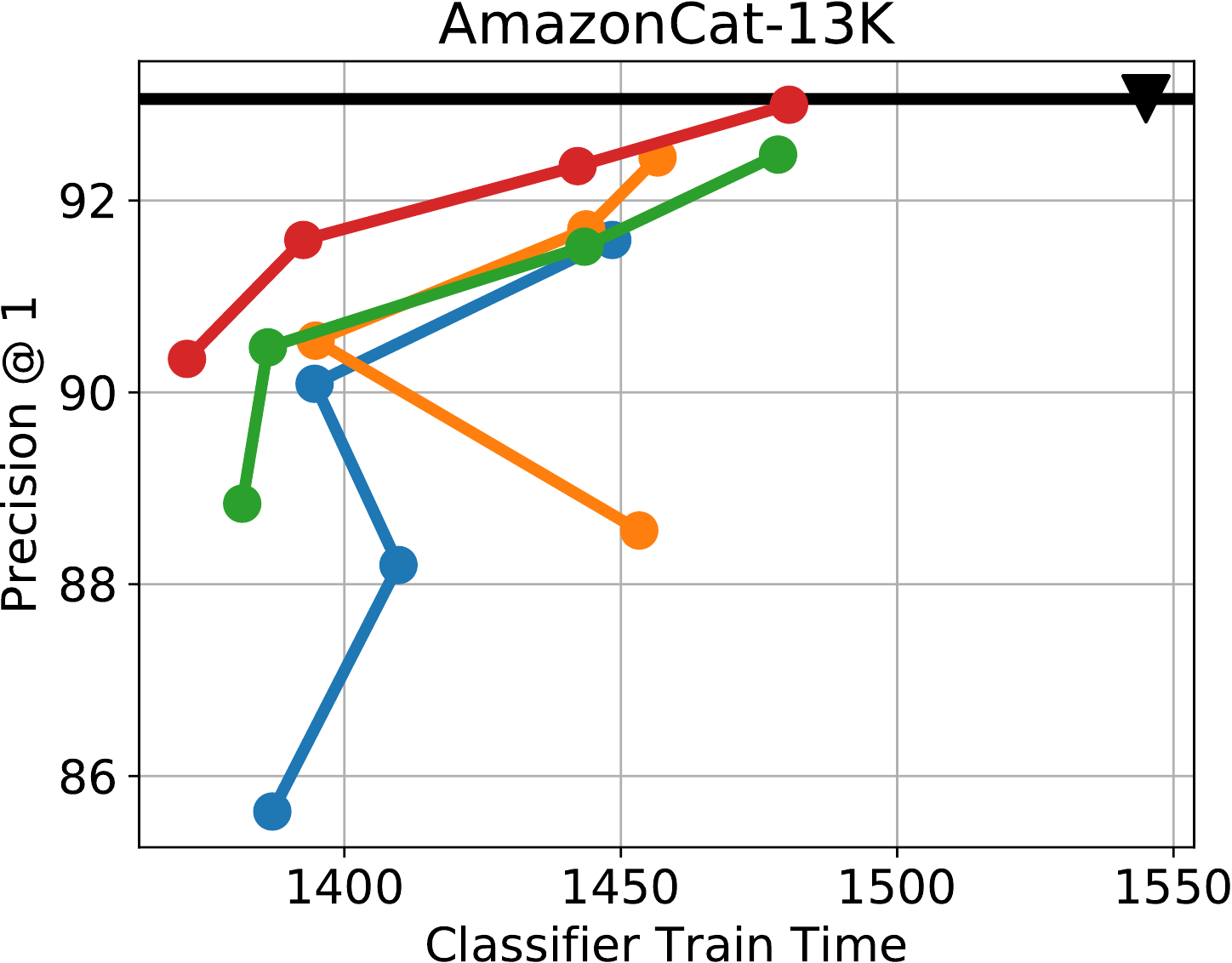}%
\end{subfigure}
\begin{subfigure}[b]{0.22\textwidth}
	\includegraphics[width=\textwidth]{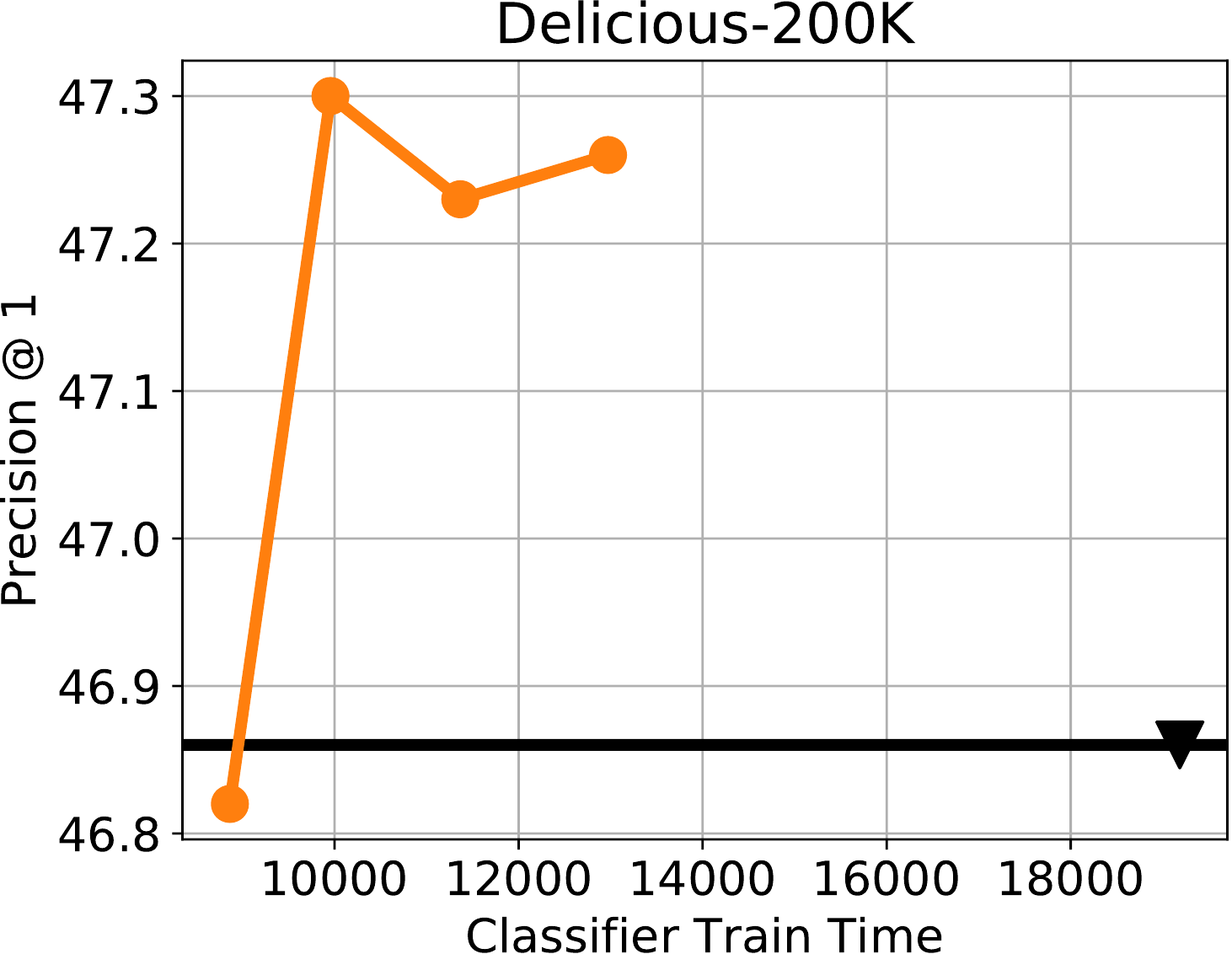}%
\end{subfigure}
\begin{subfigure}[b]{0.32\textwidth}
	\includegraphics[width=\textwidth]{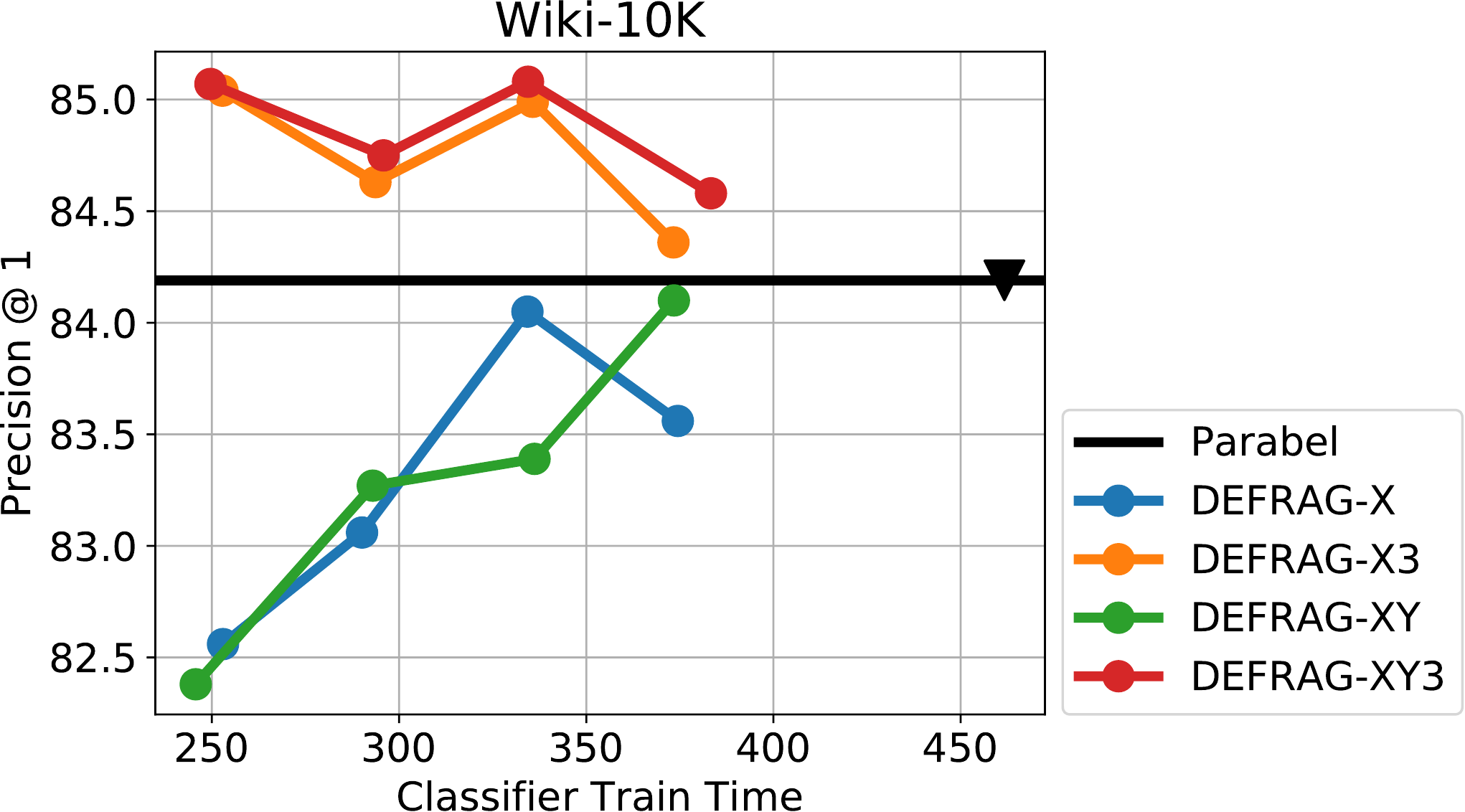}%
\end{subfigure}
\begin{subfigure}[b]{0.245\textwidth}
	\includegraphics[width=\textwidth]{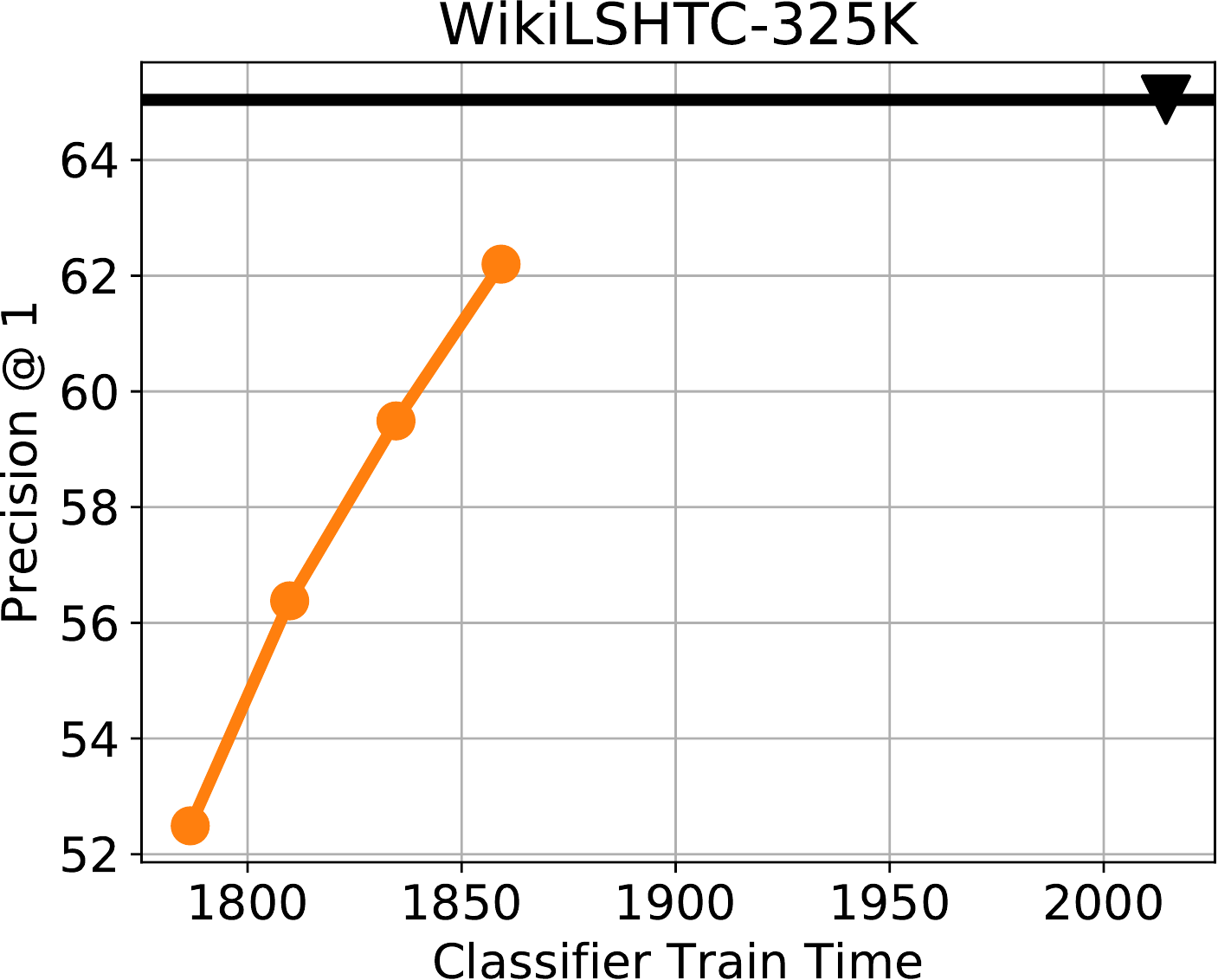}%
\end{subfigure}
\begin{subfigure}[b]{0.245\textwidth}
	\includegraphics[width=\textwidth]{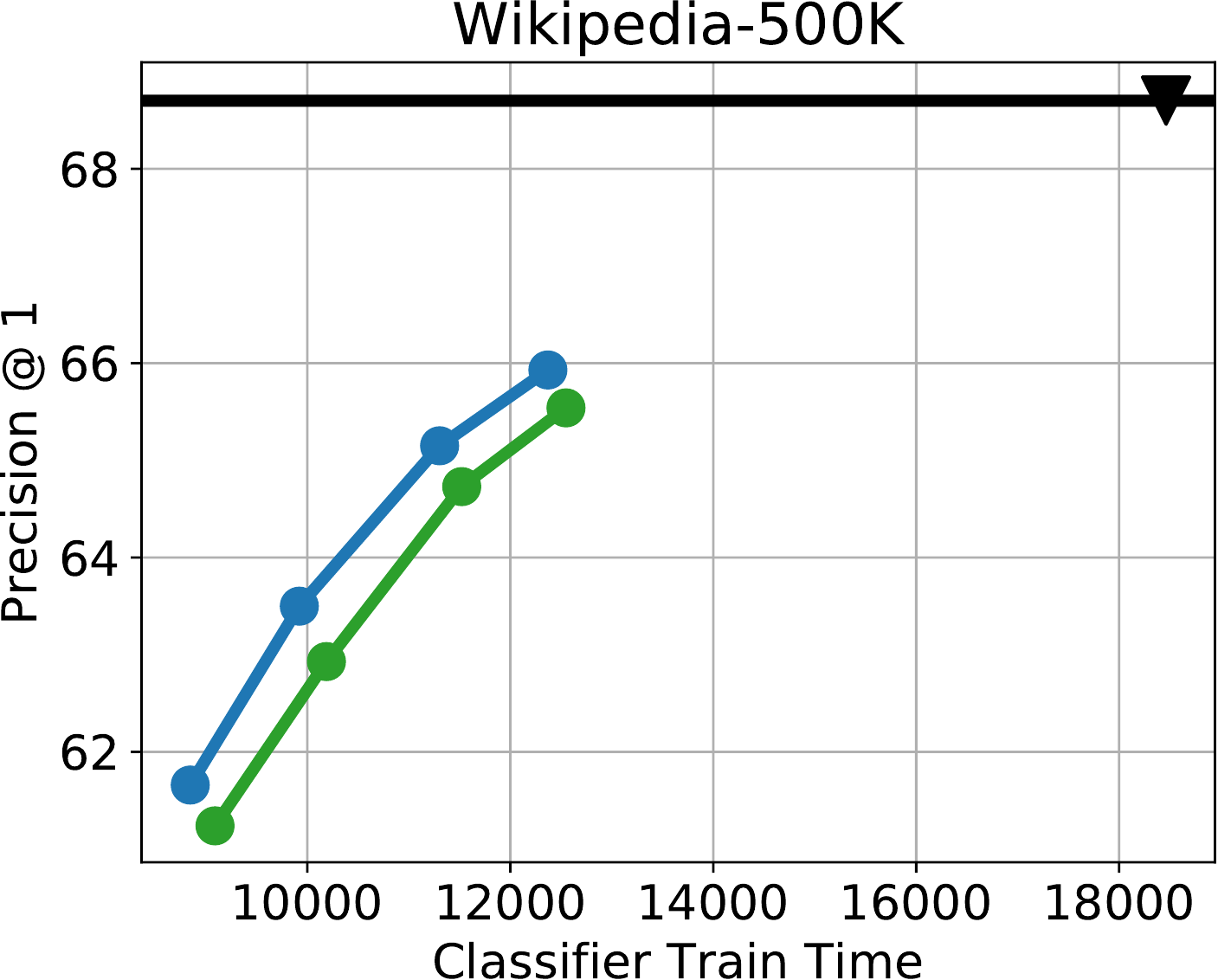}%
\end{subfigure}
\begin{subfigure}[b]{0.245\textwidth}
	\includegraphics[width=\textwidth]{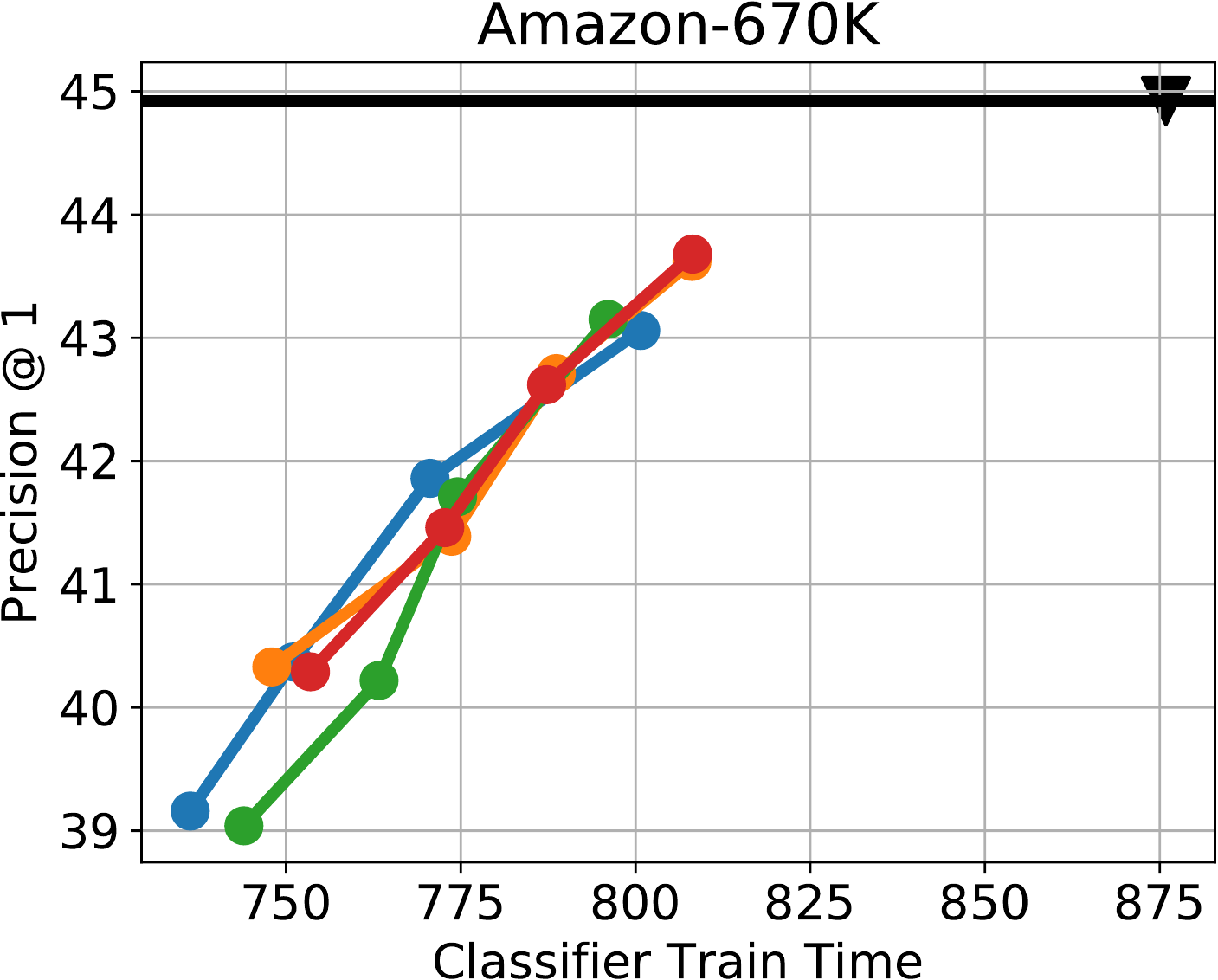}%
\end{subfigure}
\begin{subfigure}[b]{0.245\textwidth}
	\includegraphics[width=\textwidth]{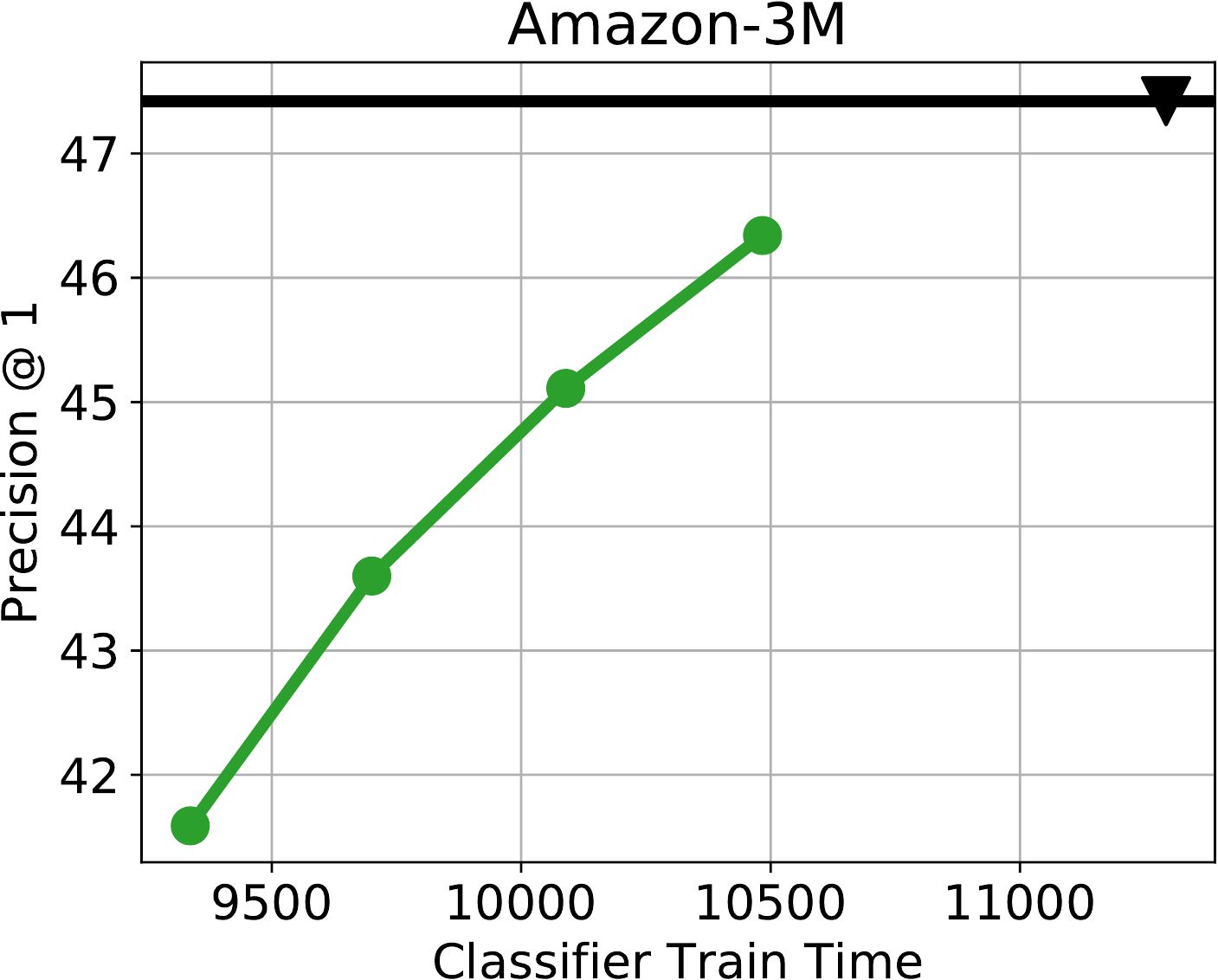}%
\end{subfigure}
\caption{Effect of \defrag variants on reducing classifier training time. \defrag variants (\defragx, \defragxy, ensemble, no-ensemble) were executed with varying number of clusters by setting the minimum leaf size parameter (see Algorithm~\ref{algo:defrag}). The different markers in the figures correspond to $d/32, d/16, d/8, d/4$ balanced clusters. The black line shows \parabel's default performance with the black triangle marking its classifier training time. Aggressive clustering leads to much faster training times and smaller model sizes but also cause a drop in performance. \defragx{}3 and \defragxy{}3 refer to an ensemble of 3 independent realizations of \defrag. \defrag variants are able to bring about significant reductions in the training time of classifiers, for instance Delicious (more than 35\% reduction), Wikipedia-500K (more than 33\% reduction), Wiki10 (more than 20\% reduction).}%
\label{fig:app-compare-class-time}%
\end{figure}

\begin{figure}[t]%
\centering
\begin{subfigure}[b]{0.22\textwidth}
	\includegraphics[width=\textwidth]{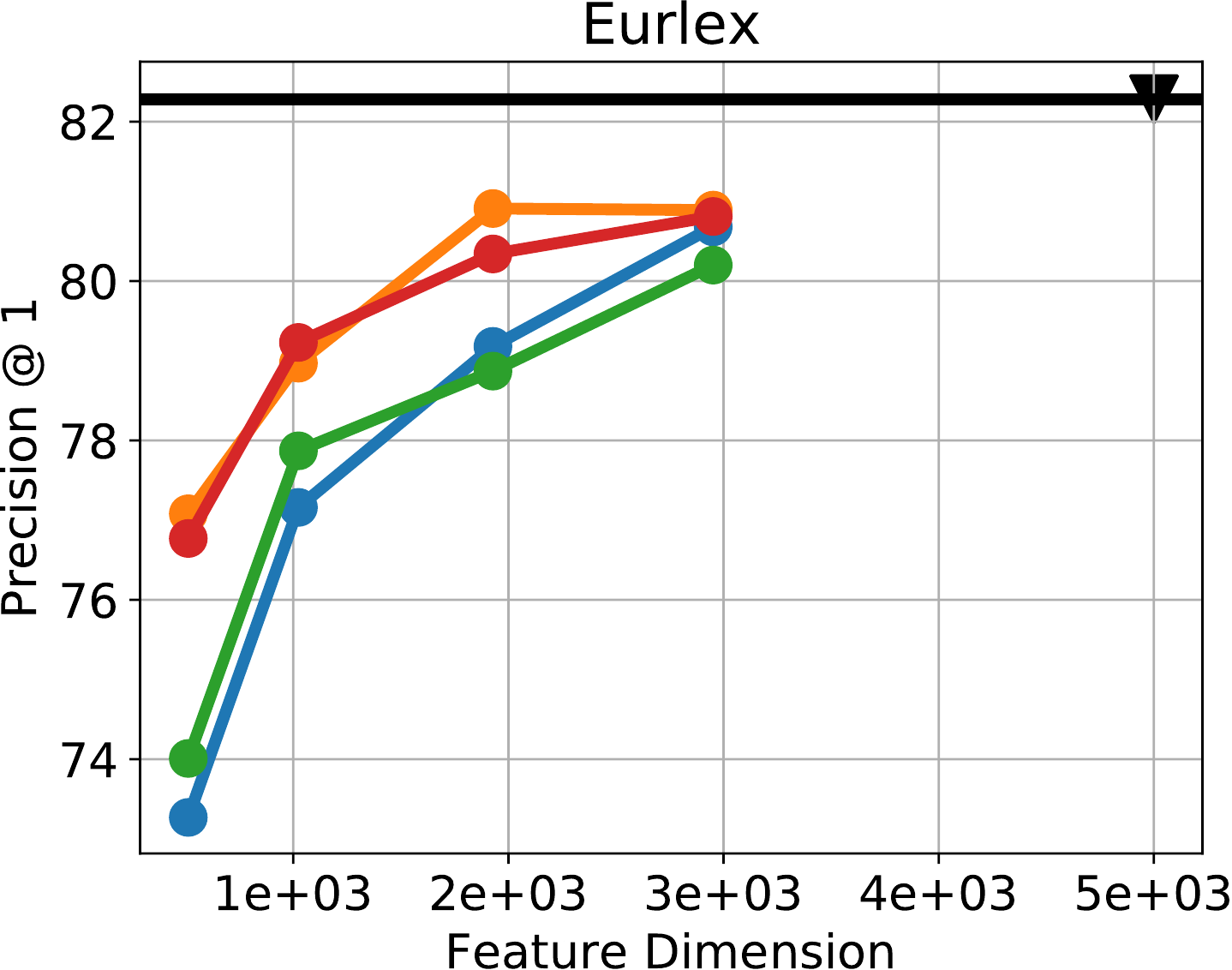}%
\end{subfigure}
\begin{subfigure}[b]{0.22\textwidth}
	\includegraphics[width=\textwidth]{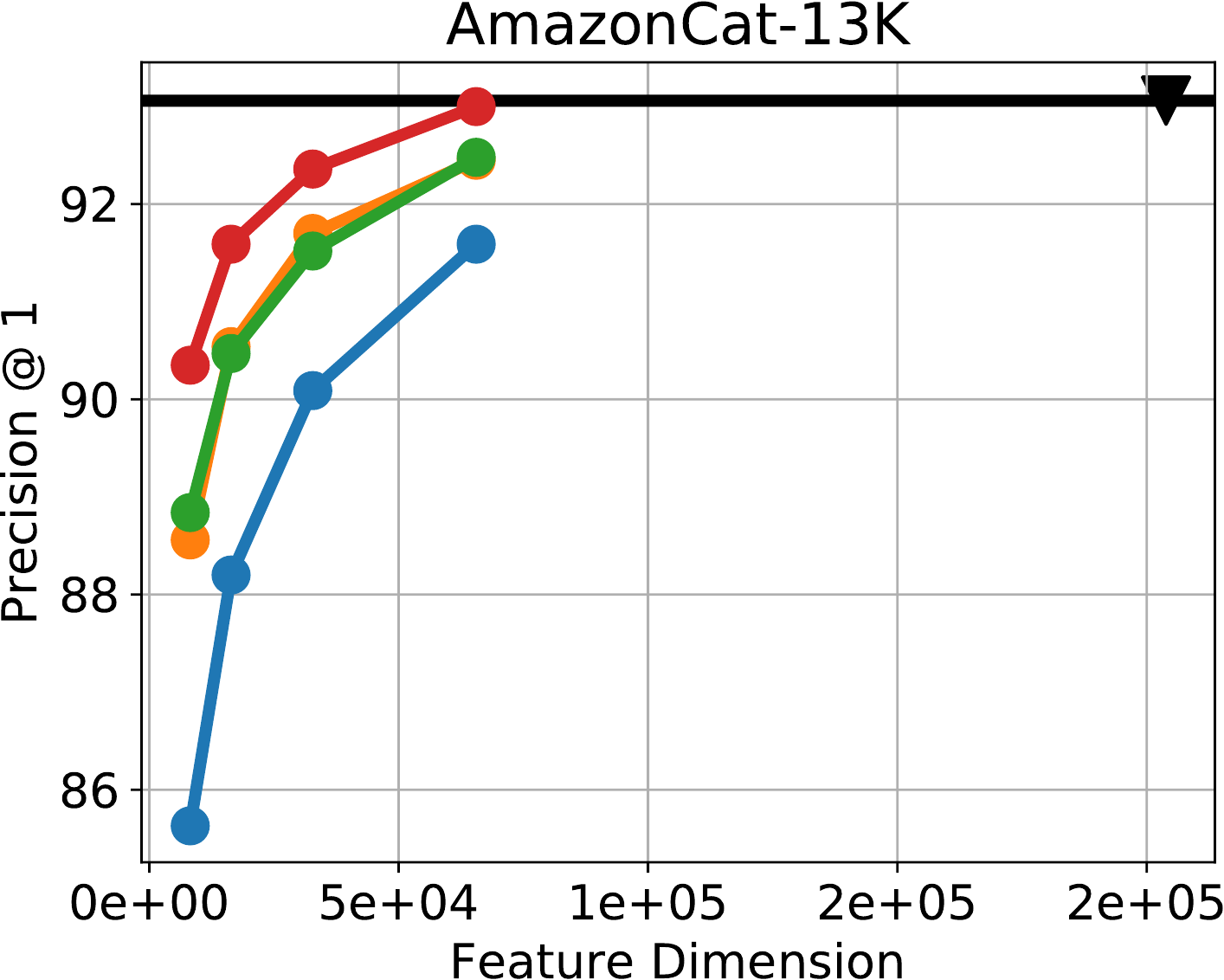}%
\end{subfigure}
\begin{subfigure}[b]{0.22\textwidth}
	\includegraphics[width=\textwidth]{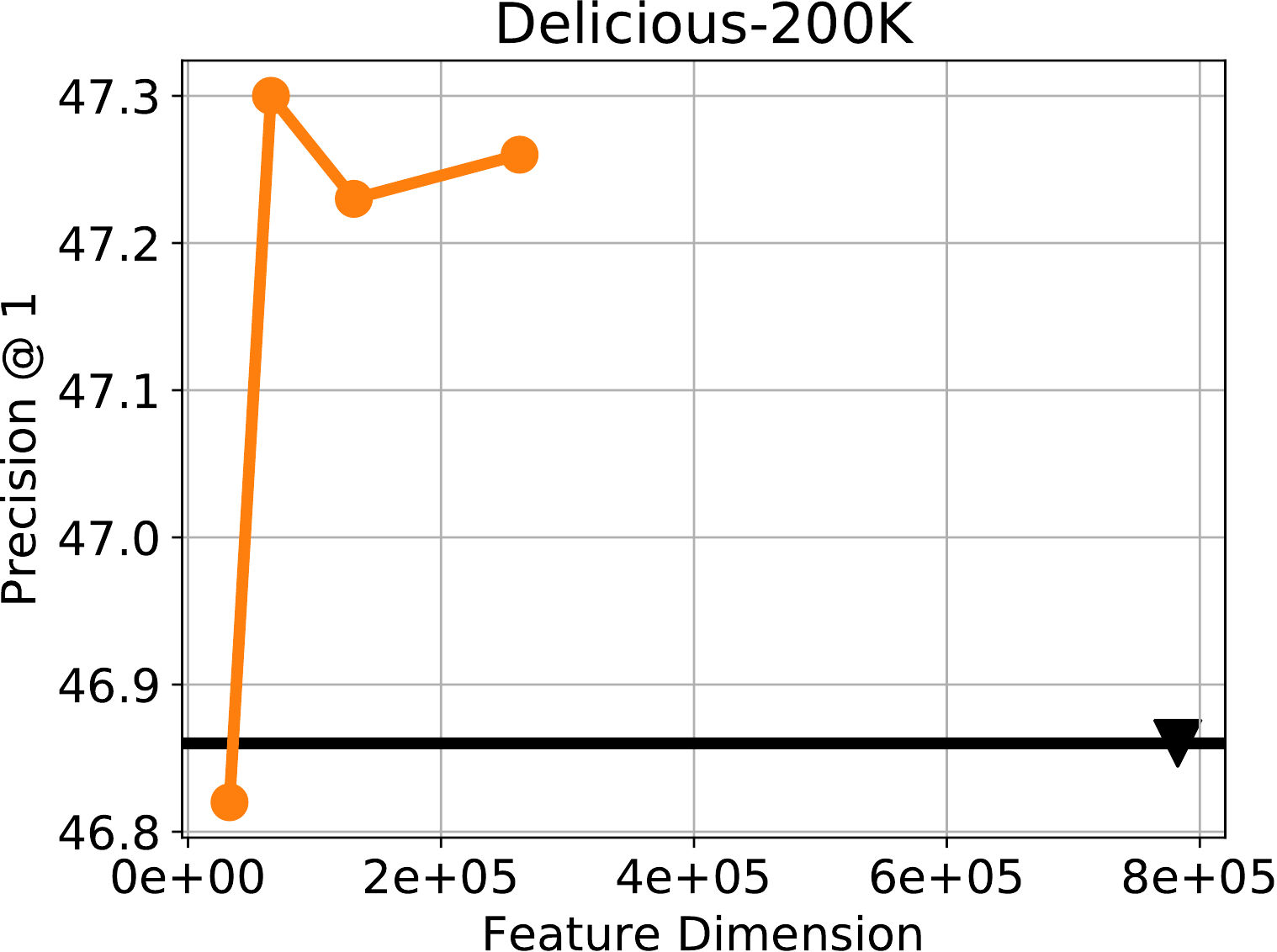}%
\end{subfigure}
\begin{subfigure}[b]{0.32\textwidth}
	\includegraphics[width=\textwidth]{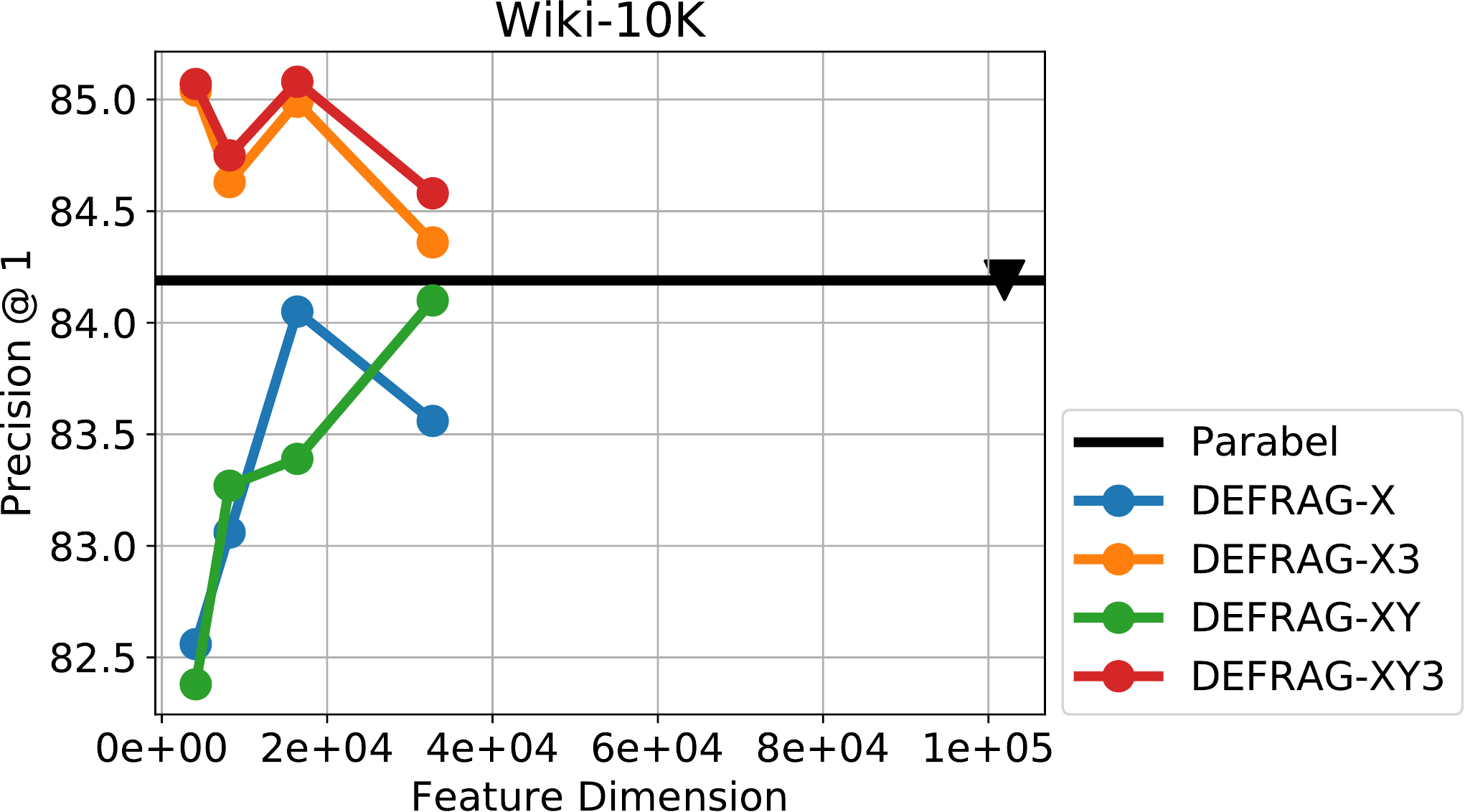}%
\end{subfigure}
\begin{subfigure}[b]{0.245\textwidth}
	\includegraphics[width=\textwidth]{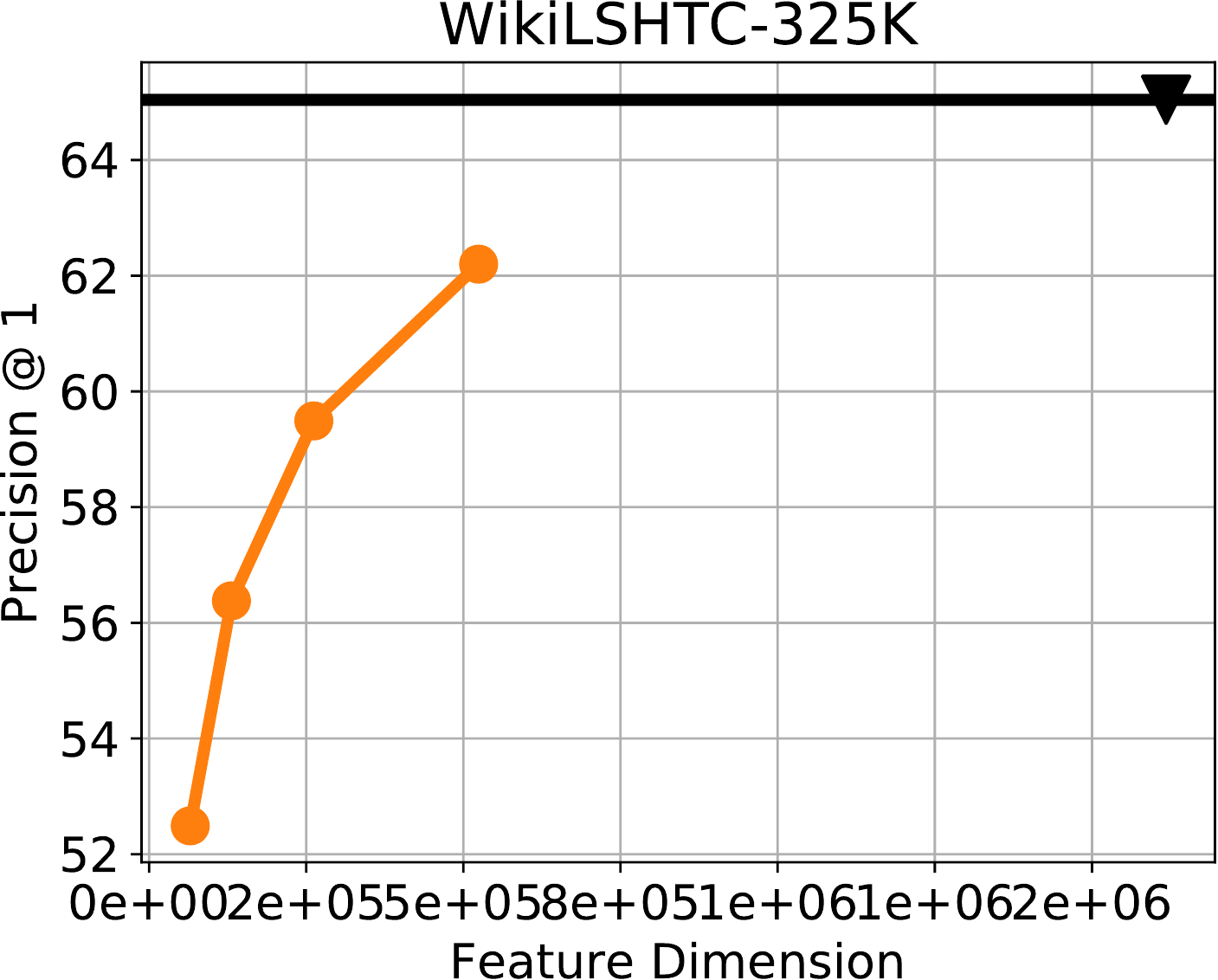}%
\end{subfigure}
\begin{subfigure}[b]{0.245\textwidth}
	\includegraphics[width=\textwidth]{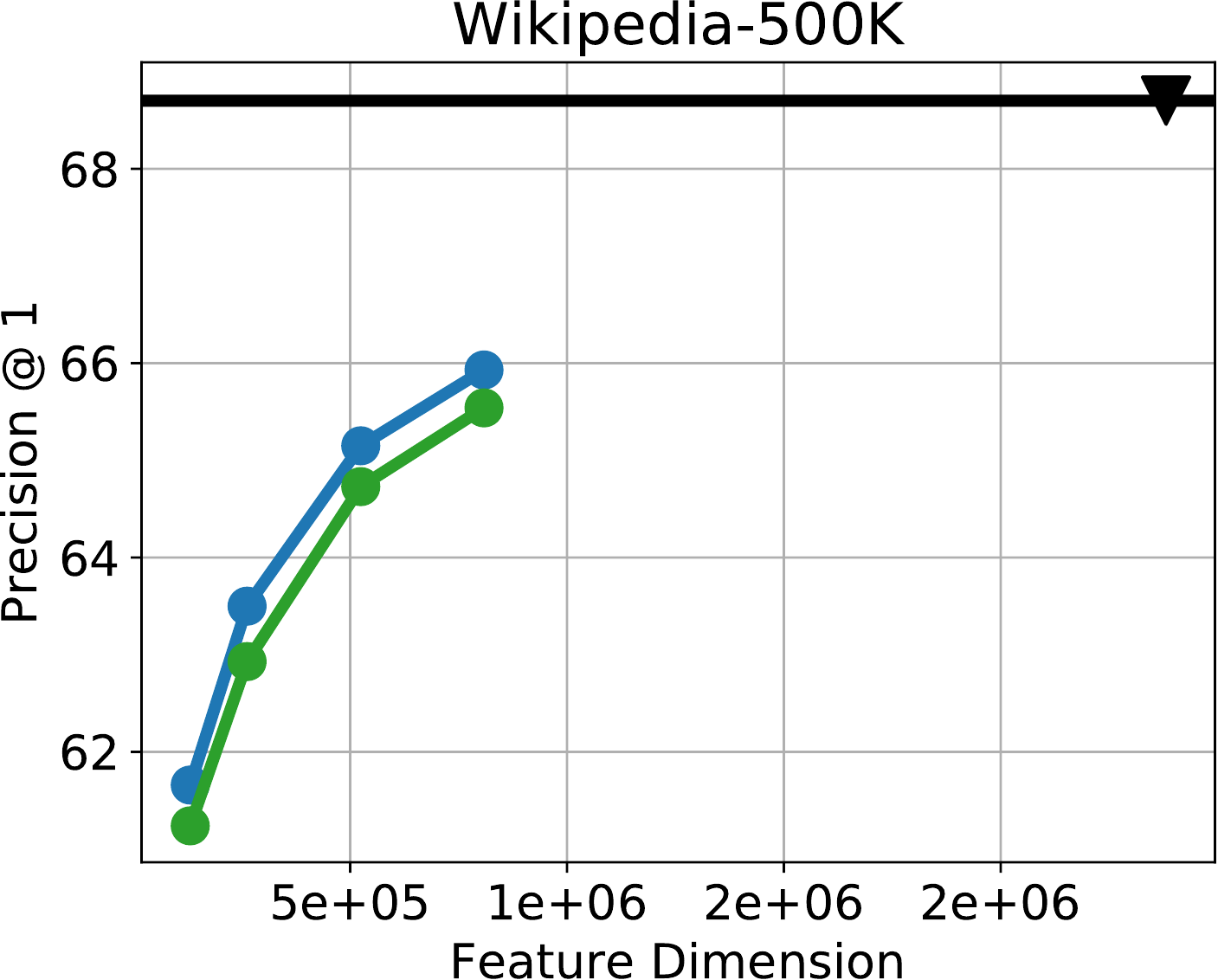}%
\end{subfigure}
\begin{subfigure}[b]{0.245\textwidth}
	\includegraphics[width=\textwidth]{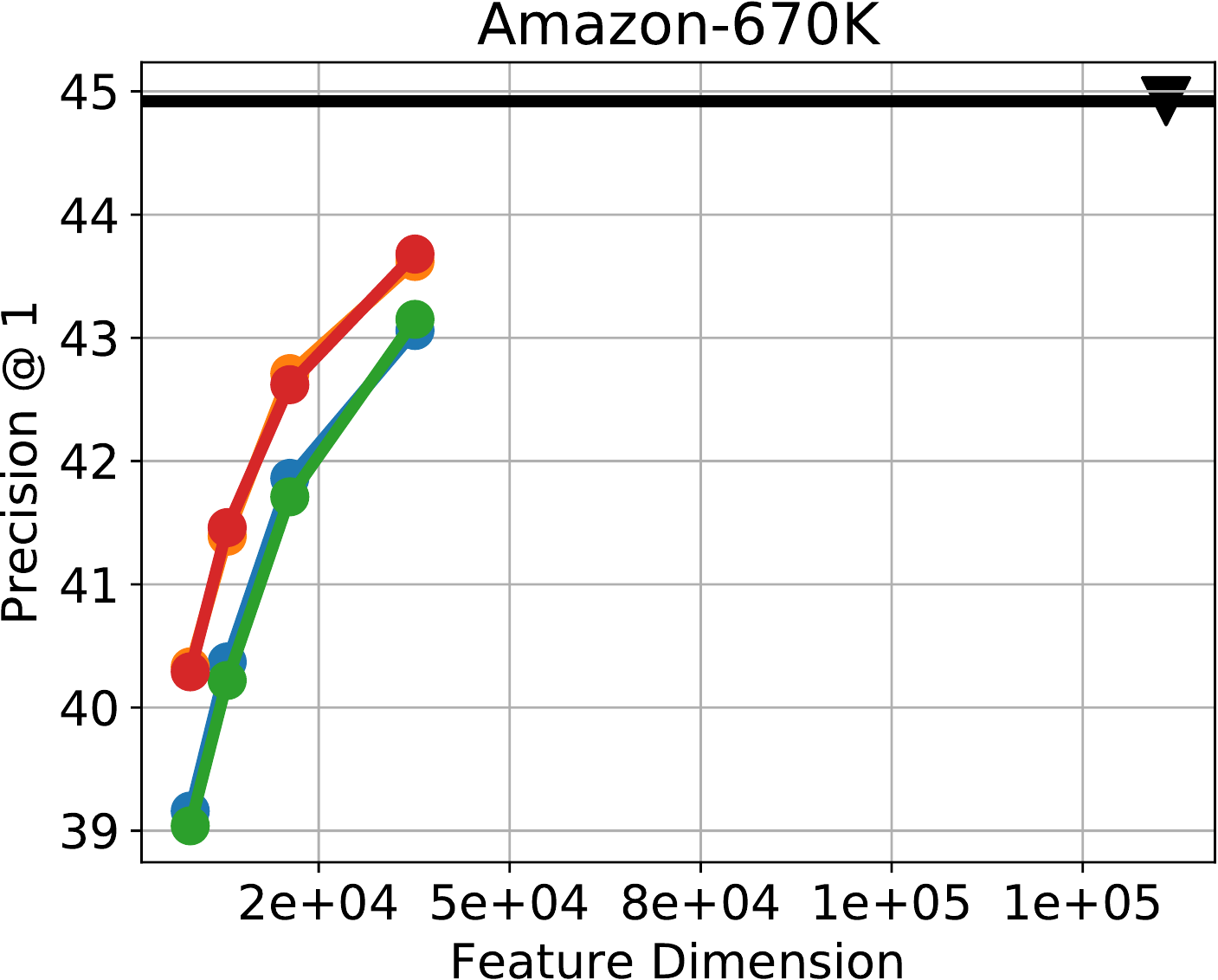}%
\end{subfigure}
\begin{subfigure}[b]{0.245\textwidth}
	\includegraphics[width=\textwidth]{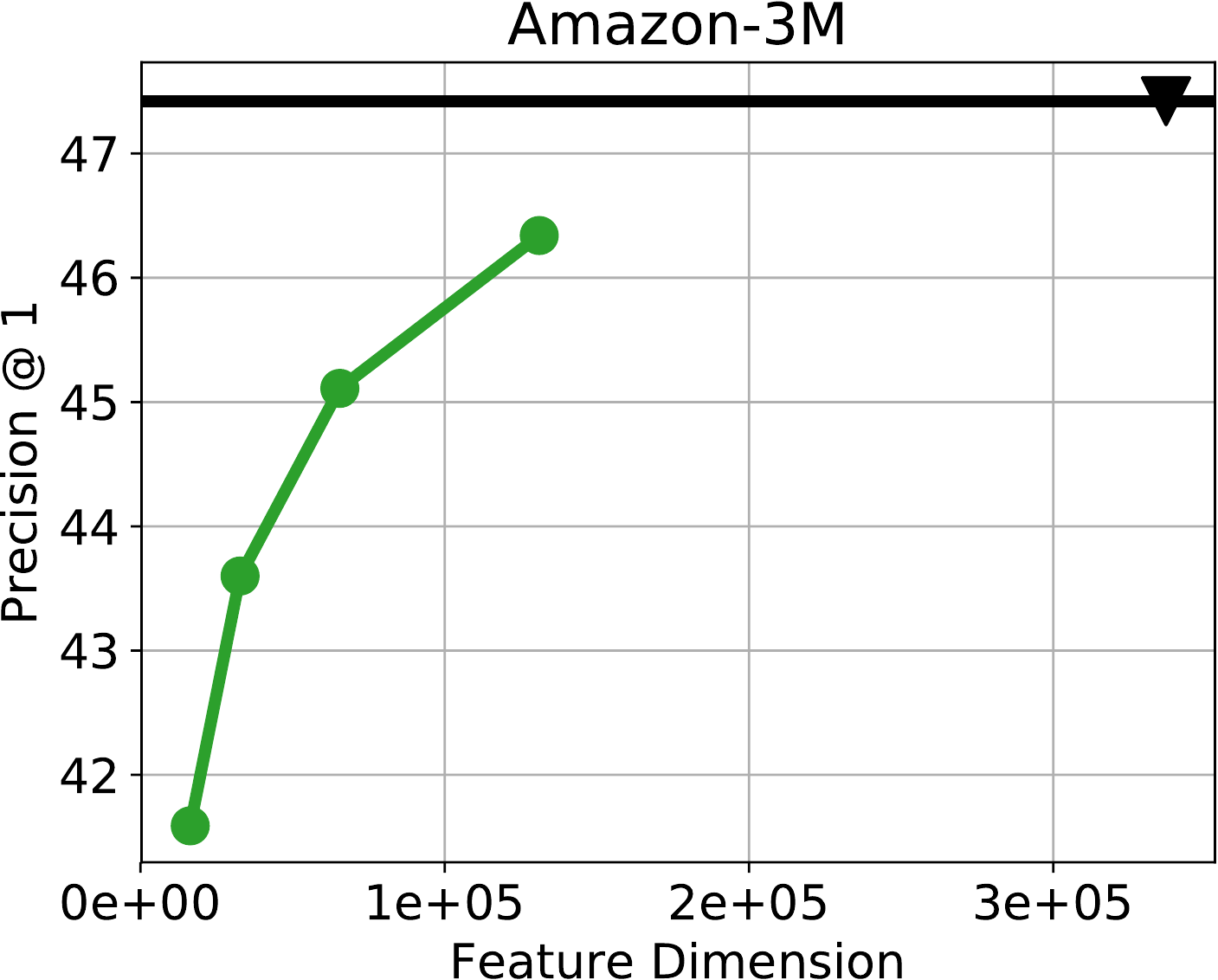}%
\end{subfigure}
\caption{Effect of \defrag variants on reducing feature dimensionality. In all cases, \defrag variants are able to offer drastic reductions in feature dimensionalities of the datasets without significant reduction in prediction accuracies, as measured by precision@1, for example WikiLSHTC (4x reduction), AmazonCat (4x reduction with the same prediction accuracy), Delicious200K (16x reduction with even higher prediction accuracy), Amazon-3M (8x reduction).}%
\label{fig:app-compare-feature-dim}%
\end{figure}

\begin{figure}[t]%
\centering
\begin{subfigure}[b]{0.22\textwidth}
	\includegraphics[width=\textwidth]{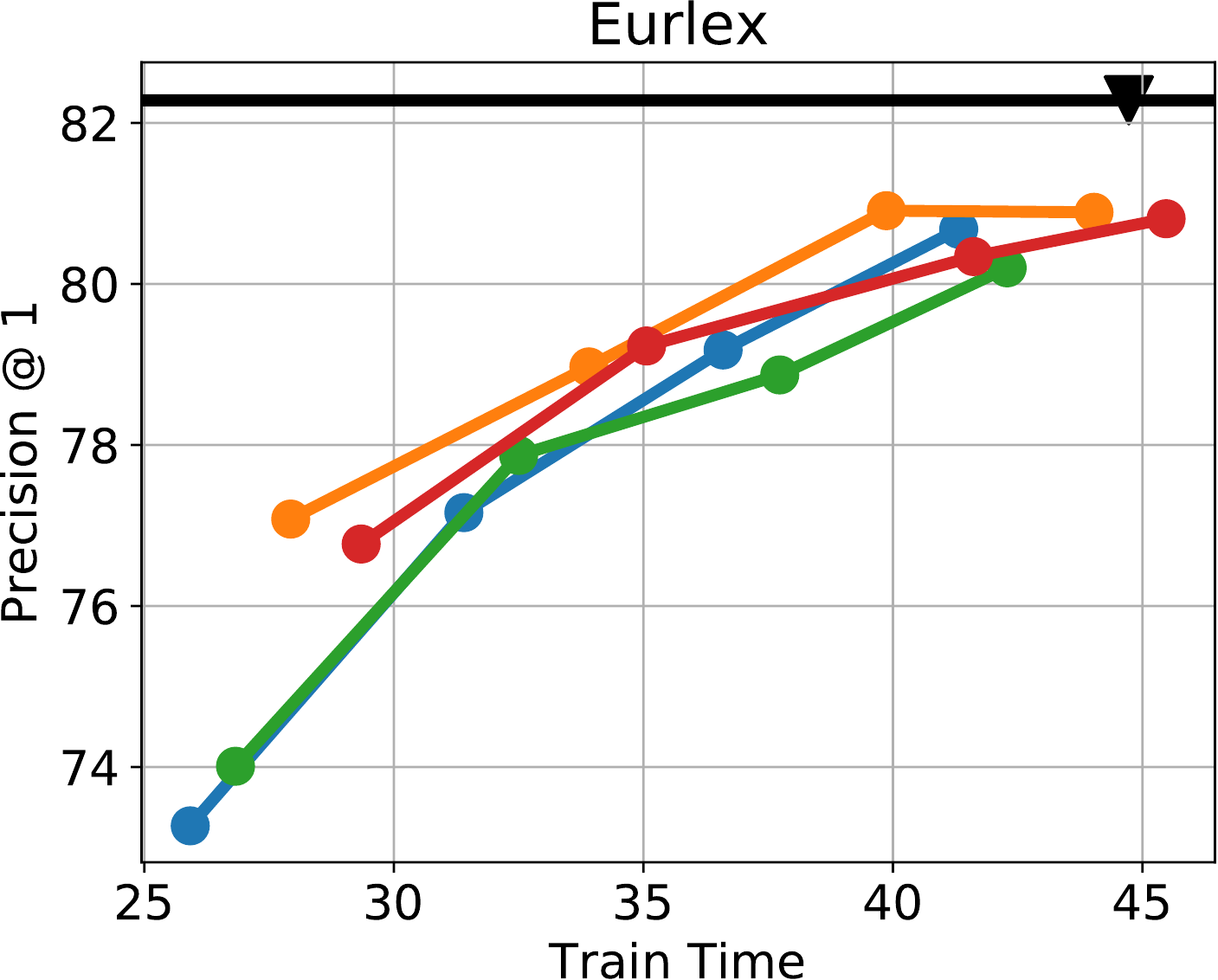}%
\end{subfigure}
\begin{subfigure}[b]{0.22\textwidth}
	\includegraphics[width=\textwidth]{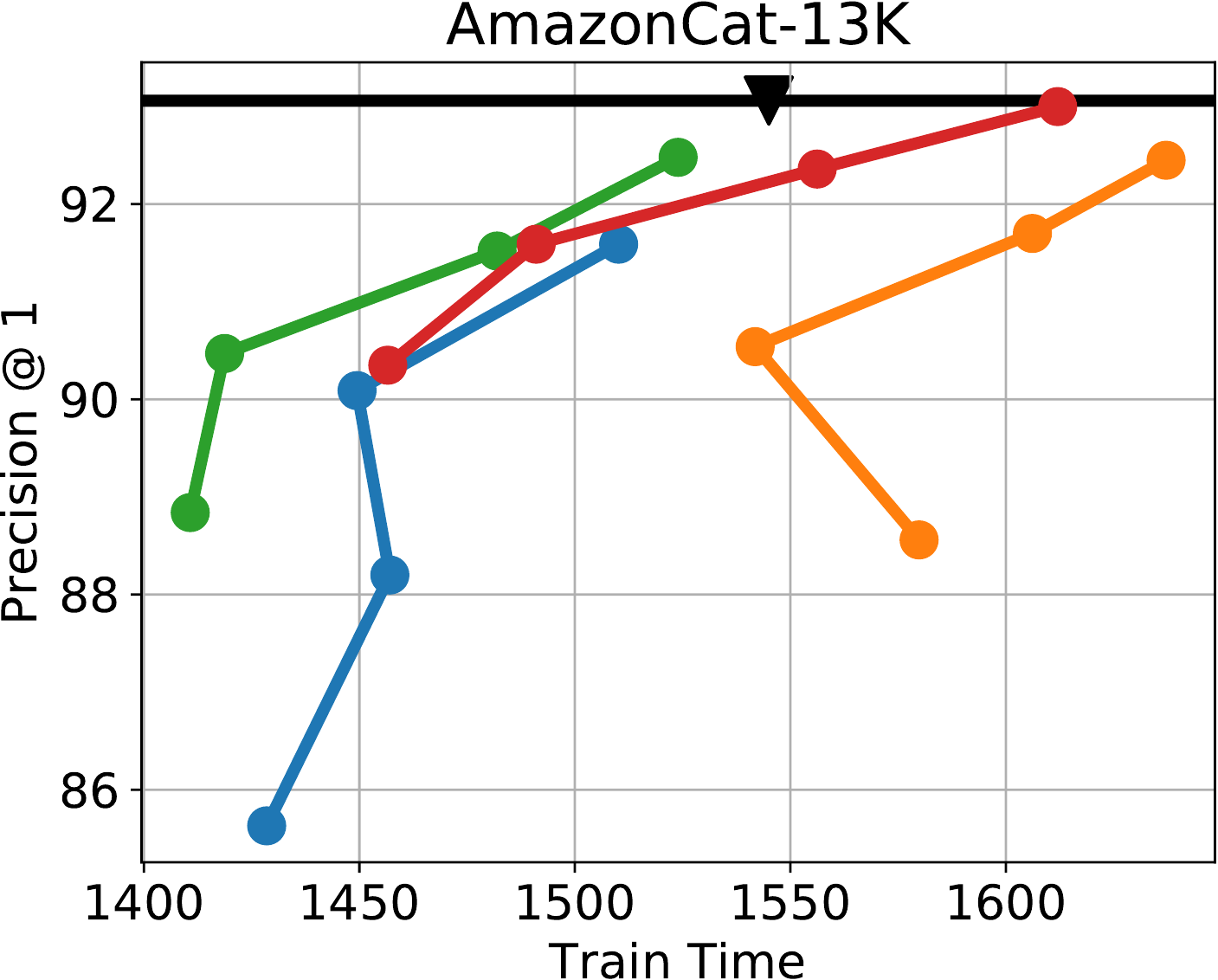}%
\end{subfigure}
\begin{subfigure}[b]{0.22\textwidth}
	\includegraphics[width=\textwidth]{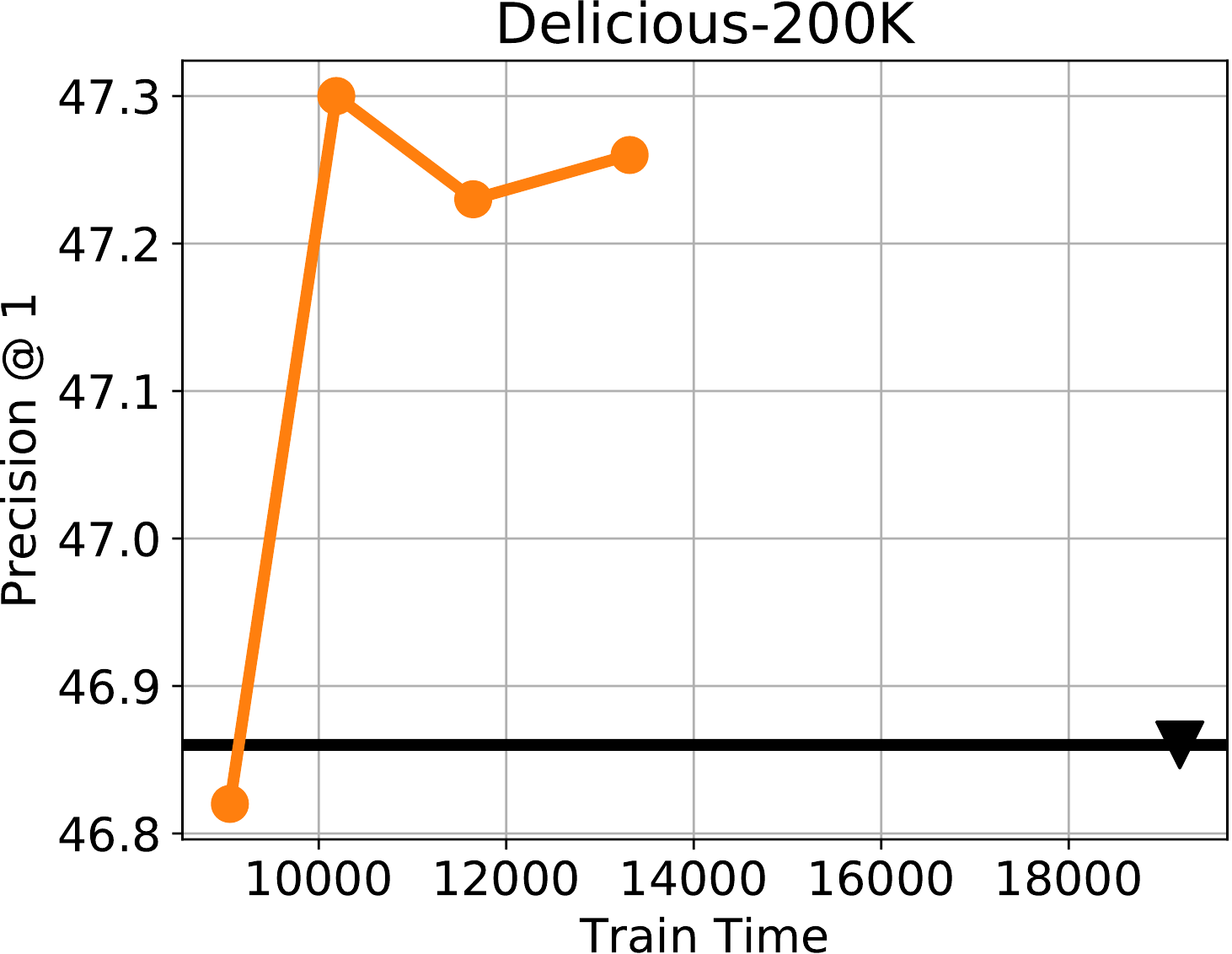}%
\end{subfigure}
\begin{subfigure}[b]{0.32\textwidth}
	\includegraphics[width=\textwidth]{compare/precision_at_1_total_train_time_wiki10.png-crop.pdf}%
\end{subfigure}
\begin{subfigure}[b]{0.245\textwidth}
	\includegraphics[width=\textwidth]{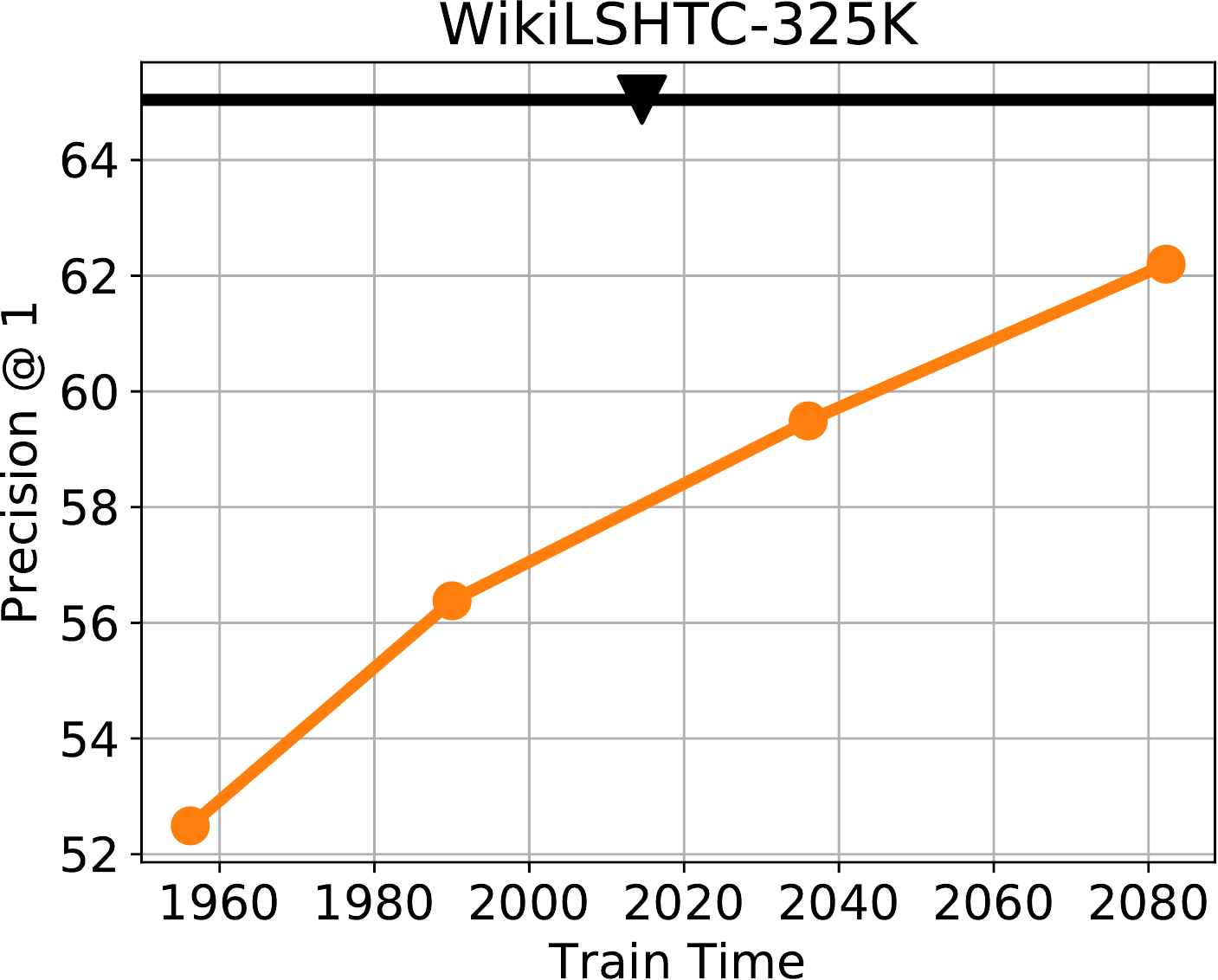}%
\end{subfigure}
\begin{subfigure}[b]{0.245\textwidth}
	\includegraphics[width=\textwidth]{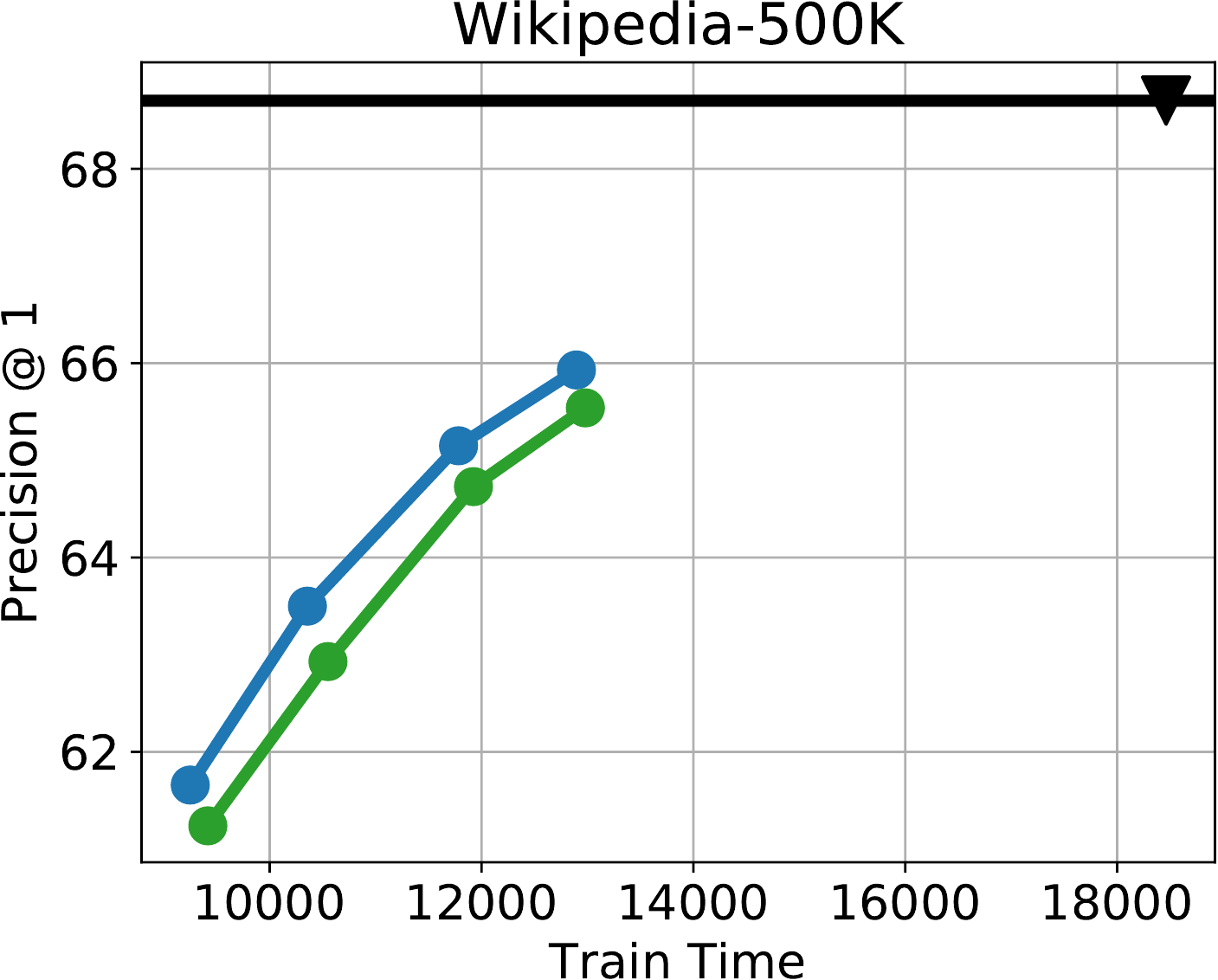}%
\end{subfigure}
\begin{subfigure}[b]{0.245\textwidth}
	\includegraphics[width=\textwidth]{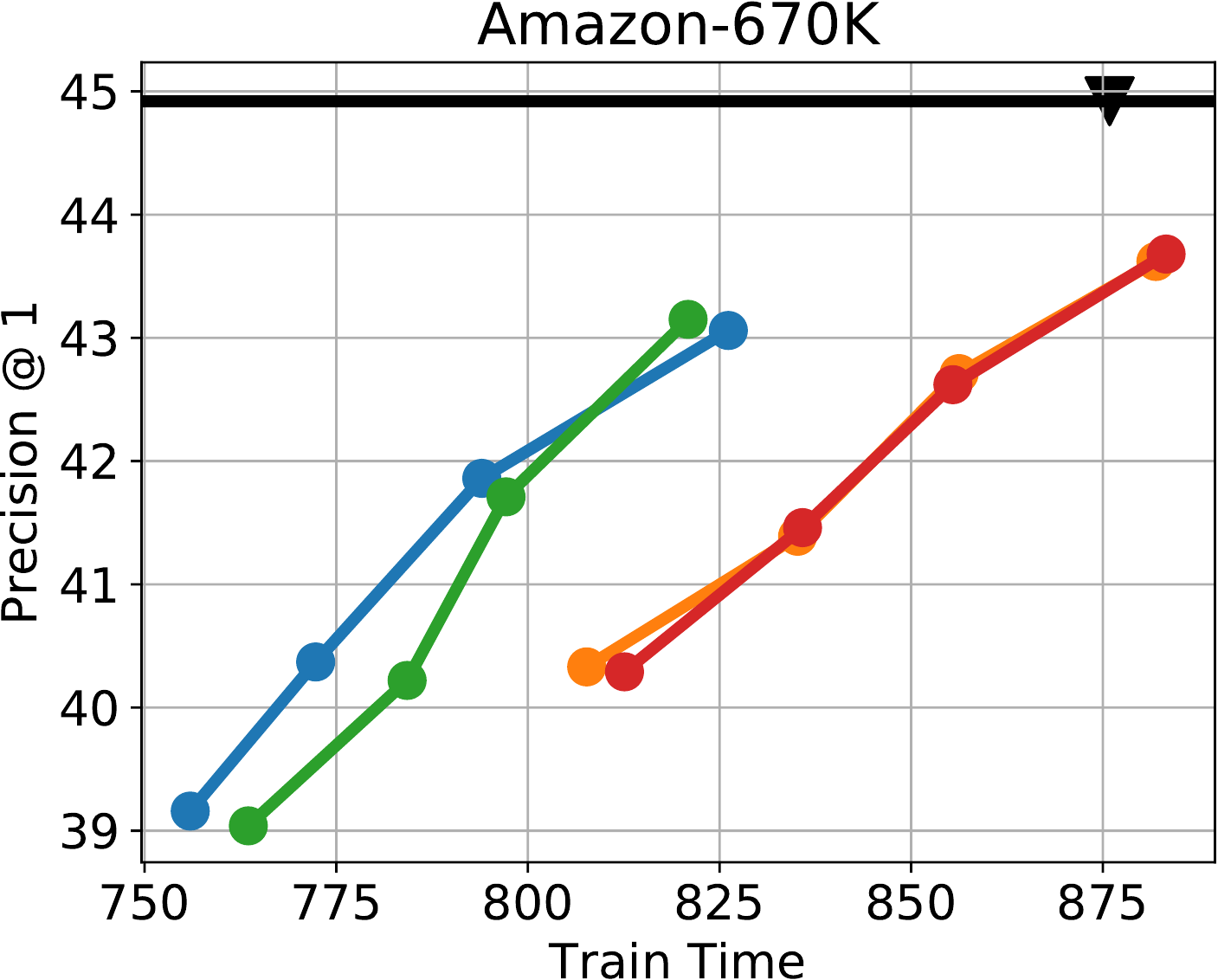}%
\end{subfigure}
\begin{subfigure}[b]{0.245\textwidth}
	\includegraphics[width=\textwidth]{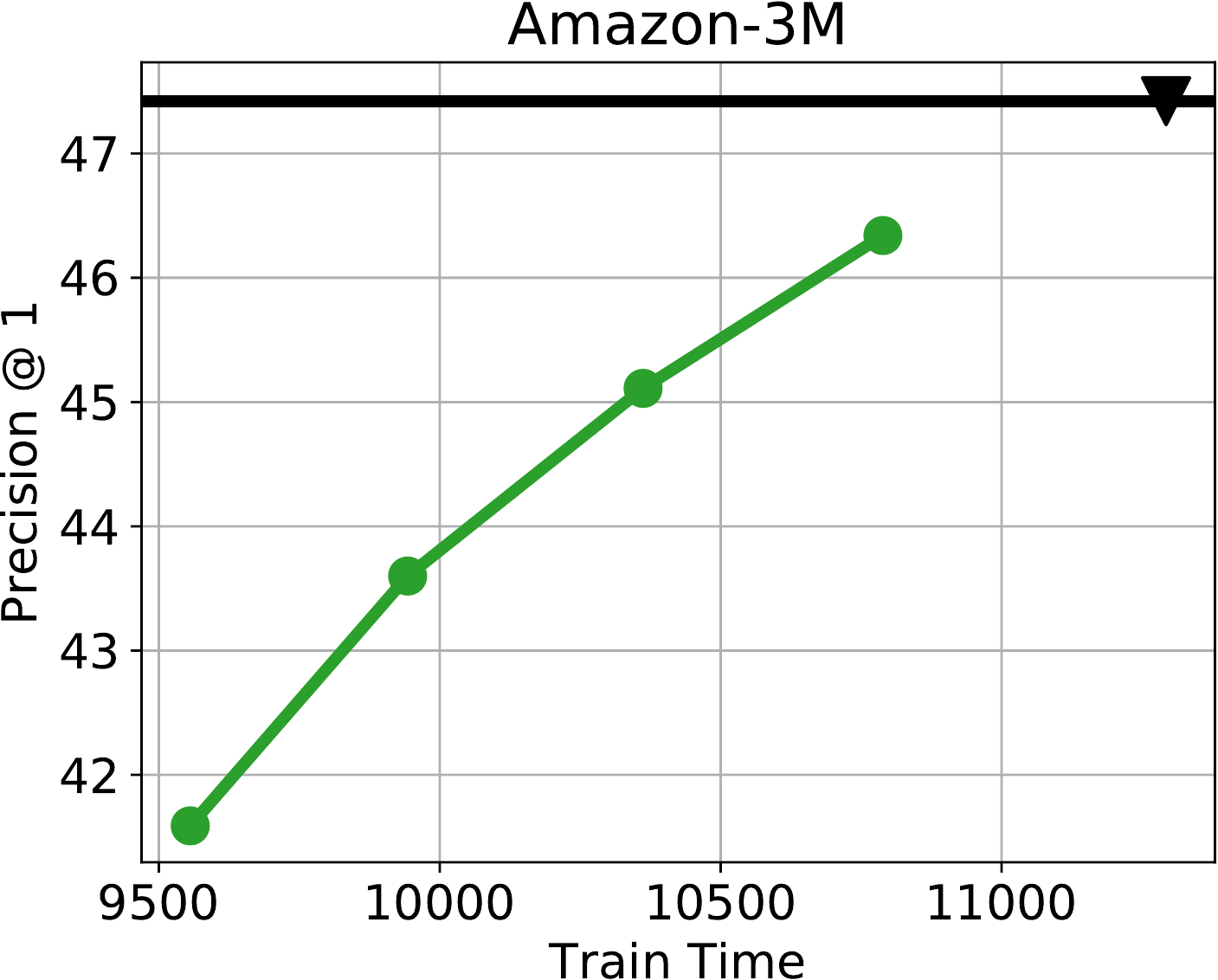}%
\end{subfigure}
\caption{Effect of \defrag variants on reducing total training time. The total training time includes the classifier training time, as well as \defrag's clustering time. Note that the total times are very close to the classifier training time since the hierarchical clustering techniques adopted by \defrag operate very efficiently and put very little overhead on the total training time.}%
\label{fig:app-compare-total-time}%
\end{figure}

\begin{figure}[t]%
\centering
\begin{subfigure}[b]{0.22\textwidth}
	\includegraphics[width=\textwidth]{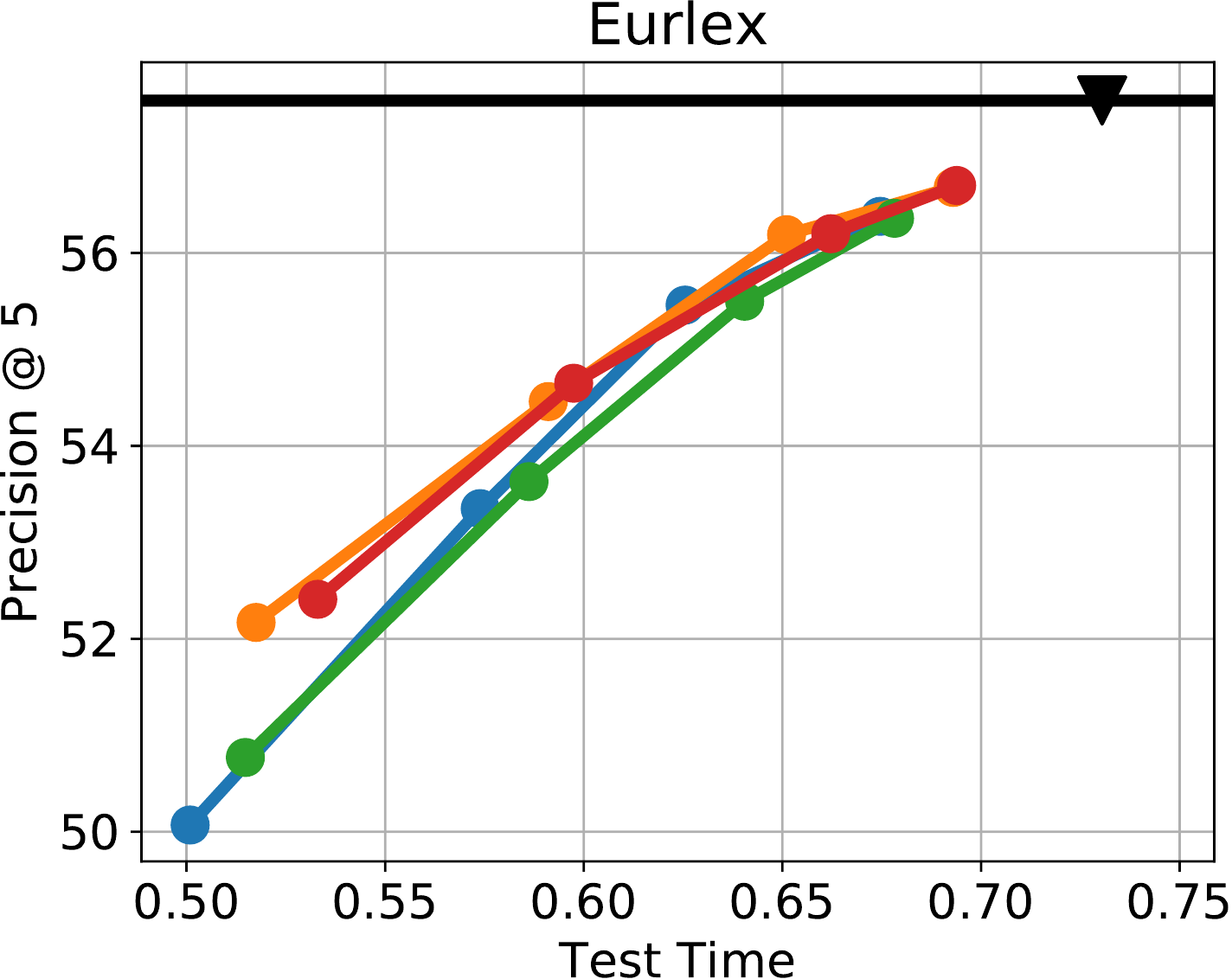}%
\end{subfigure}
\begin{subfigure}[b]{0.22\textwidth}
	\includegraphics[width=\textwidth]{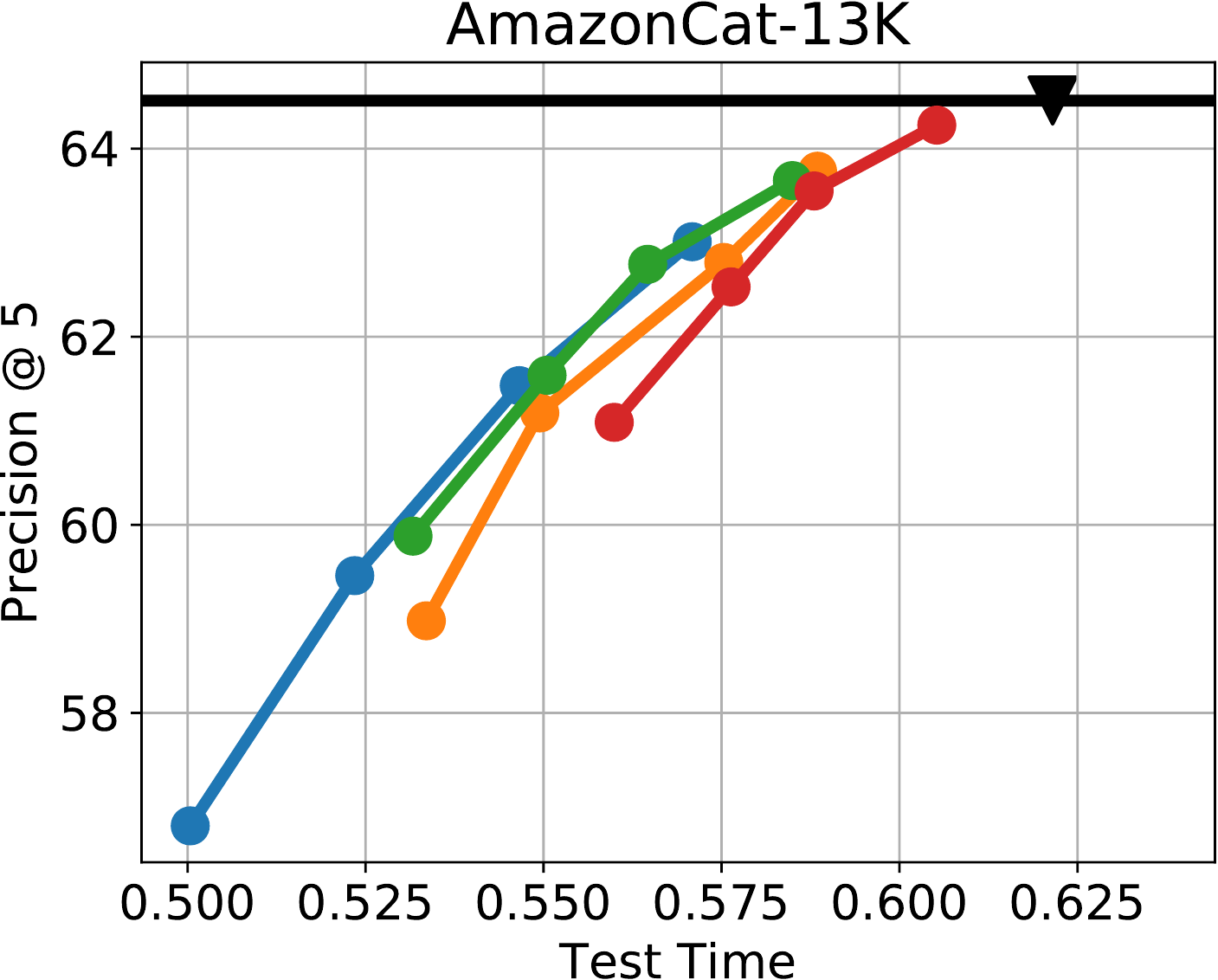}%
\end{subfigure}
\begin{subfigure}[b]{0.22\textwidth}
	\includegraphics[width=\textwidth]{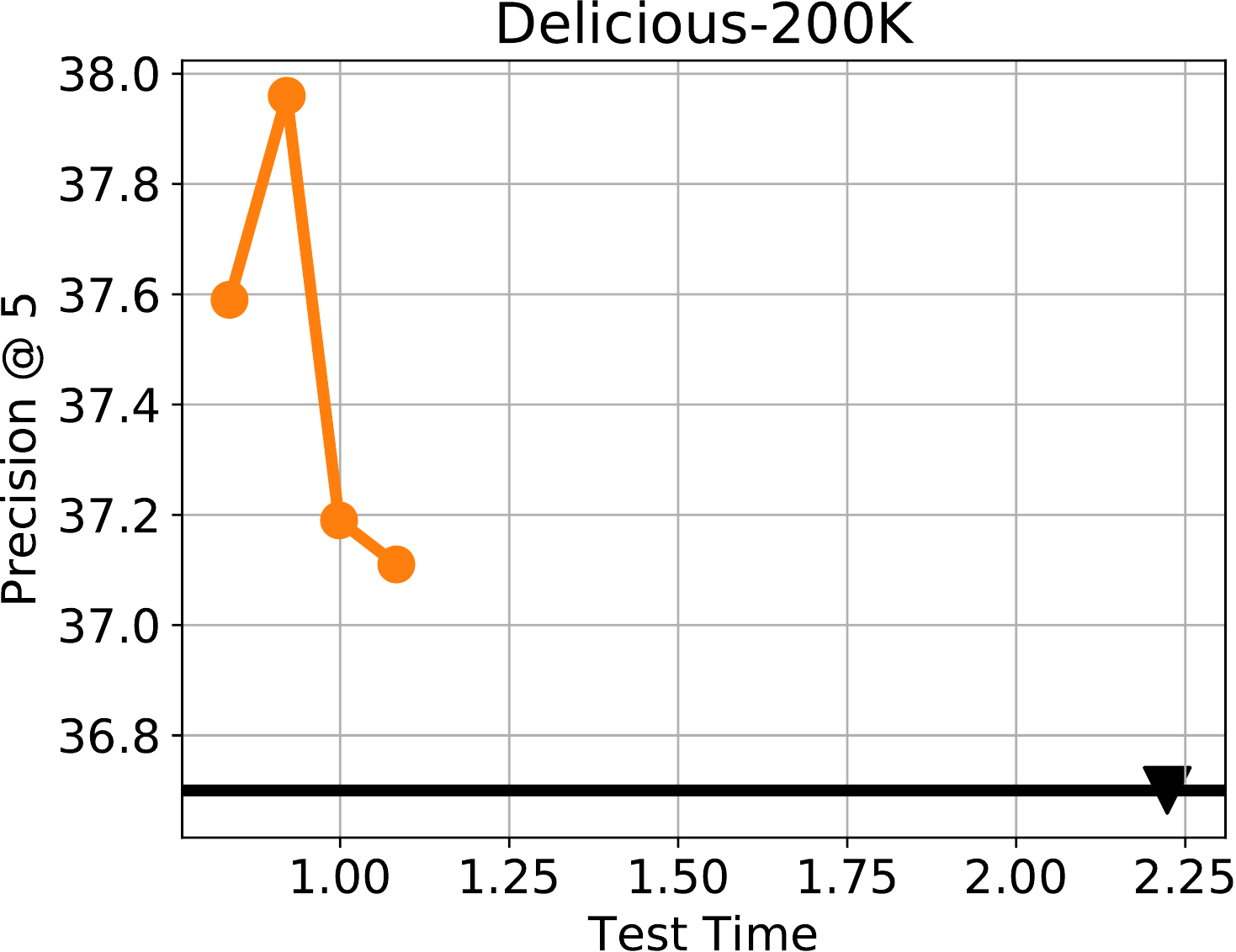}%
\end{subfigure}
\begin{subfigure}[b]{0.32\textwidth}
	\includegraphics[width=\textwidth]{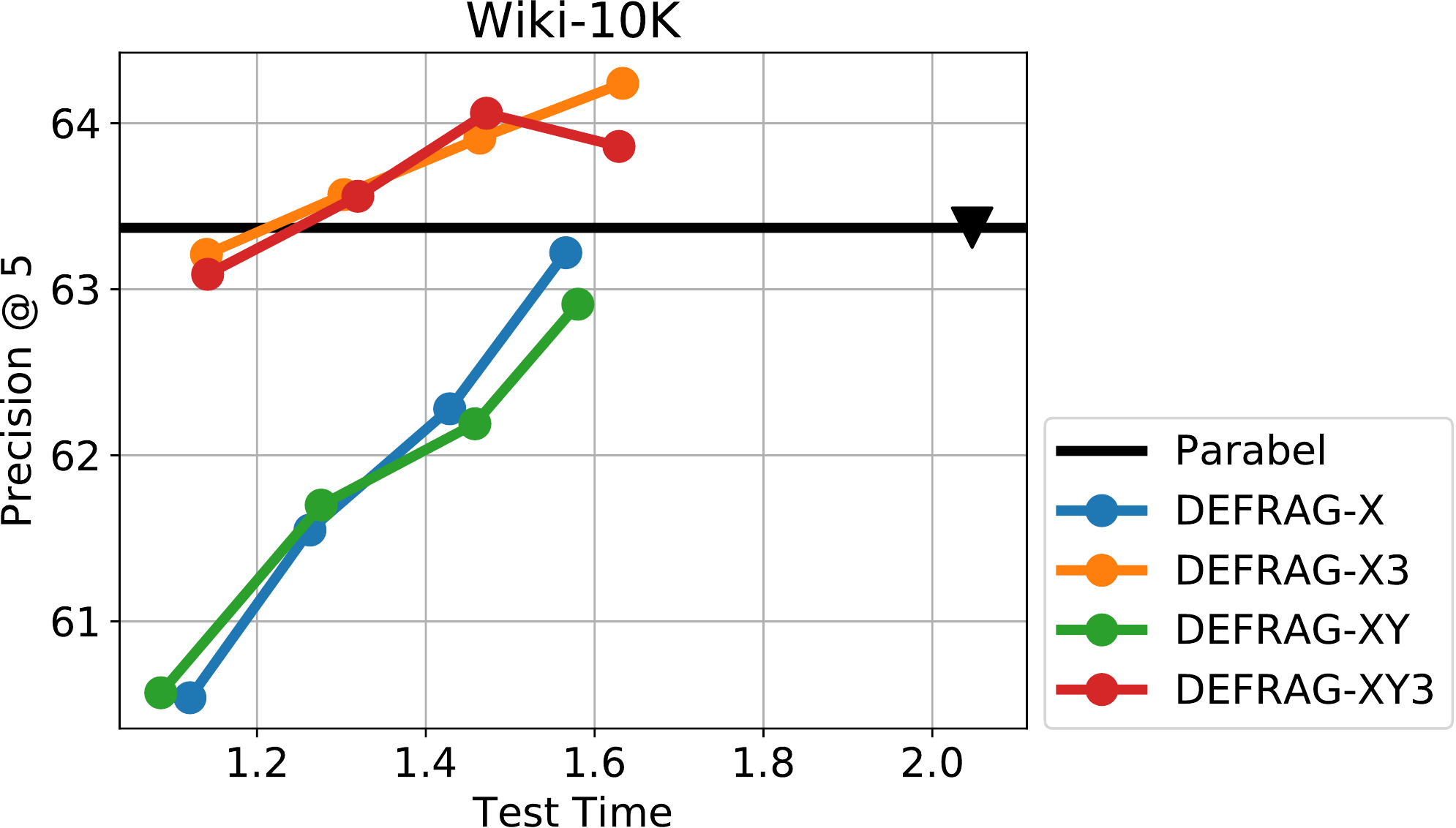}%
\end{subfigure}
\begin{subfigure}[b]{0.245\textwidth}
	\includegraphics[width=\textwidth]{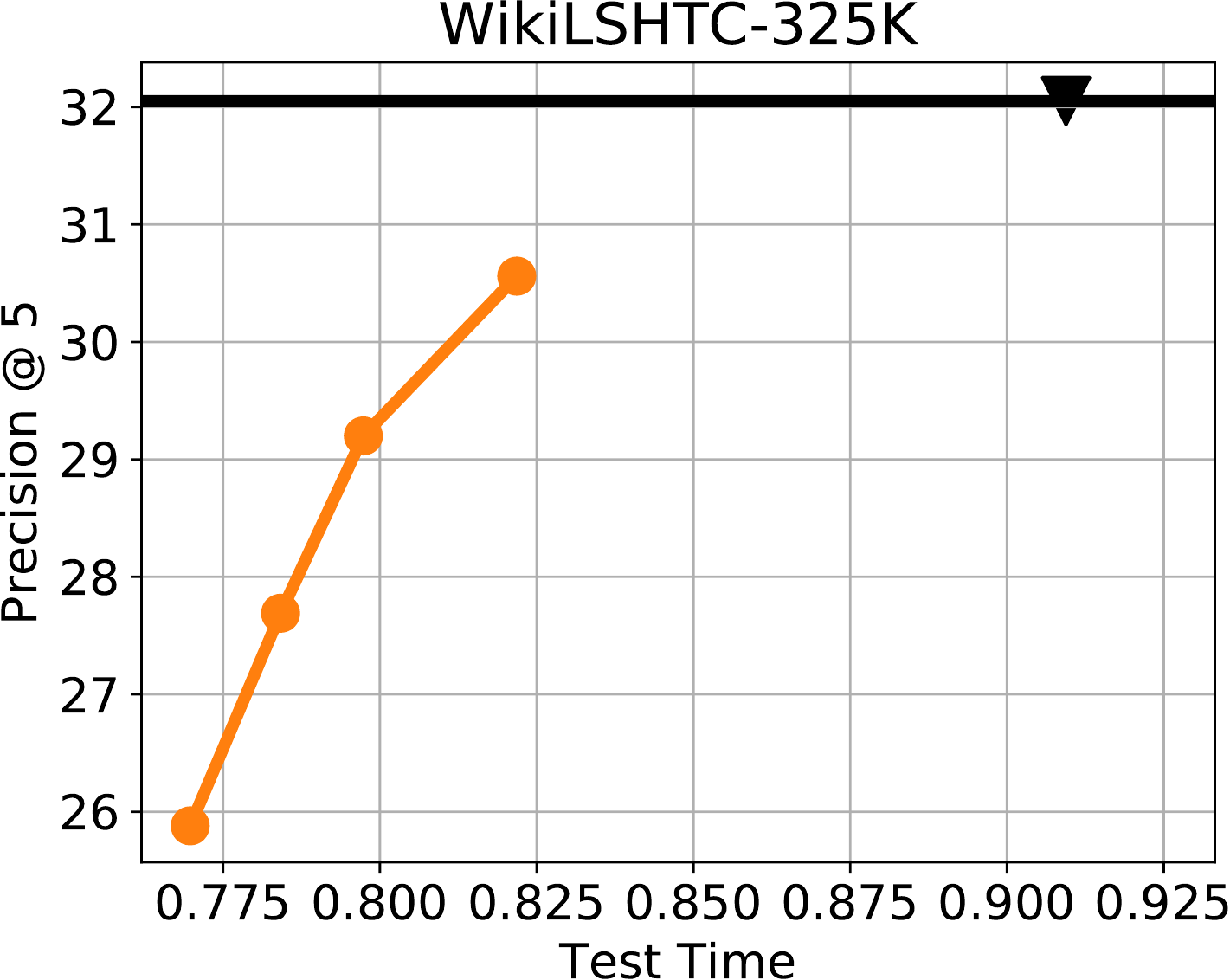}%
\end{subfigure}
\begin{subfigure}[b]{0.245\textwidth}
	\includegraphics[width=\textwidth]{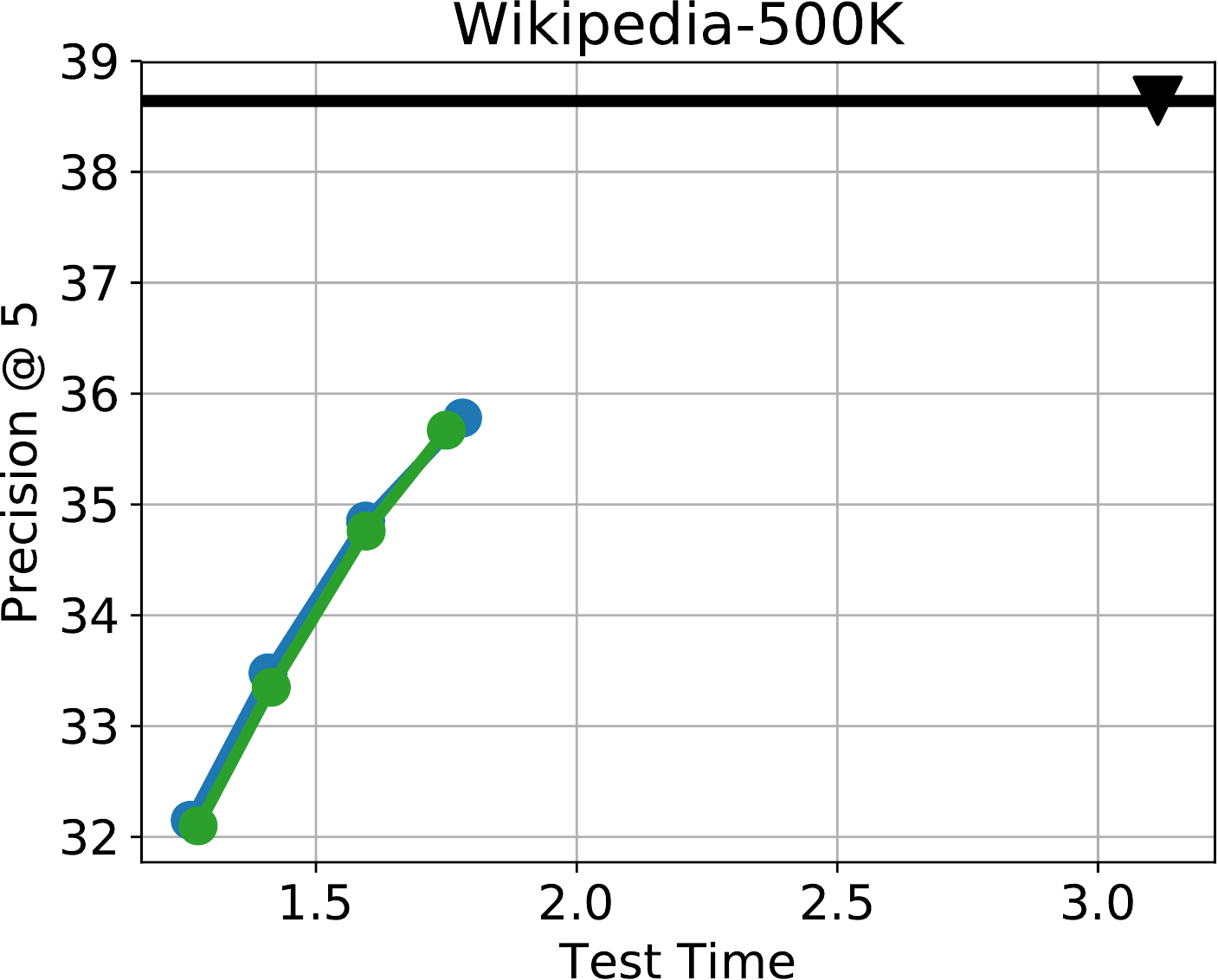}%
\end{subfigure}
\begin{subfigure}[b]{0.245\textwidth}
	\includegraphics[width=\textwidth]{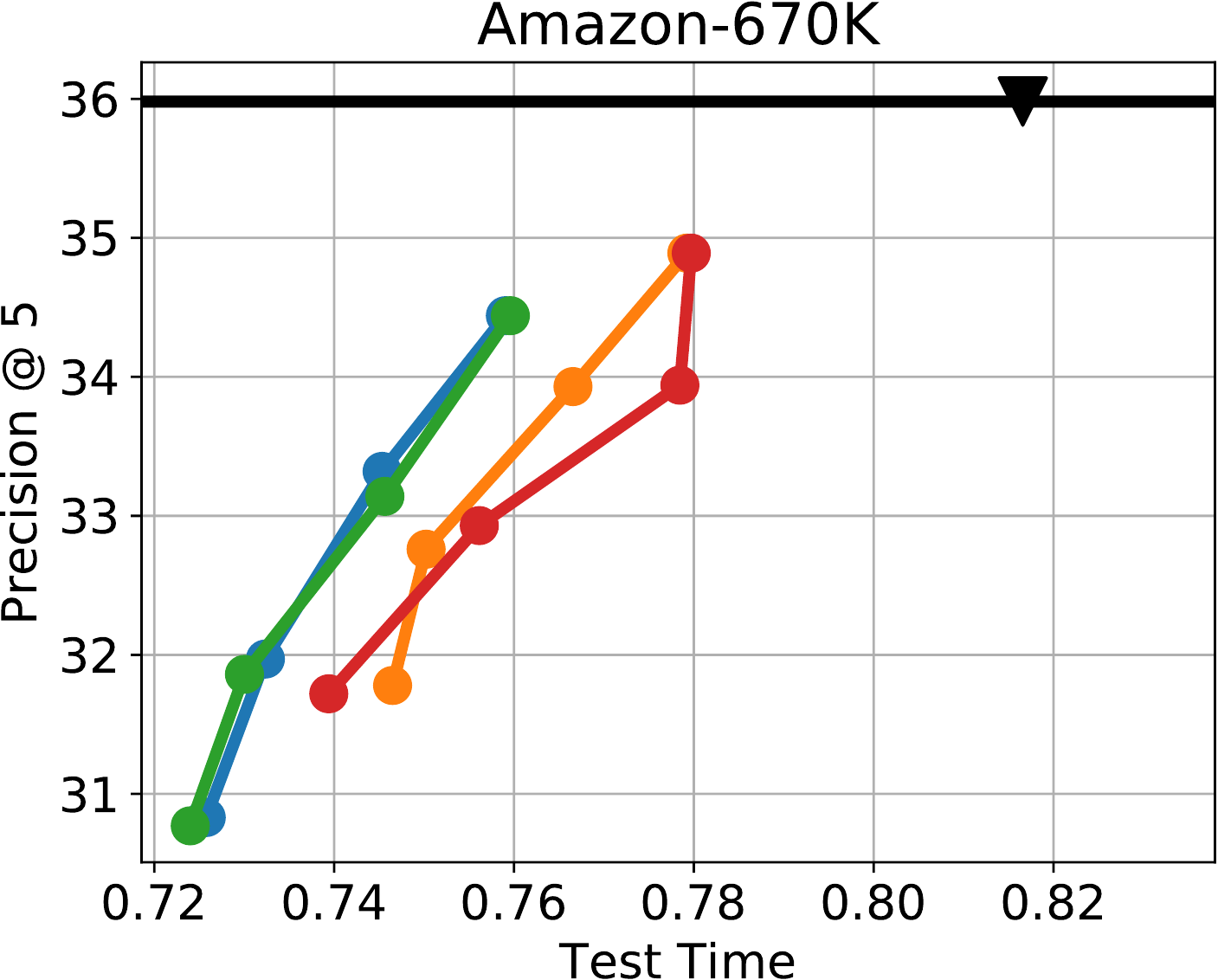}%
\end{subfigure}
\begin{subfigure}[b]{0.245\textwidth}
	\includegraphics[width=\textwidth]{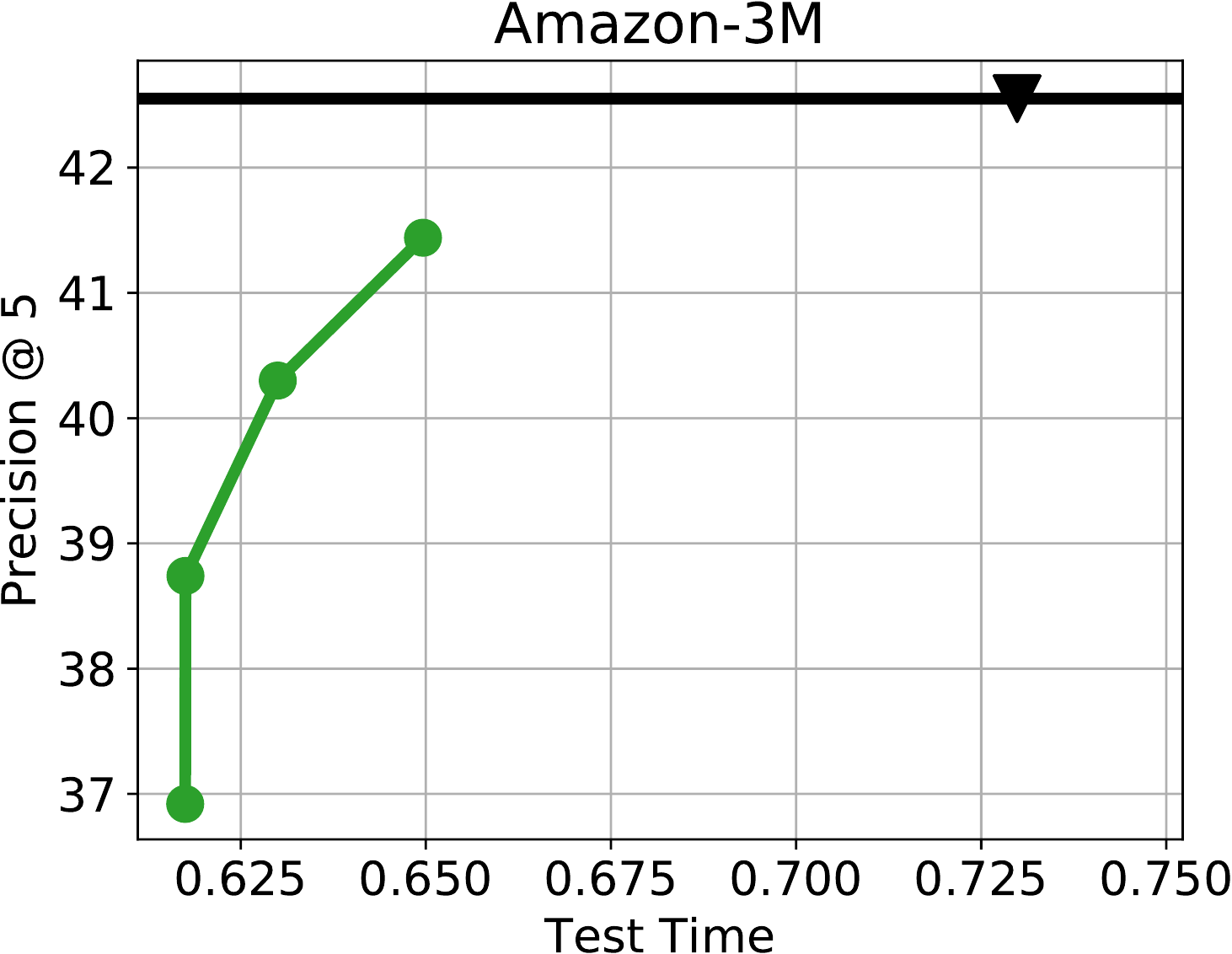}%
\end{subfigure}
\caption{Effect of \defrag variants on reducing prediction time. Note that the prediction time includes the feature agglomeration time as well as the actual prediction time. The performance of the classifier is indicated using the precision@5 metric. The feature agglomeration step being so inexpensive, allows \defrag to offer drastic reductions in prediction times on most datasets, for example Delicious (70\% reduction), Wiki10 (40\% reduction), Wikipedia-500K (more than 40\% reduction)}%
\label{fig:app-compare-test-time}%
\end{figure}

\begin{figure}[t]%
\centering
\begin{subfigure}[b]{0.22\textwidth}
	\includegraphics[width=\textwidth]{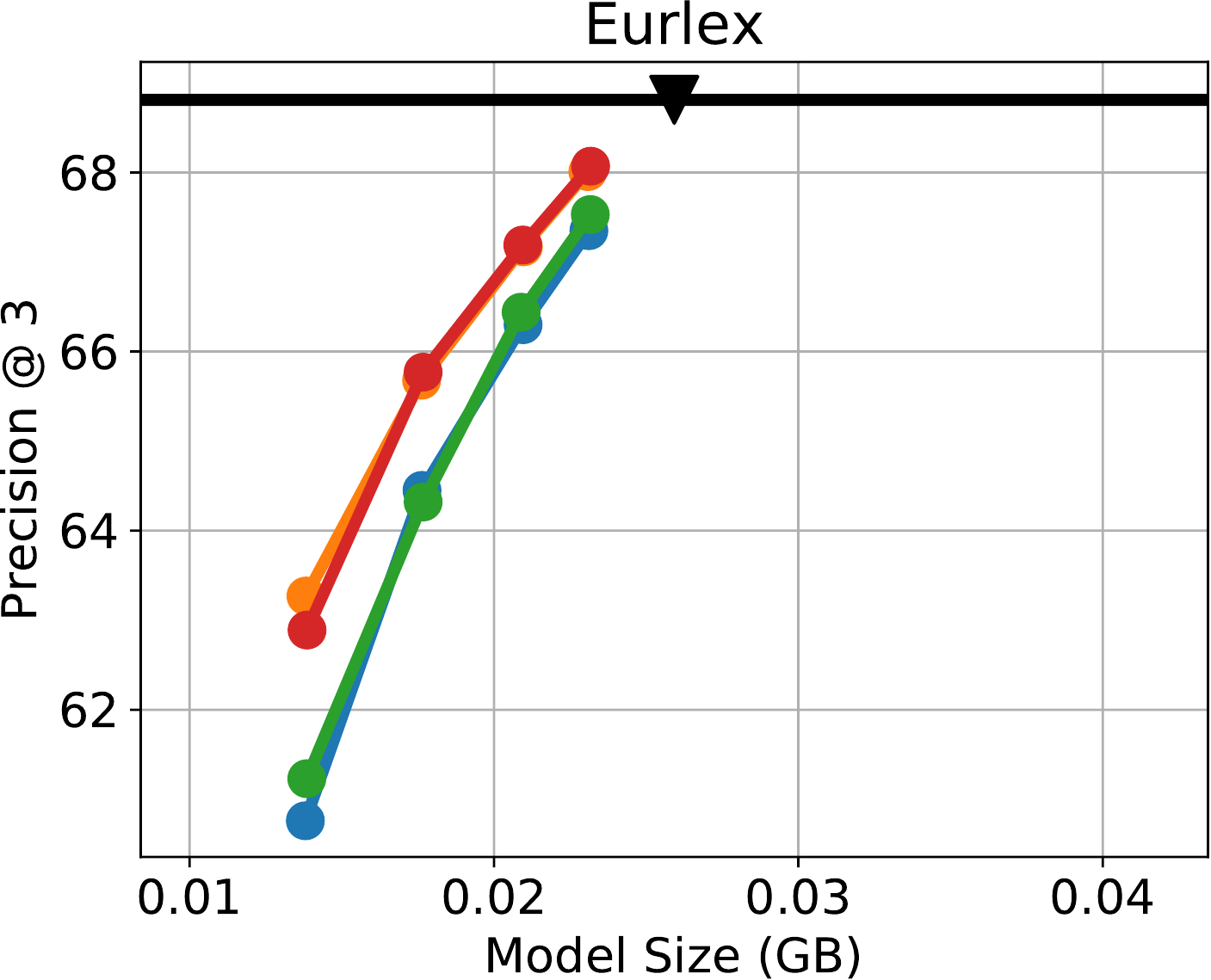}%
\end{subfigure}
\begin{subfigure}[b]{0.22\textwidth}
	\includegraphics[width=\textwidth]{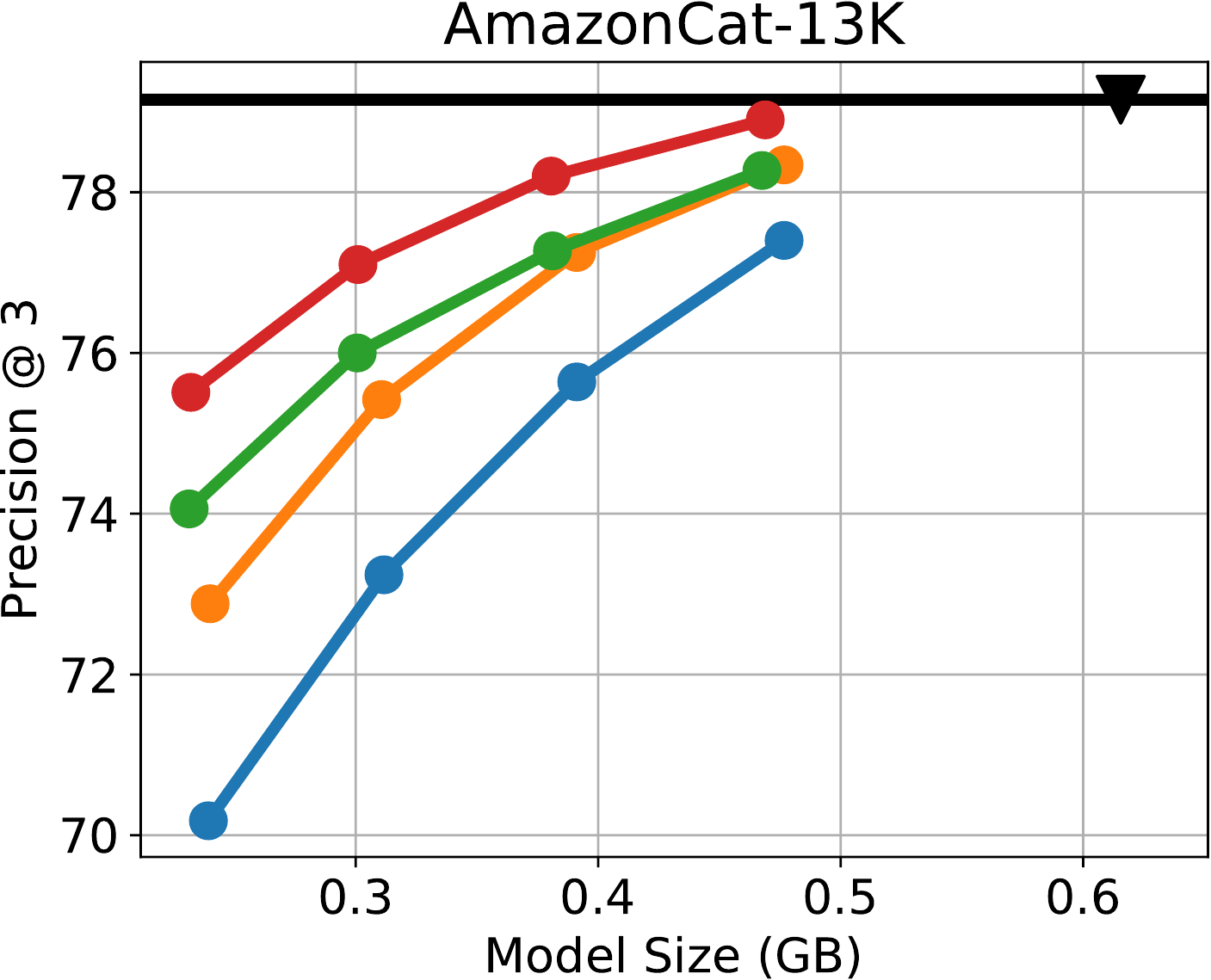}%
\end{subfigure}
\begin{subfigure}[b]{0.22\textwidth}
	\includegraphics[width=\textwidth]{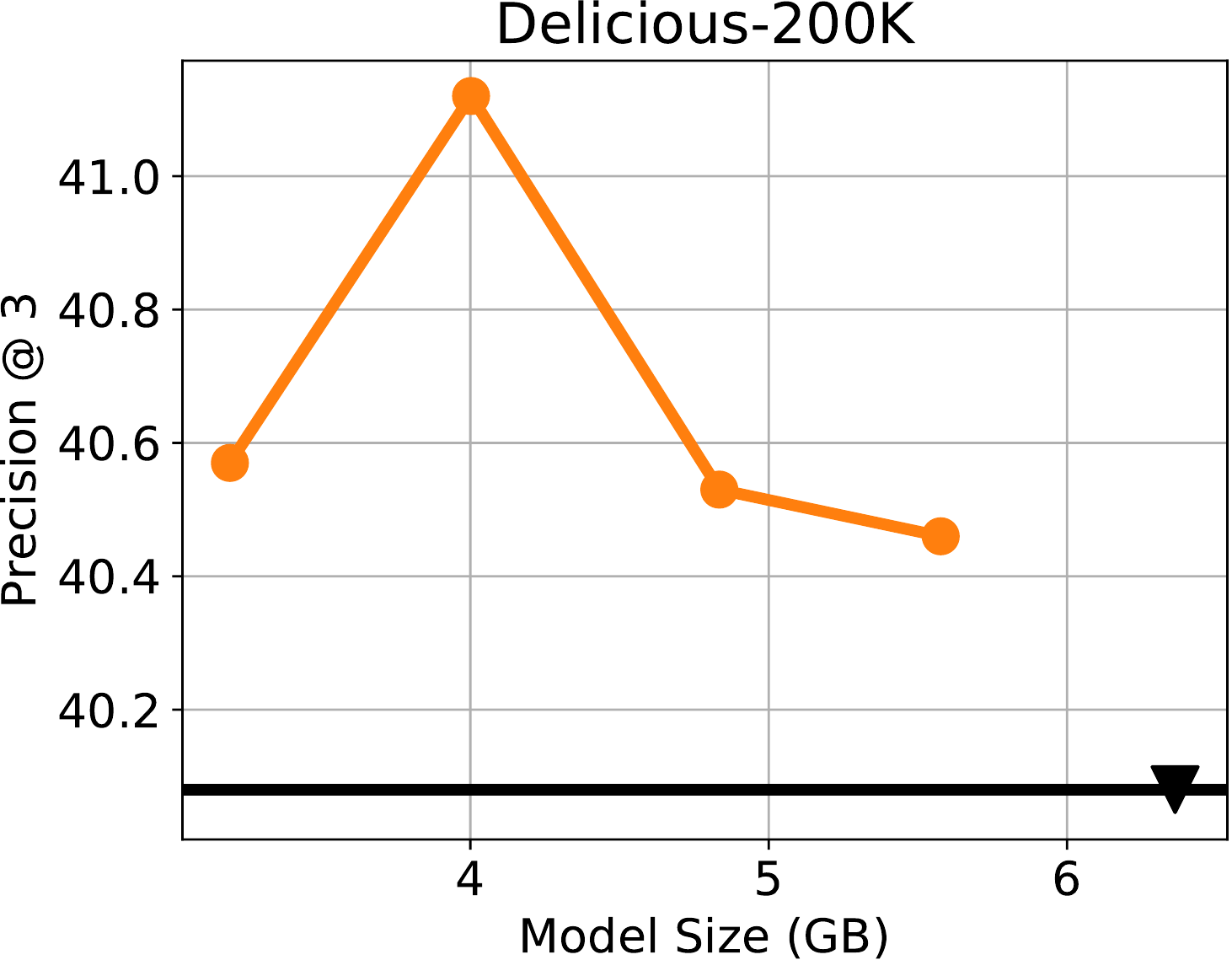}%
\end{subfigure}
\begin{subfigure}[b]{0.32\textwidth}
	\includegraphics[width=\textwidth]{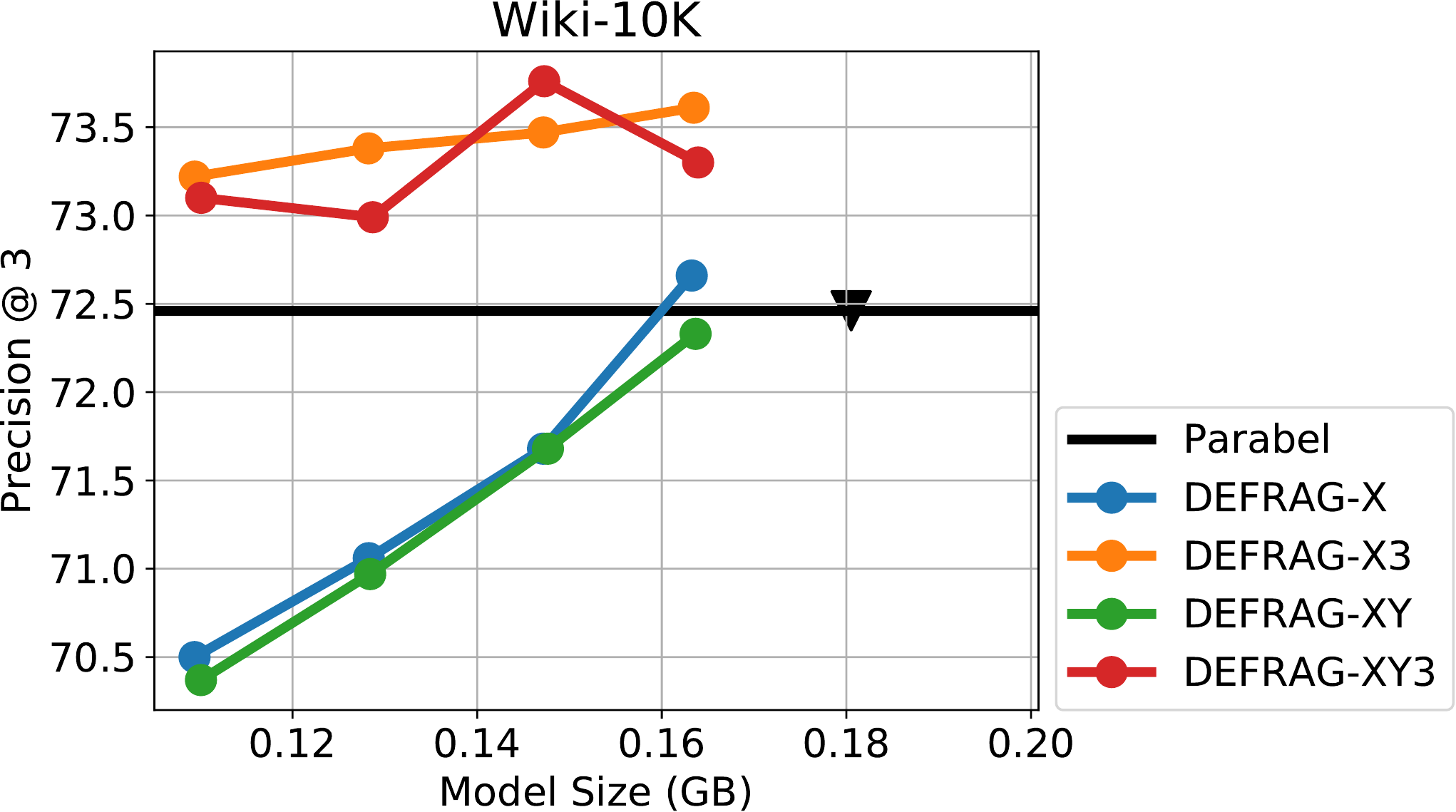}%
\end{subfigure}
\begin{subfigure}[b]{0.245\textwidth}
	\includegraphics[width=\textwidth]{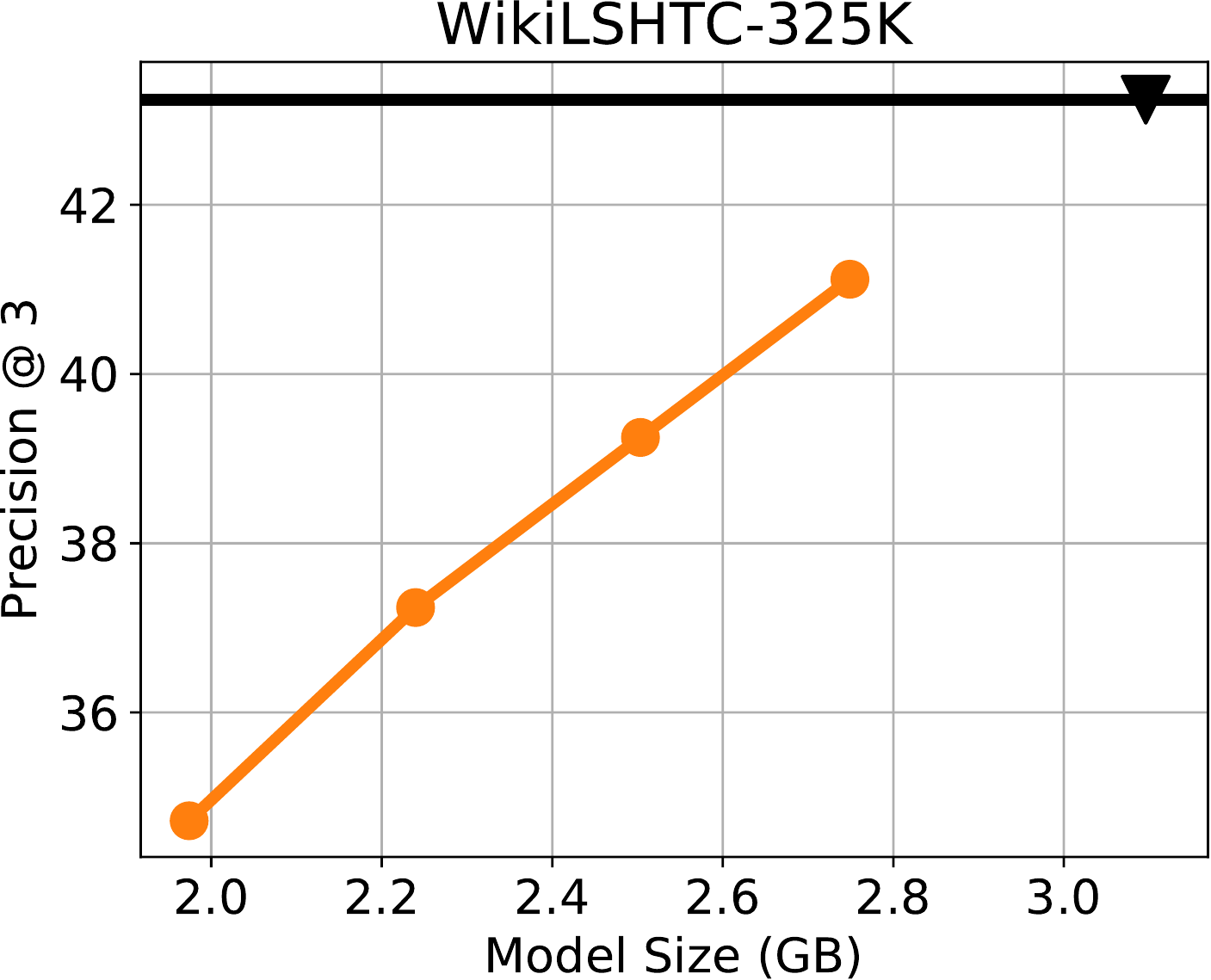}%
\end{subfigure}
\begin{subfigure}[b]{0.245\textwidth}
	\includegraphics[width=\textwidth]{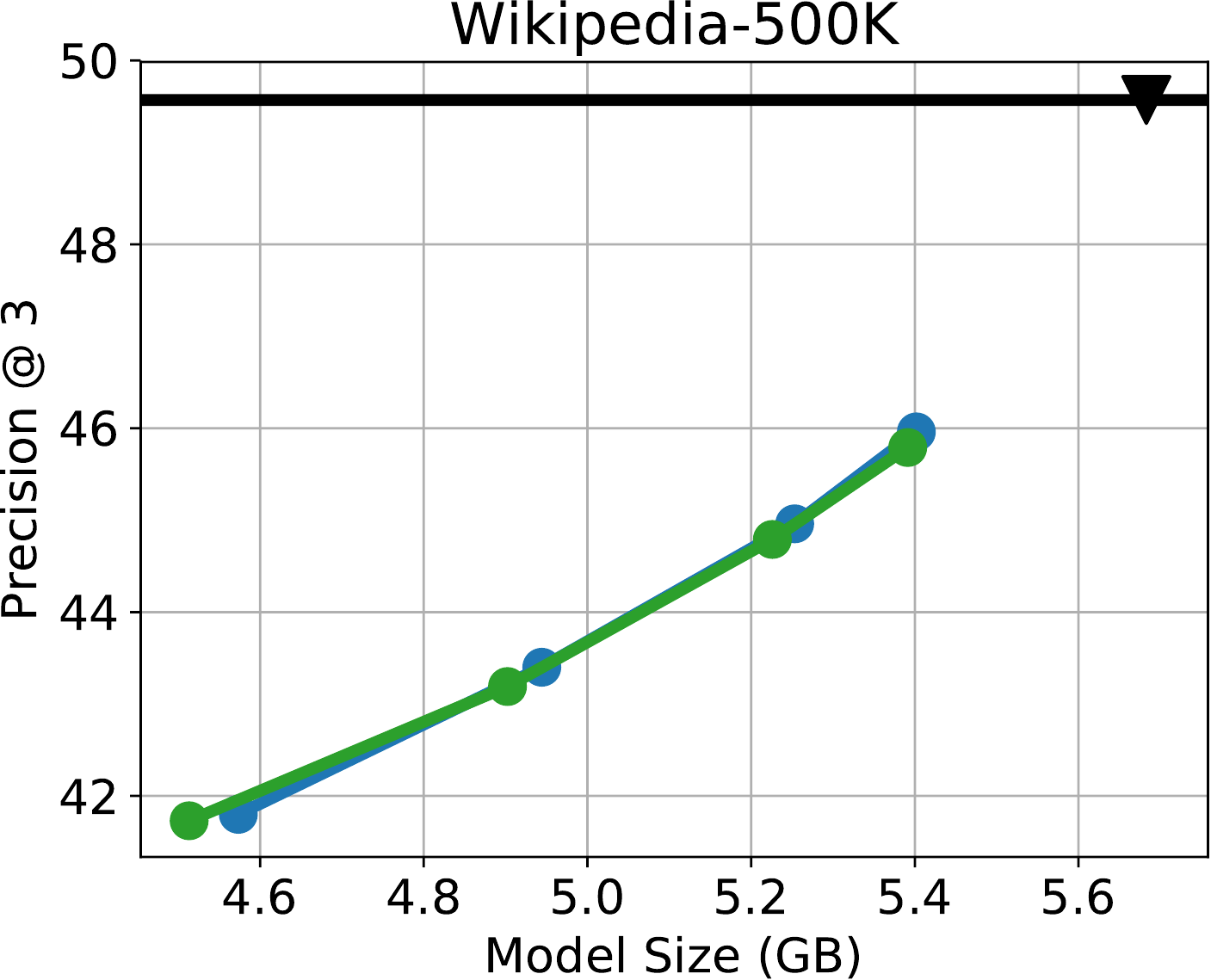}%
\end{subfigure}
\begin{subfigure}[b]{0.245\textwidth}
	\includegraphics[width=\textwidth]{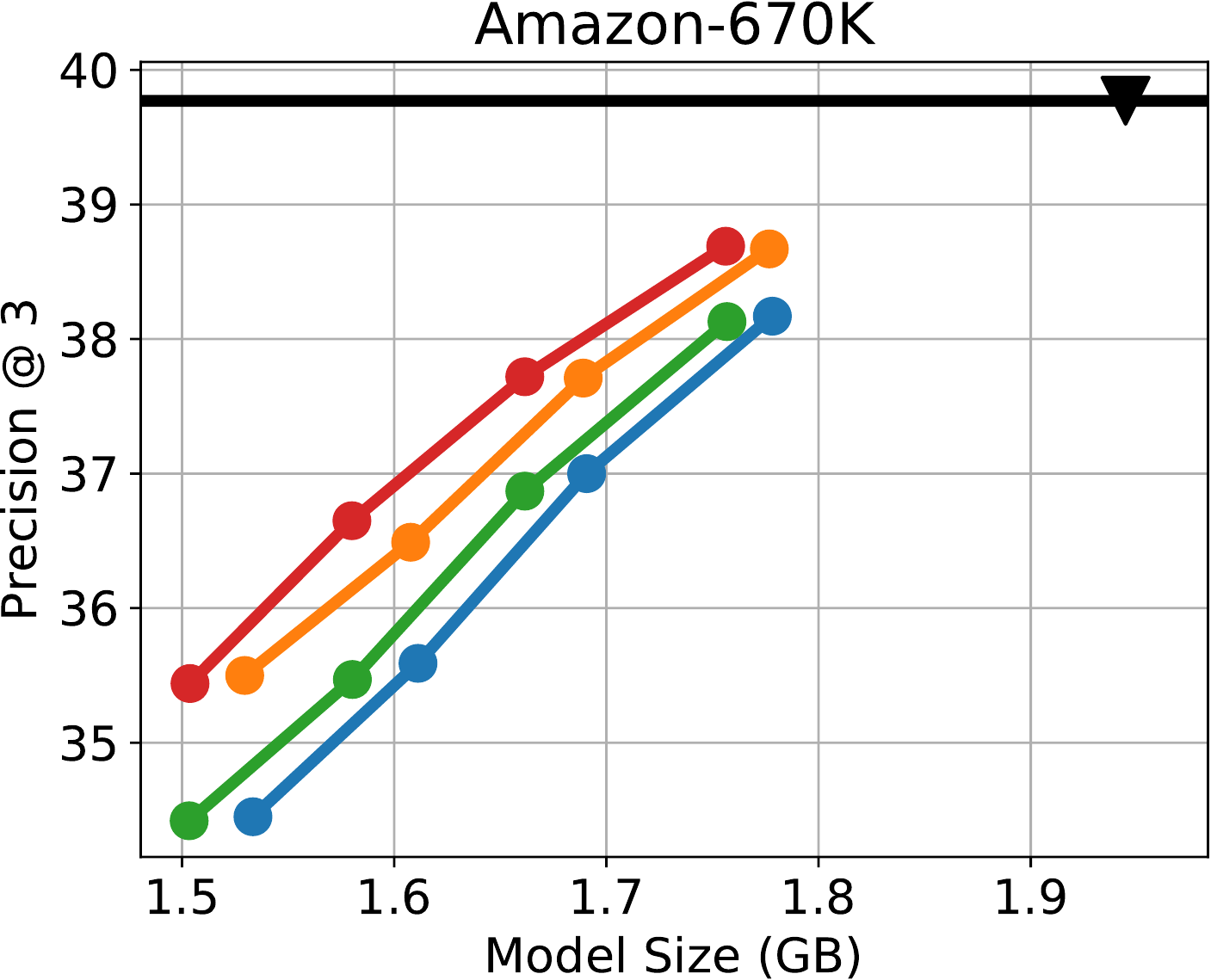}%
\end{subfigure}
\begin{subfigure}[b]{0.245\textwidth}
	\includegraphics[width=\textwidth]{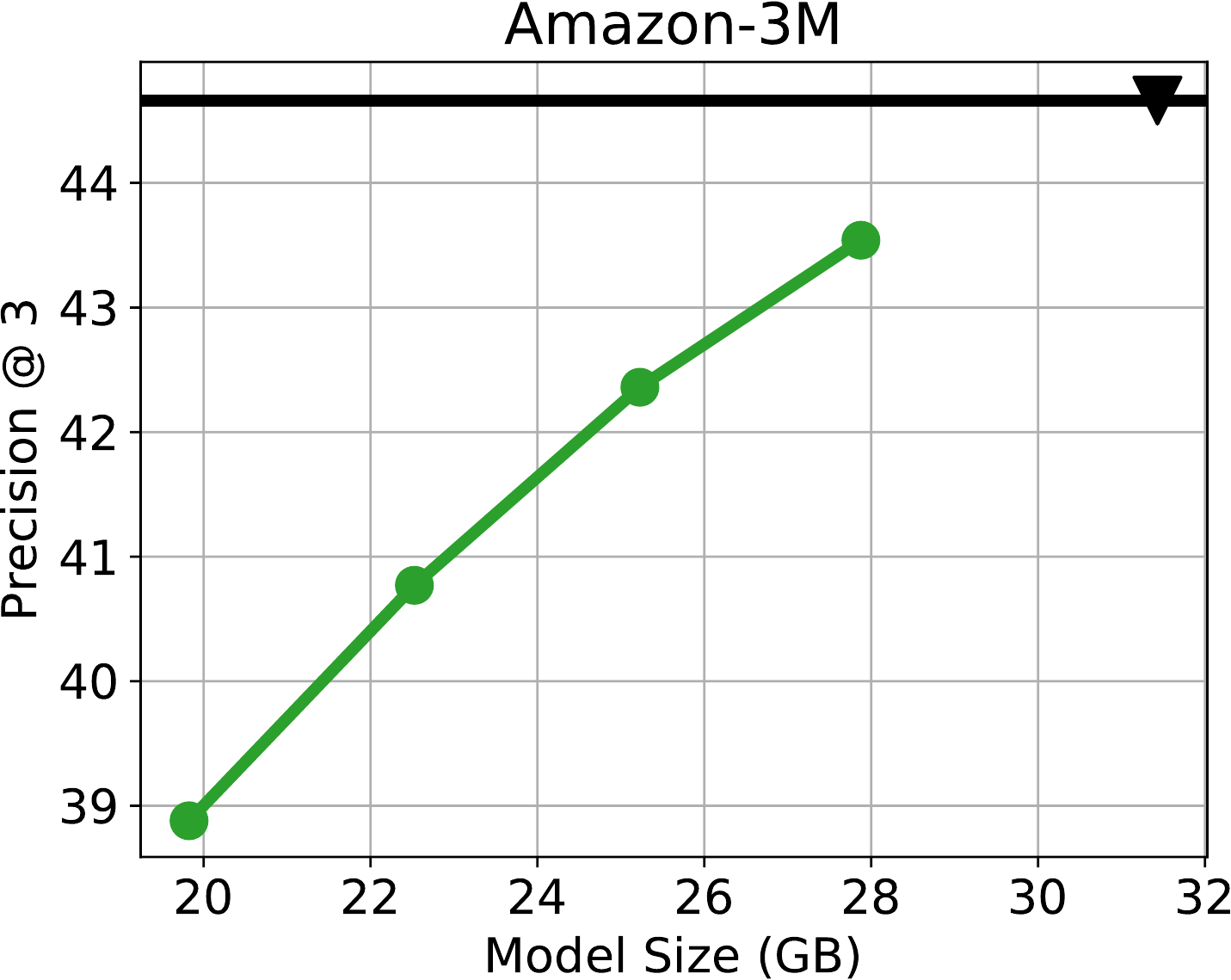}%
\end{subfigure}
\caption{Effect of \defrag variants on reducing model sizes. The performance of the classifier is indicated using the precision@3 metric. Since the model sizes of most algorithms depend on the dimensionality of the data, by performing feature agglomeration, \defrag variants offer model size reductions as well.}%
\label{fig:app-compare-model-size}%
\end{figure}

\begin{figure}[t]%
\centering
\begin{subfigure}[b]{0.22\textwidth}
	\includegraphics[width=\textwidth]{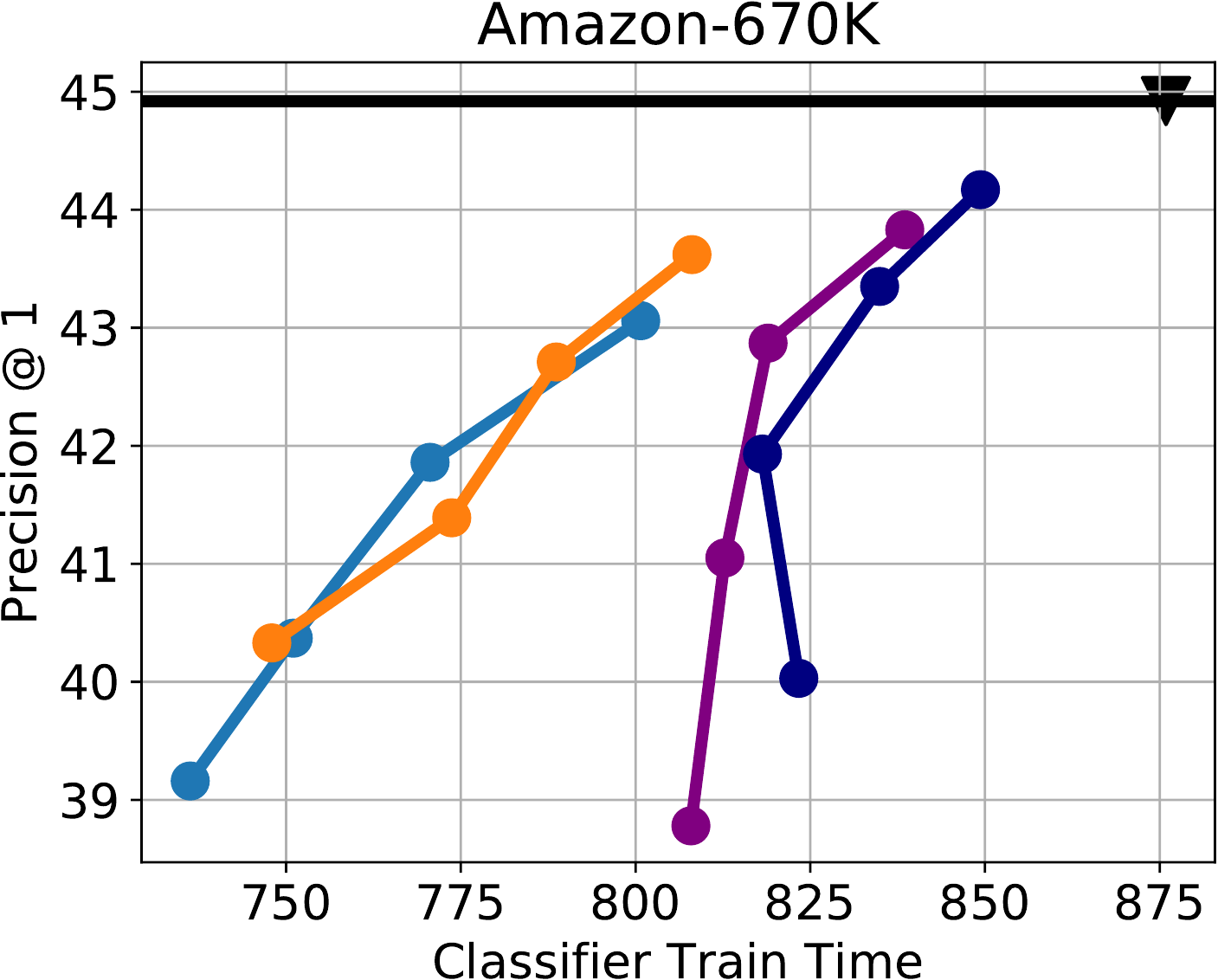}%
\end{subfigure}
\begin{subfigure}[b]{0.22\textwidth}
	\includegraphics[width=\textwidth]{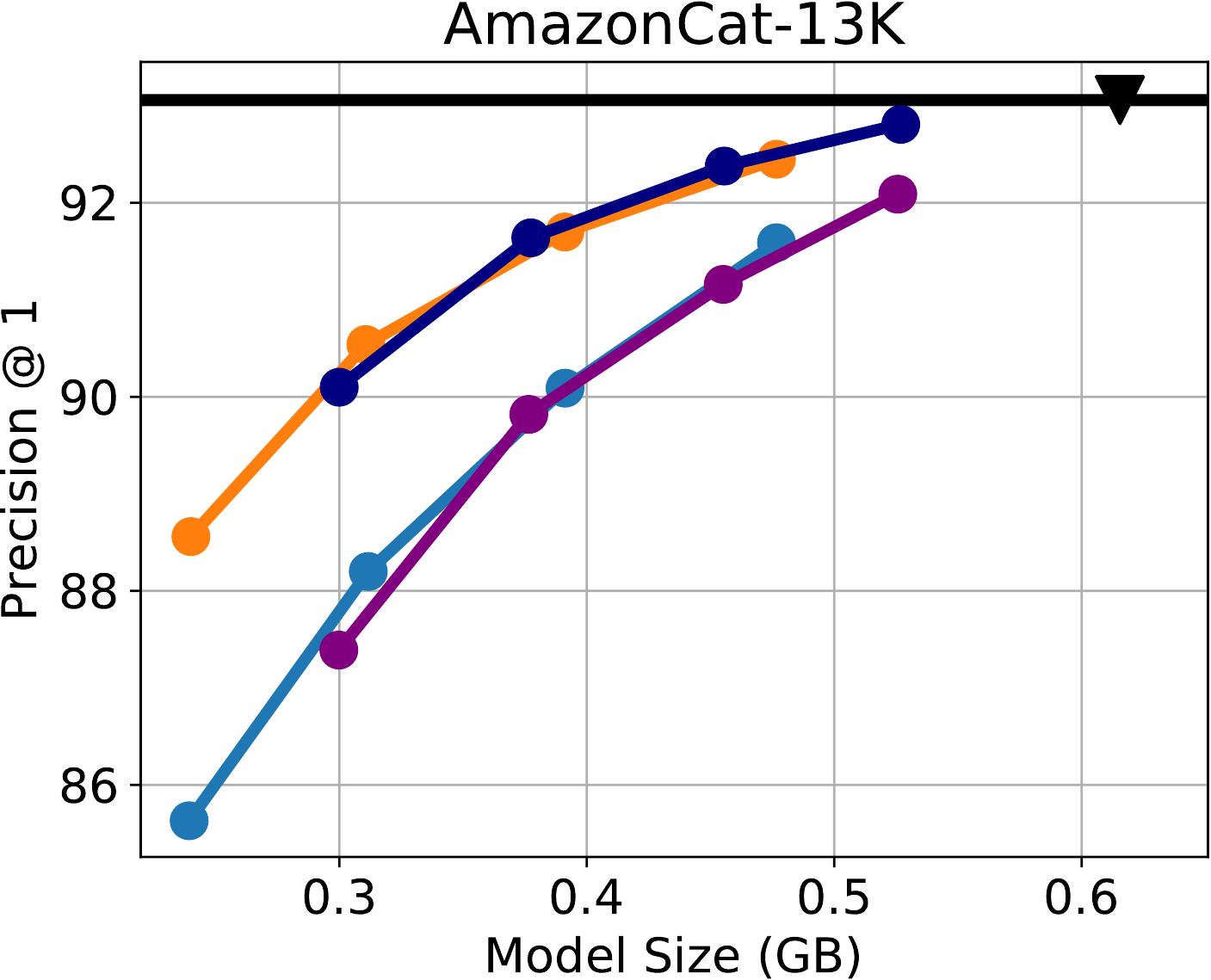}%
\end{subfigure}
\begin{subfigure}[b]{0.22\textwidth}
	\includegraphics[width=\textwidth]{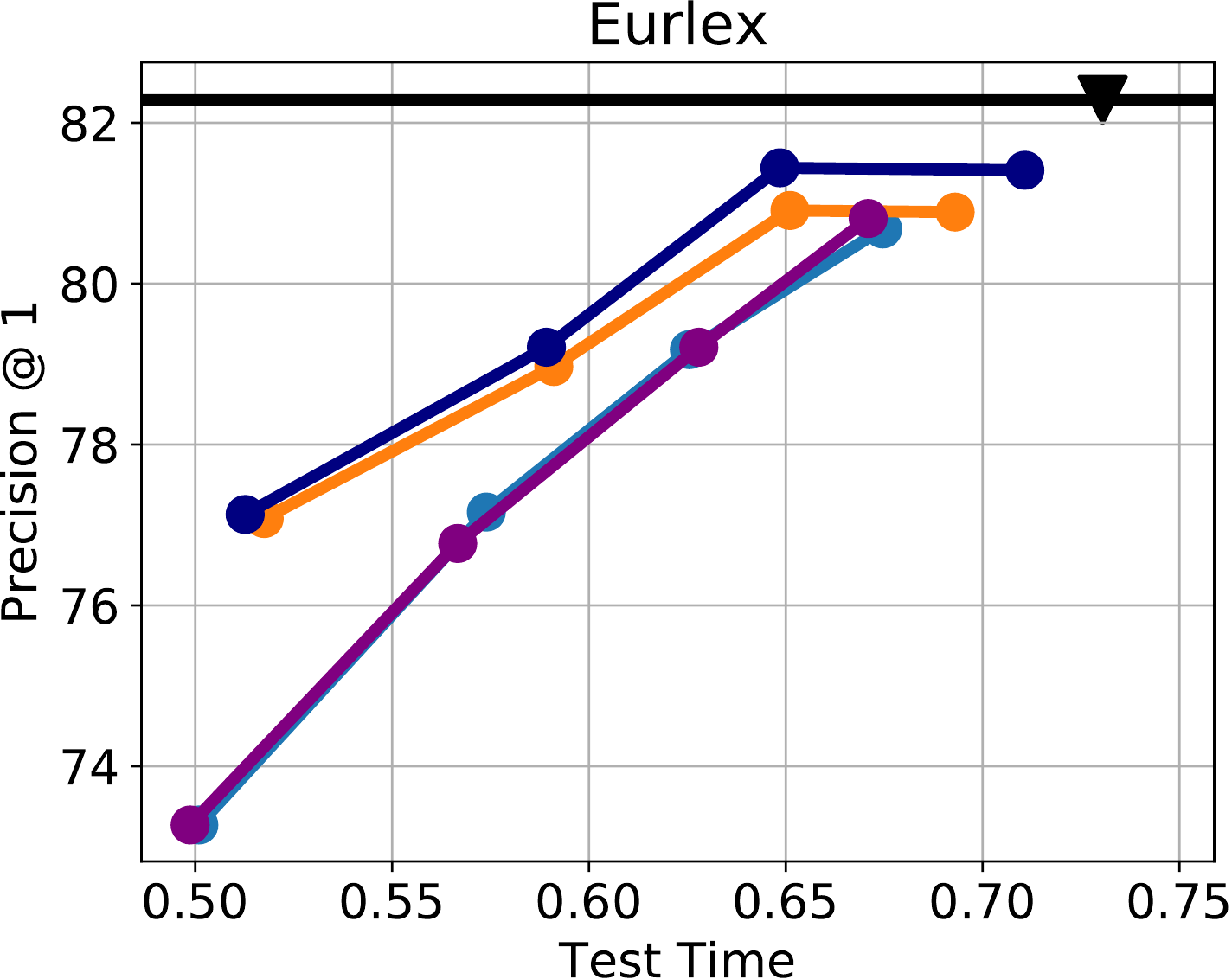}%
\end{subfigure}
\begin{subfigure}[b]{0.32\textwidth}
	\includegraphics[width=\textwidth]{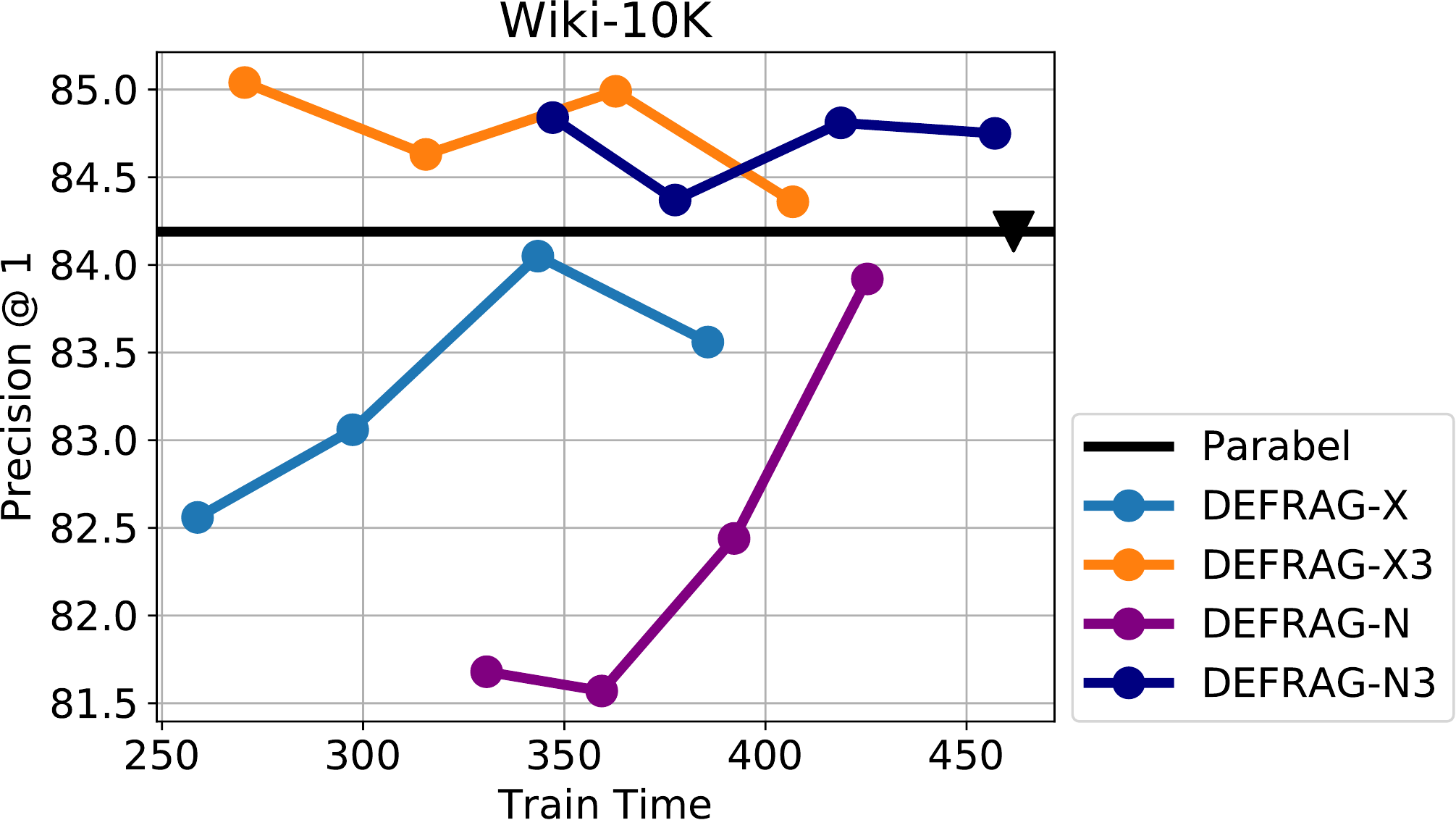}%
\end{subfigure}\\
\begin{subfigure}[b]{0.245\textwidth}
	\includegraphics[width=\textwidth]{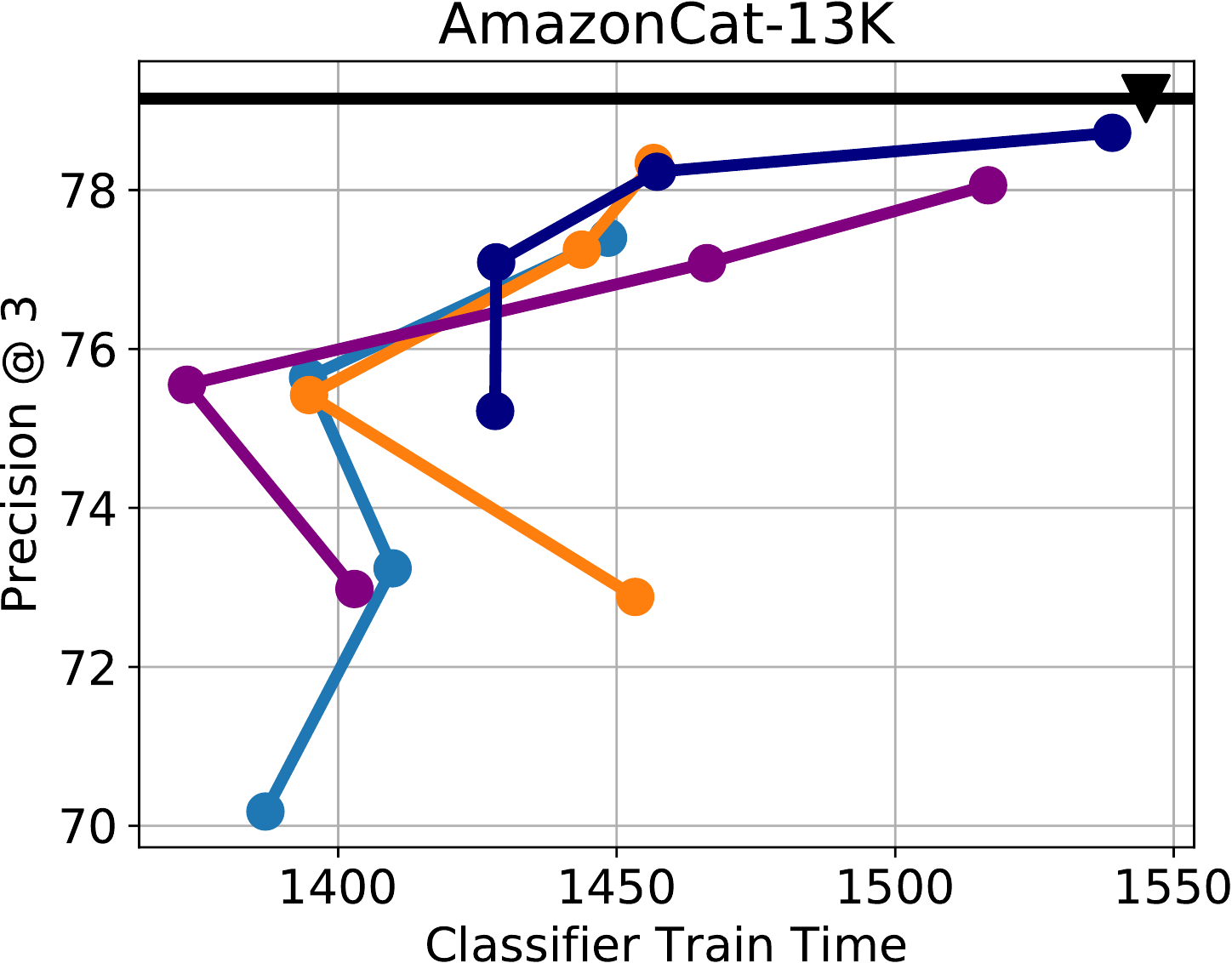}%
\end{subfigure}
\begin{subfigure}[b]{0.245\textwidth}
	\includegraphics[width=\textwidth]{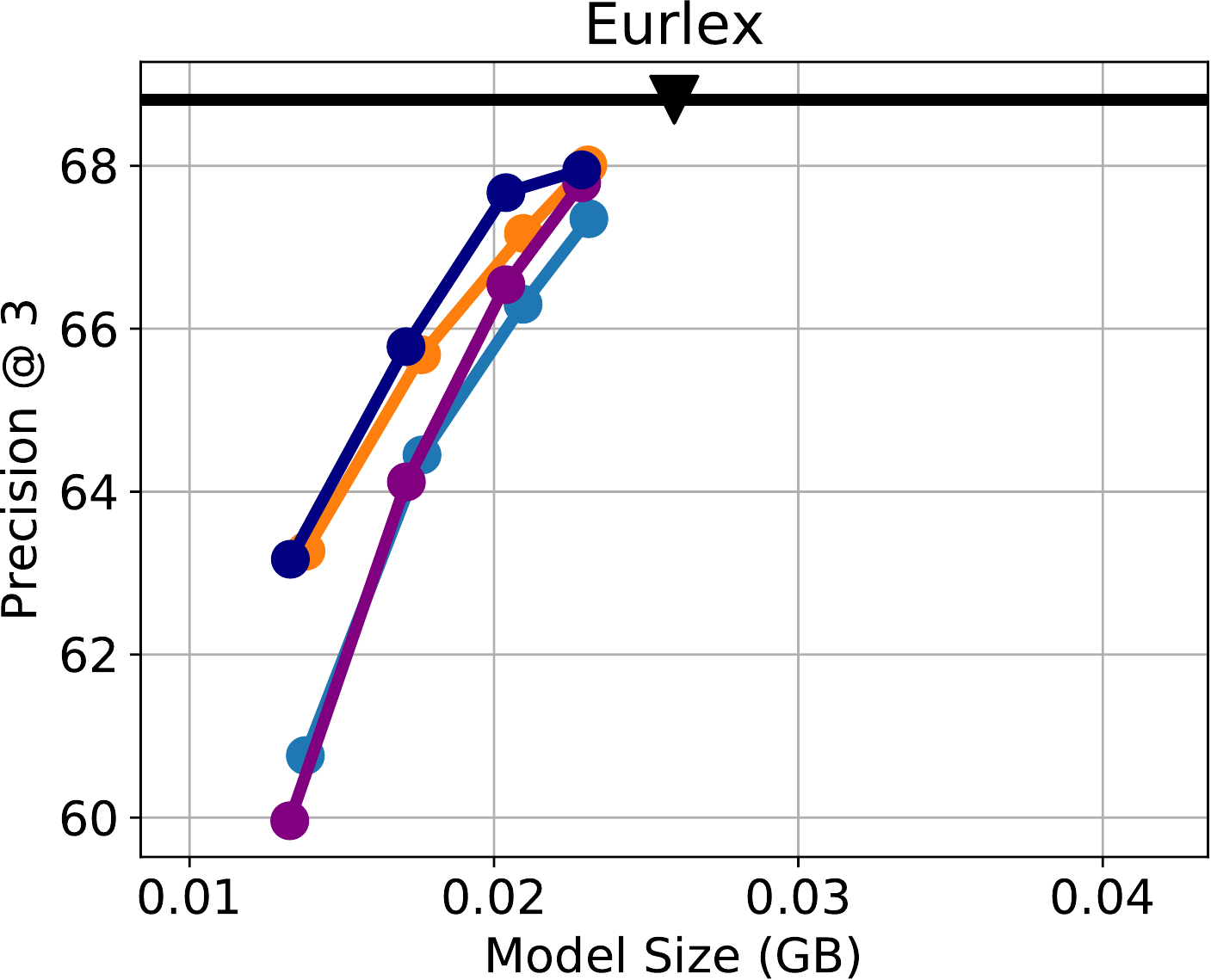}%
\end{subfigure}
\begin{subfigure}[b]{0.245\textwidth}
	\includegraphics[width=\textwidth]{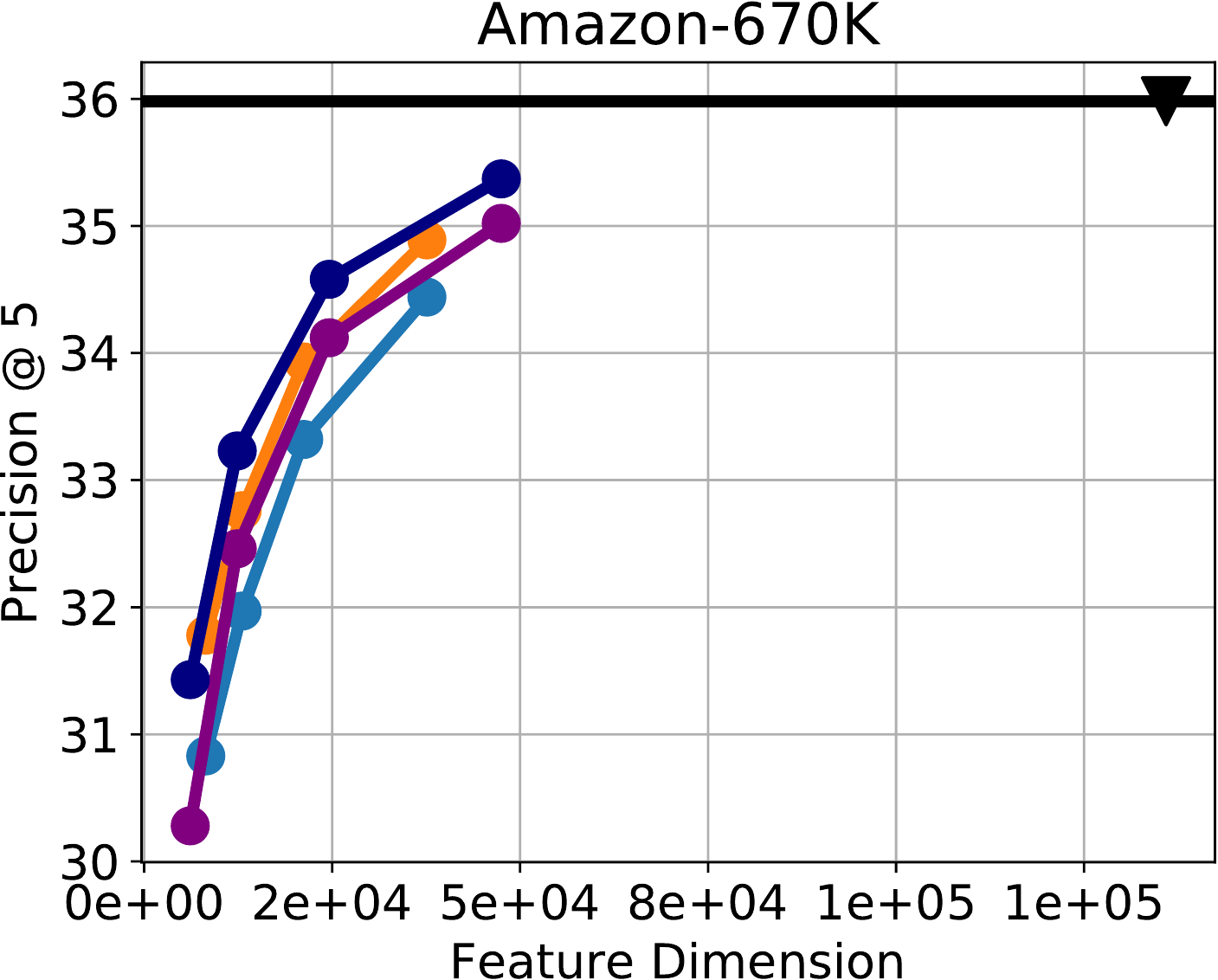}%
\end{subfigure}
\begin{subfigure}[b]{0.245\textwidth}
	\includegraphics[width=\textwidth]{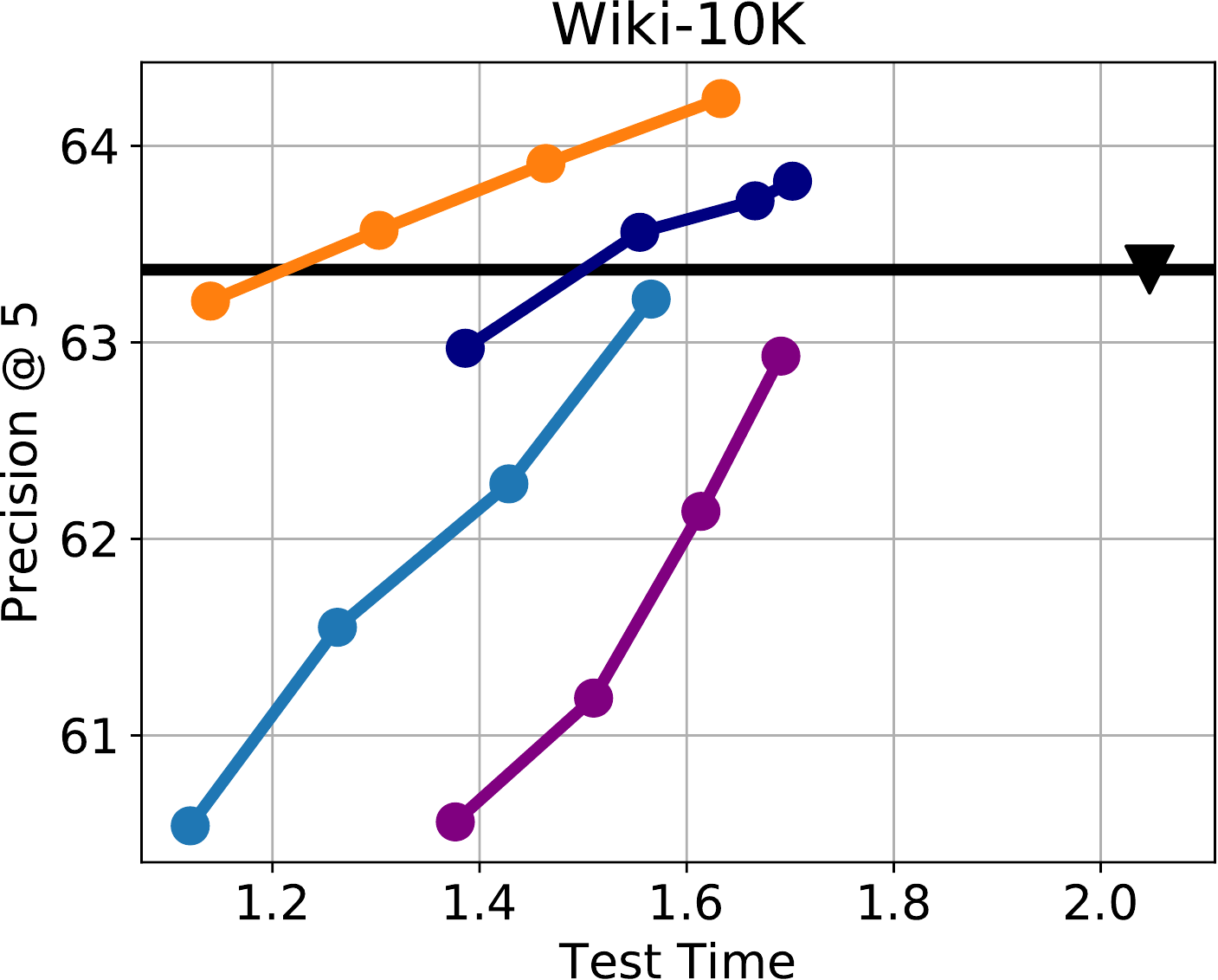}%
\end{subfigure}
\caption{A comparison between k-means based clustering, as followed by \defragx and nDCG based clustering, as followed by \defragn. Note that \ndcg-based clustering is more expensive (see Algorithm~\ref{algo:ndgc}). However, in several cases, e.g. Amazon-670K, AmazonCat, EURLex, \ndcg clustering can offer slightly superior performance in terms of classification accuracy.}%
\label{fig:app-ndcg}%
\end{figure}

\end{document}